\documentclass{article}
\usepackage[utf8]{inputenc}
\usepackage{amsmath, amssymb, amsthm}
\usepackage{graphicx}
\usepackage{hyperref}
\usepackage{geometry}
\usepackage{times}
\usepackage{natbib} 
\usepackage{authblk}  
\usepackage{setspace}
\usepackage{algorithm}
\usepackage{algpseudocode}
\usepackage{mathtools}
\usepackage{subcaption} 
\usepackage{booktabs}
\usepackage{amsfonts,amssymb,dsfont,color}
\usepackage{amsmath}
\usepackage{amssymb}
\usepackage{mathtools}
\usepackage{hyperref}
\usepackage{url}
\usepackage{amsthm}

\usepackage{enumitem}

\newtheorem{theorem}{Theorem}
\newtheorem{proposition}[theorem]{Proposition}
\newtheorem{lemma}[theorem]{Lemma}
\newtheorem{corollary}[theorem]{Corollary}
\newtheorem{definition}[theorem]{Definition}
\newtheorem{assumption}[theorem]{Assumption}

\newtheorem{remark}[theorem]{Remark}

\newcommand*\PR{\mathbb{P}}

\newcommand*\EXP{\mathbb{E}}

\newcommand{\bs}{\boldsymbol}
\newcommand{\BX}{\bar{\mathcal{X}}}
\newcommand{\BP}{\bar{{P}}}
\newcommand{\Br}{\bar{{r}}}
\newcommand{\Bx}{{\bar{{x}}}}

\newcommand\IND{\mathds{1}}

\newcommand*\REGRET{\mathcal{R}}

\newcommand{\BellOp}{{\mathfrak{T}}}

\newcommand{\Halmos}{{\qed}}

 \newcommand{\EDcomment}[1]{}

 \newcommand{\NimaEdits}[1]{{\color{black} #1}}
 \newcommand{\NimaResponse}[1]{}
\newcommand{\removed}[1]{}

\usepackage{aliascnt}
\newcounter{FNRB}
\newcounter{RFNRB}
\newcounter{LRFNRB}
\newcounter{ISRB}
\newcounter{RISRB}
\newcounter{LRISRB}
\renewcommand{\theFNRB}{FNRB}
\renewcommand{\theRFNRB}{RFNRB}
\renewcommand{\theLRFNRB}{LRFNRB}
\renewcommand{\theISRB}{ISRB}
\renewcommand{\theRISRB}{RISRB}
\renewcommand{\theLRISRB}{LRISRB}
\newenvironment{FNRB}[1][]{%
	\refstepcounter{FNRB}%
	\noindent\textbf{Problem~\theFNRB: #1}%
	\label{prob:ns-neutral}%
}{}
\newenvironment{RFNRB}[1][]{%
	\refstepcounter{RFNRB}%
	\noindent\textbf{Problem~\theRFNRB: #1}%
	\label{prob:ns-risk}%
}{}
\newenvironment{LRFNRB}[1][]{%
	\refstepcounter{LRFNRB}%
	\noindent\textbf{Problem~\theLRFNRB: #1}%
	\label{prob:ns-learning}%
}{}
\newenvironment{ISRB}[1][]{%
	\refstepcounter{ISRB}%
	\noindent\textbf{Problem~\theISRB: #1}%
	\label{prob:inf-neutral}%
}{}
\newenvironment{RISRB}[1][]{%
	\refstepcounter{RISRB}%
	\noindent\textbf{Problem~\theRISRB: #1}%
	\label{prob:inf-risk}%
}{}

\author[1,2]{Nima Akbarzadeh \thanks{Corresponding author.\\Email addresses: \href{mailto:nima.akbarzadeh@hec.ca}{\textit{nima.akbarzadeh@hec.ca}},   \href{mailto:yossiri.adulyasak@hec.ca}{\textit{yossiri.adulyasak@hec.ca}}, and
\href{mailto:erick.delage@hec.ca}{\textit{erick.delage@hec.ca}}.}}
\author[1]{Yossiri Adulyasak}
\author[1,2]{Erick Delage }
\affil[1]{GERAD \& Department of Decision Sciences, HEC Montréal}
\affil[2]{MILA - Québec AI Institute}

\title{Risk-Aware Decision Making in Restless Bandits: Theory and Algorithms for Planning and Learning}

\date{}

\begin{document}
\maketitle

\begin{abstract}
  In restless bandits, a central agent is tasked with optimally distributing limited resources across several bandits (arms), with each arm being a Markov decision process. In this work, we generalize the traditional restless bandits problem with a risk-neutral objective by incorporating risk-awareness, which is particularly important in various real-world applications especially when the decision maker seeks to mitigate downside risks. We establish indexability conditions for the case of a risk-aware objective and provide a solution based on Whittle index for the first time for the planning problem with finite-horizon non-stationary and for infinite-horizon stationary Markov decision processes. In addition, we address the learning problem when the true transition probabilities are unknown by proposing a Thompson sampling approach and show that it achieves bounded regret that scales sublinearly with the number of episodes and quadratically with the number of arms. The efficacy of our method in reducing risk exposure in restless bandits is illustrated through a set of numerical experiments in the contexts of machine replacement and patient scheduling applications under both planning and learning setups.
\end{abstract}

\section{Introduction}

The restless bandits (RB) problem is a class of sequential stochastic control problems for dynamic decision-making under uncertainty. In RB, a central agent confronts the challenge of allocating limited resources over time among competing options, which we refer to as \textit{arms}, each characterized by a Markov decision process (MDP). Such a framework has numerous applications in scheduling problems that appear in 
machine maintenance~\citep{glazebrook2005index,akbarzadeh2019restless}, healthcare~\citep{deo2013improving}, finance~\citep{glazebrook2013monotone}, power management in smart grids~\citep{wang2014adaptive,abad2016near}, opportunistic scheduling in networks~\citep{liu2010indexability,nino2009restless,ouyang2015downlink,borkar2017opportunistic,wang2019opportunistic},
and operator allocation in multi-robot systems~\citep{dahiya2022scalable}.

The studies of RB have primarily focused on risk-neutral/expected value objectives under reward maximization or cost minimization. The assumption of risk neutrality, however, is not always suitable in practice, as risk is an essential aspect to consider in real-world applications \citep{xu2021online,mate2021risk}. Such applications encompass various domains, including preventive maintenance~\citep{amiri2018providing}, surgery and medical scheduling in healthcare ~\citep{he2019controlling,najjarbashi2019variability}, financial portfolio management, and production lot-sizing~\citep{long2023robust}, where risk-neutral solutions can be impracticable as they might give rise to undesirable outcome. In such circumstances, risk-aware policies that account for potential risks offer resilient solutions. These are designed to reduce adverse effects of uncertain outcomes thus guaranteeing that the allocation policy remains effective even in adverse conditions~\citep{rausand2013risk}.

In the traditional RB with the risk-neutral objective, the key challenge that prevents applying traditional stochastic control methods is the curse of dimensionality due to the state space. As the number of arms increases, the computational complexity of identifying the optimal policy grows exponentially. This poses an obstacle in implementing the optimal policy for real-world applications. \cite{whittle1988restless} introduced a scalable and computationally tractable index policy as a heuristic for RB. The Whittle index acts as a priority index that highlights the urgency of selecting an arm. In what follows, we refer to this policy as the \textit{Whittle index policy}~\citep{nino2023markovian}.

While the Whittle index policy requires a technical condition, known as \textit{indexability}, to be satisfied, most of the studies in the literature of RB implement Whittle index policy \NimaEdits{despite the fact that the indexability is not guaranteed}. In these problems, either the problem structure is such that the indexability is satisfied~\citep{jacko2012opportunistic,borkar2017opportunistic,yu2018deadline,wang2019opportunistic}, or sufficient conditions under which the problem is indexable are specified~\citep{glazebrook2006some,ninomora2007,akbarzadeh2022conditions}. \NimaEdits{Many studies suggest} that Whittle index policy works well in practice ~\citep{glazebrook2006some,ninomora2007,akbarzadeh2022conditions,avrachenkov2013congestion,wang2019whittle}.

Risk-aware objectives have been studied extensively for MDPs and reinforcement learning (RL)~\citep{le2007robust,osogami2012robustness,bauerle2014more,chow2015risk,mannor2016robust,jaimungal2022robust,xu2023regret}. 
To the best of our knowledge, only \cite{mate2021risk} have considered risk-aware RB and do so through numerical experiments by using a specific utility function for a binary-state partially-observable MDP.
In contrast with \cite{mate2021risk}, our research provides a comprehensive analysis of risk-awareness across more general state spaces, dynamics, and utility functions while providing some analytical sufficient conditions under which the problem is indexable. This broader approach allows us to uncover deeper insights into risk-aware decision-making in RB. In this work, we consider three distinct cases: the finite-horizon RB, its non-stationary finite-horizon extension, and the infinite-horizon discounted formulation.

Our risk-aware framework is motivated by the high-stakes nature of modern applications, where optimizing for the expected (risk-neutral) performance is often insufficient. For instance, in domains like industrial maintenance or patient scheduling, a risk-averse approach is critical. Risk-neutral models can undervalue the probability of catastrophic equipment failures~\citep{calabro2024emerging} or lead to severe adverse health outcomes by failing to integrate the severity of systemic uncertainties~\citep{kohn2000to}. Conversely, in fields such as venture capital and exploratory research, a risk-seeking perspective might be encouraged, prioritizing potentially high payoffs and breakthroughs despite greater uncertainty~\citep{vahidi2025risk}. The stochastic dynamics central to our work mirror the challenges present across this spectrum, making a risk-aware objective essential for developing robust and effective policies tailored to specific domain requirements.

The contributions of our work are threefold. First, we generalize the traditional RB with a risk-neutral objective by incorporating risk-awareness to optimize decision-making with respect to a risk criterion in non-stationary finite-horizon and stationary infinite-horizon discounted settings. Second, we derive conditions under which an arm with a risk-aware objective is indexable, thereby enabling the derivation of the Whittle index policy in non-stationary finite-horizon and stationary infinite-horizon discounted settings. Third, we address the learning problem under a Bayesian regret setting when the true transition probabilities are unknown by proposing a Thompson sampling approach that samples from posterior distributions over the unknown parameters~\citep{osband2013more,russo2018tutorial}.

In a recent study,~\cite{akbarzadeh2023learning} has focused on RB with unknown transition dynamics, broadening the scope of applications and solution approaches. Note that papers by \cite{liu2012learning,khezeli2017risk}, \NimaEdits{and} \cite{xu2021online} adopt another viewpoint toward RB which is not exactly the same as our problem of interest and is not based on the Whittle index policy. These studies consider a variant of RB in which there is a single ``best'' arm delivering the highest stationary reward; the objective is to learn this arm and pull it indefinitely to maximize long-run return. Although such strategies are computationally tractable, they differ structurally from our optimization approach.

In our finite-horizon setup, we extend the solution proposed in~\cite{akbarzadeh2023learning} by deriving regret bounds that scale sublinearly with the number of episodes and quadratically with the number of arms. 
It should be noted that applying a conventional reinforcement learning algorithm to the RB in a naive manner is inefficient due to the linear growth of regret in the state space of Markov decision processes. \NimaEdits{This implies that regret for RBs} grows exponentially with the number of arms~\citep{akbarzadeh2023learning}. Finally, we numerically illustrate the efficacy of our methodology in reducing risk exposure in RB through experiments in machine replacement and patient scheduling applications under both planning and learning setups.

In Section~\ref{sec:prob}, we present the notation, problem formulation for planning and learning setups. Section~\ref{sec:whittle} describes the Whittle index solution concept and a class of indexable RB under a risk-aware objective and how Whittle indices can be computed. In Section~\ref{sec:unknown}, we address the learning problem. Section \ref{sec:infiniteH}  extends the planning result to an infinite horizon setting and discusses the question of learning. Finally, the numerical analysis is discussed in Section~\ref{sec:numerical} and the conclusion is presented in Section~\ref{sec:conclusion}. Note that all proofs are included in the Electronic Companion.

\section{Problem Definition for Finite-Horizon RB}\label{sec:prob}
In this section, we present the problem formulation for the finite-horizon restless bandits problem, encompassing both risk-neutral and risk-aware objectives, as well as a learning problem where the arm parameters are unknown.

\subsection{Notation}
Events that occur in a discrete time space will be indexed by $t \in \mathcal{T} := \{0, \ldots, T-1\}$. Random variables and their realizations are denoted by capital and lowercase letters; for example, $X_t$ and $x_t$\NimaEdits{, respectively}. We use calligraphic letters to denote the set of all realizations, such as $\mathcal{X}$. Let $X_{a:b} := (X_a, \ldots, X_b)$ represent a collection of the random variables from time~$a$ to time~$b$, and let $\bs{X}_t = (X^1_t, \ldots, X^N_t)$ represent a collection of random variables from $N$ processes at time~$t$. The probability and the expected value of random events are denoted by $\mathbb{P}(\cdot)$ and $\mathbb{E}[\cdot]$, respectively. We let $\mathbb{I}(\cdot)$ be an indicator function which returns $1$ if the inner clause is true, and $0$ otherwise, let ${\bs 1}_{k}$ denote a vector of zeros where only the $k$-th element is one. The notation $(x)^+$ represents $\max \{0, x\}$. A function~$f$ is called superadditive on partially-ordered sets $\mathcal{X}$ and ${\cal Y}$ if given $x_1, x_2 \in \mathcal{X}$ and $y_1, y_2 \in {\cal Y}$ where $x_1 \geq x_2$ and $y_1 \geq y_2$, then
\(f(x_1, y_1) - f(x_1, y_2) \geq f(x_2, y_1) - f(x_2, y_2).\)

\subsection{Finite-Horizon Non-Stationary RB (FNRB)}\label{subsec:RBp}
A finite-horizon non-stationary restless bandit process (arm) is a Markov decision process defined by the tuple $(\mathcal{X}, \mathcal{A}, \{P_t(a)\}_{a \in \{0, 1\}, t \in \mathcal{T}}, \{r_t\}_{t \in \mathcal{T}}, x_0)$ where $\mathcal{X}$ denotes a finite state space, $\mathcal{A} = \{0, 1\}$ denotes the action space where we call action~$0$ the \textit{passive} action and action~$1$ the \textit{active} action, $P_t(a)$ denotes a time-dependent transition probability matrix when action $a \in \{0, 1\}$ is chosen at time $t$, $r_t: \mathcal{X} \times \{0, 1\} \to [r_{\min}, r_{\max}]$ denotes the non-stationary reward function and $r_{\min}$ and $r_{\max}$ are finite and non-negative, and $x_0$ denotes the initial state of the process. By the Markov property we have 
\[
\PR(X_{t+1} = x_{t+1} \mid X_{1:t} = x_{1:t}, A_{1:t} = a_{1:t}) =: P_t(x_{t+1} \mid x_t, a_t).
\]
Note that the time-dependence in both the transition probabilities and reward functions distinguishes the non-stationary setting from the stationary case.

An FNRB problem consists of a set of $N$ independent arms 
\[
(\mathcal{X}^i, \mathcal{A}, \{P^i_t(a)\}_{a \in \{0, 1\}, t \in \mathcal{T}}, \{r^i_t\}_{t\in\mathcal{T}}, x^i_0), \quad i \in \mathcal{N} \coloneqq \{1, \ldots, N\}.
\]
An agent observes the state of all arms and may decide to activate up to $M \leq N$ of them.
Let $\boldsymbol{\mathcal{X}} \coloneqq \prod_{i \in \mathcal{N}} \mathcal{X}^i$ denote the joint state space and let $\bs{\mathcal{A}}(M) \coloneqq \bigl\{ \bs{a} \in \mathcal{A}^N : \sum_{i=1}^{N} a^i \leq M \bigr\}$
denote the action set.
The immediate reward realized at time $t$ is
\[
\bs{r}_t(\bs{x}_t, \bs{a}_t) \coloneqq \sum_{i \in \mathcal{N}} r^i_t(x^i_t, a^i_t)
\]
when the system is in state $\bs{x}_t$ and the agent chooses action $\bs{a}_t \in \bs{\mathcal{A}}(M)$. 
Since the arms are independent, the probability of observing state $\bs{x}_{t+1}$, given the state $\bs{x}_t$ and the action $\bs{a}_t$, is denoted by
\[
\bs{P}_t(\bs{x}_{t+1} \mid \bs{x}_t, \bs{a}_t) \coloneqq \prod_{i \in \mathcal{N}} P^i_t(x^i_{t+1} \mid x^i_t, a^i_t).
\]

\subsection{Planning Problems}
Let ${\bs \pi} = ({\bs \pi}^1, \ldots, {\bs \pi}^N) :\boldsymbol{\mathcal{X}} \times \mathcal{T} \to \bs{\mathcal A}(M)$ denote a time-dependent Markovian deterministic policy for the system where ${\bs \pi}^i$ defines the action for arm~$i$ in the policy of the system, and let ${\bs \Pi}_M$ be the set of all such time-dependent deterministic Markov policies. Assume action~$A^i_t$ is prescribed by policy $\bs{\pi}^i$ at time~$t$. Then, any policy leads to a total reward for the system as follows:
\begin{align*}
    \bs{J}_{{\bs x}_0}(\bs{\pi}) & := \sum_{i \in \mathcal{N}} \sum_{t = 0}^{T-1} r^i_t\left(X^i_t, A^i_t\right) \biggm|_{\bs{\pi}, {\bs X}_0=\bs{x}_0}.
\end{align*}

We first describe the classical risk-neutral optimization problem as follows~\citep{whittle1988restless}.

\begin{FNRB} \label{prob:risk-neutral} 
    Given a set of $N$ arms~$(\mathcal{X}^i, \mathcal{A}, \{P^i_t(a)\}_{a \in \{0, 1\}, t \in \mathcal{T}}, \{r^i_t\}_{t \in \mathcal{T}}, x^i_0)$, $i \in \mathcal{N}$, where at most $M$ of them can be activated at a time, find a $\bs{\pi} \in \bs{\Pi}_M$ that maximizes $\mathbb{E}\Bigl[ \bs{J}_{\bs{x}_0}(\bs{\pi}) \Bigr]$.
\end{FNRB}

Problem~\ref{prob:risk-neutral} is a multi-stage stochastic control problem where the optimal policy can be obtained using dynamic programming~\citep{puterman2014markov}. However, as the cardinality of the state space is $\prod_{i \in {\cal N}}|\mathcal{X}^i|$, computing the optimal policy is intractable for large~$N$. In Section~\ref{sec:whittle}, we describe a well-known heuristic, i.e., the \textit{Whittle index policy}, as a solution to tackle this problem.

As discussed earlier, the assumption of risk-neutrality may not be suitable for various practical applications of RB. Thus, we generalize the objective to incorporate risk-sensitivity at the level of the total reward generated by each arm.
To this end, we leverage an expected utility formulation \NimaEdits{\citep{vonneumann1947}}, which is commonly used in the literature of risk-aware decision-making \citep{vonneumann1947,fishburn1968utility,pratt1978risk}.  More specifically, a concave or convex utility function models a risk-averse or risk-seeking behavior, respectively. \cite{Tversky:1979:Prospect} further suggest using an {S}-shaped utility function with inflection point at a reference value in order to model an attitude 
of risk aversion above the target whereas risk seeking interests below.

As shown in \cite{bauerle2014more}, history-dependent policies are generally considered for risk-sensitive MDPs due to the fact that the marginal value of the utility depends on the cumulative reward. In fact, the authors also prove that an optimal decision rule will exploit the cumulative reward accrued up to the time of implementing the action. We thus denote the set of history-dependent policies by ${\bs \Pi}_H$. Next, we formally define our problem of interest.


\begin{RFNRB}[(Risk-aware FNRB)] \label{prob:ns-risk-aware}
Given a set of non-decreasing Lipschitz continuous utility functions $U^i$, $i \in \mathcal{N}$, and a set of $N$ arms~$(\mathcal{X}^i, \mathcal{A}, \{P^i_t(a)\}_{a \in \{0, 1\}, t \in \mathcal{T}}, \{r^i_t\}_{t \in \mathcal{T}}, x^i_0)$, $i \in {\cal N}$, where at most $M \leq N$ arms can be activated at a time, find a history-dependent policy ${\bs \pi} \in {\bs \Pi}_H$ that maximizes $\mathbb{E}\left[ \bs{D}_{{\bs x}_0}(\bs{\pi}) \right]$ with
\begin{align*}
    \bs{D}_{{\bs x}_0}(\bs{\pi}) & := \sum_{i \in \mathcal{N}} U^i\left( \sum_{t = 0}^{T-1} r^i\left(X^i_t, A^i_t\right) \right) \biggm|_{\bs{\pi}, {\bs X}_0=\bs{x}_0}.
\end{align*}
\end{RFNRB}
Some examples of  risk-aware utility functions are described in Section~\ref{sec:numerical}.

It is important to note that when all utility functions are linear, the risk-aware objective reduces to the risk-neutral objective due to the linearity of expectation. In this special case, the problem simplifies, and an optimal policy is indeed Markovian, as stated for Problem~\ref{prob:risk-neutral}.

Problem~\ref{prob:ns-risk-aware} highlights risk-awareness for each arm, which aligns with scalarization methods for multi-objective setups~\citep{marler2010weighted,gunantara2018review}.
Note that solving the risk-aware problem (Problem~\ref{prob:ns-risk-aware}) is more difficult than the risk-neutral problem (Problem ~\ref{prob:risk-neutral}) as already in the case $N=1$, the dynamic programming equation must be written on a state space augmented by one continuous state capturing the cumulative reward so far~\citep{bauerle2014more}. In Section~\ref{sec:whittle}, we present how the \textit{Risk-Aware} Whittle index for Problem~\ref{prob:ns-risk-aware} can be obtained.

\subsection{Learning Problem under Bayesian Regret}
The transition probabilities of arms in Problem~\ref{prob:ns-risk-aware} may be unknown in various practical applications. Notable examples include a drug discovery problem when a new drug is discovered in a clinical setup \citep{ribba2020model} and a machine maintenance problem for new machines where state transition functions are unknown \citep{ogunfowora2023reinforcement}. One plausible objective in such contexts is to determine a learning policy that converges to the ideal policy (the solution to Problem~\ref{prob:ns-risk-aware}) as quickly as possible. Let us assume that the agent interacts with the system for $K$ episodes and let ${\bs \pi}_k$ denote a learning policy that is deployed in episode $k$. The performance of a learning policy is measured by Bayesian regret, which quantifies the difference between the policy's performance and that of an oracle who possesses complete knowledge of the environment and executes an optimal policy $\bs{\pi}^\star$, i.e.,
\begin{equation}
\REGRET(K) 
:= \EXP\biggl[ \sum_{k = 1}^{K} 
\mathbb{E}\left[ \bs{D}_{{\bs x}_0}(\bs{\pi}^\star) \right] - \mathbb{E}\left[ \bs{D}_{{\bs x}_0}(\bs{\pi}_k) \right]
\biggr]
\label{eq:reg-def}
\end{equation}
where the first expectation is calculated from the prior distribution on $\{P^i_t(\cdot|x^i, a)\}_{i \in \mathcal{N}, x^i \in \mathcal{X}^i, a \in \{0, 1\}, t \in \mathcal{T}}$, while the second expectation is calculated based on the initial states ${{\bs x}_0}$ and the learning policy. Bayesian regret is a widely-adopted metric in numerous studies   \citep{rusmevichientong2010linearly,agrawal2013thompson,russo2014learning,ouyang2017learning,akbarzadeh2023learning}.
\NimaEdits{We thus define the following learning problem.}

\begin{LRFNRB}[(Learning Risk-aware FNRB)] \label{prob:learning} 
Given a set of $N$ arms~$(\mathcal{X}^i, \mathcal{A}, \{P^i_t(a)\}_{a \in \{0, 1\}, t \in \mathcal{T}}, \{r^i_t\}_{t \in \mathcal{T}}, x^i_0)$, $i \in {\cal N}$, where at most \NimaEdits{$M \le N$} can be activated at a time, find a sequence of history-dependent policies $\left\{{\bs \pi}_k\right\}_{k \geq 1}$ that minimizes $\REGRET(K)$.
\end{LRFNRB} 

We present a Thompson sampling algorithm to tackle the learning problem and prove a regret bound of $\mathcal{O}(N^2 \sqrt{KT})$. Practically, when the optimal policy of RFNRB is computationally intractable, one can use the Whittle index policy as a proxy for the optimal policy. The effectiveness of this approach is demonstrated in our numerical experiments.

\section{Indexability and Whittle Index}
\label{sec:whittle}

\subsection{The Case of the Risk-neutral FNRB} \label{subsec:overview}
\citet{whittle1988restless} introduced a priority index policy as a heuristic for Problem \ref{prob:risk-neutral}, which has become widely accepted as the conventional method for solving the FNRB \citep{nino2023markovian}. This policy is obtained by relaxing the original hard constraint of activating at most $M$ arms at a time,
\begin{align*}
\max_{{\bs \pi} \in {\bs \Pi}_M} \mathbb{E}\Bigl[\bs{J}_{\bs{x}_0}({\bs \pi})\Bigr] \quad \text{s.t.} \quad \lVert \bs{A}_t \rVert_1 \leq M, \mbox{a.s.},
\end{align*}
to a constraint on the expected average number of activated arms per time period,
\begin{align*}
\max_{{\bs \pi} \in {\bs \Pi}_M} \mathbb{E}\Bigl[\bs{J}_{\bs{x}_0}({\bs \pi})\Bigr] \quad \text{s.t.} \quad \EXP\Biggl[\frac{1}{T}\sum_{t=0}^{T-1}\lVert \bs{A}_t \rVert_1 \Bigg| \bs{X}_0=x_0\Biggr] \leq M.
\end{align*}
This relaxation is crucial as it allows the overall problem to be decomposed into $N$ independent subproblems, thereby significantly reducing the computational complexity.

The relaxed problem is then decoupled into $N$ independent optimization problems using a Lagrangean relaxation parameterized by a multiplier $\lambda \in \mathbb{R}_+$:
\begin{equation}
\max_{{\bs \pi} \in {\bs \Pi}_M} \mathbb{E}\left[{\bs J}_{{\bs x}_0}({\bs \pi}) \right] - \lambda  \left(\EXP\biggl[ \dfrac{1}{T} \sum_{t = 0}^{T-1} \lVert {\bs A}_t \rVert_1 \bigg| {\bs X}_0 = \bs{x}_0 \biggr] - M\right) = \sum_{i=1}^N \max_{\pi^i \in \Pi^i_M} {J}^{i}_{\lambda, x^i_0}(\pi^i) + M \lambda \label{eq:lagr}
\end{equation}
where the policy functions captured in $\Pi^i_M$ are time-dependent Markovian and of the form $\pi^i: {\mathcal X}^i \times {\mathcal T} \to {\mathcal A}^i$, and where
\begin{equation*}
{J}^{i}_{\lambda, x^i_0}(\pi^i_\lambda) := \mathbb{E}\left[\sum_{t = 0}^{T-1} r^i_t\left(X^i_t, A^i_t\right)  - \dfrac{\lambda}{T} \sum_{t = 0}^{T-1} A^i_t \biggm| {X^i_0=x^i_0}\right].
\end{equation*}
For each arm $i$, let $\pi^{i*}_\lambda \in \Pi_M^i$ denote the optimal time-dependent Markov policy derived via dynamic programming on an MDP parameterized by $(\mathcal{X}^i,\mathcal{A},\{P^i_t(a)\}_{a\in\{0,1\}, t\in\mathcal{T}},\{r^i_{\lambda, t}\}_{t \in \mathcal{T}},x^i_0)$ with 
\[
r^i_{\lambda, t}(x,a):=r^i_t(x,a)-\frac{\lambda}{T}a
\]
The policies $\{{\pi}^{i*}_\lambda\}_{i=1}^N$ can be combined to form the solution to \eqref{eq:lagr} as
\(
{\bs \pi}^{*}_\lambda:=({\pi}^{1*}_\lambda,\dots,{\pi}^{N*}_\lambda).
\)

Next, we define the indexability and Whittle index for FNRB.

\begin{definition}[Indexability and Whittle index] \label{def:idxbl}
Given any optimal Markov policy ${\pi}^{i*}_{\lambda}$, let the \emph{passive set} be 
\[
\mathcal{W}^i_{\lambda} := \bigl\{ (x, t) \in \mathcal{X}^i \times \mathcal{T}: {\pi}^{i*}_{\lambda}(x, t) = 0 \bigr\}.
\]
An FNRB is \emph{indexable} if for all $i \in \mathcal{N}$, $\mathcal{W}^i_{\lambda}$ is non-decreasing in $\lambda$, i.e., for any $\lambda_1, \lambda_2 \in \mathbb{R}$ such that $\lambda_1 \leq \lambda_2$, we have $\mathcal{W}^i_{\lambda_1} \subseteq \mathcal{W}^i_{\lambda_2}$, for some sequence of optimal $\{{\pi}^{i*}_{\lambda}\}_{\lambda\geq0}$. For an indexable FNRB, the \emph{Whittle index}~$w^i(x, t)$ of state~$x \in \mathcal{X}^i$ at time~$t$ is the smallest value of $\lambda$ for which the state~$x$ is part of the passive set $\mathcal{W}^i_{\lambda}$ at time step~$t$, i.e., 
\[
w^i(x, t) := \inf \left\{ \lambda \in \mathbb{R}_+: (x, t) \in \mathcal{W}^i_{\lambda} \right\}.
\]
\end{definition}

The \textit{Whittle index policy} activates the arms with the $M$ largest Whittle indices at each time~$t$. By construction, the policy adheres to the $M$ activation limit.

Determining whether a problem is indexable is not immediately apparent, hence researchers have examined various sufficient conditions for indexability~\citep{glazebrook2006some,ninomora2007,akbarzadeh2022conditions}. Under certain conditions, the Whittle index policy is optimal~\citep{gittins1979bandit,weber1990index,lott2000optimality} and in other cases, the Whittle index policy is close to optimal~\citep{glazebrook2006some,ninomora2007,avrachenkov2013congestion,wang2019whittle,akbarzadeh2022conditions}. General algorithms for computing the Whittle indices are proposed in \citep{nino2007dynamic,akbarzadeh2022conditions}, while in cases where analytical verification of indexability is challenging, numerical methods can be employed to approximately verify indexability empirically by testing monotonicity of the passive set with respect to the penalty/subsidy parameter \citep{avrachenkov2018whittle,akbarzadeh2019dynamic}.

\subsection{Solution to Relaxation of \ref{prob:ns-risk-aware}} \label{subsec:augment}
To establish indexability conditions for Problem~\ref{prob:ns-risk-aware}, we first apply the relaxation and decomposition approach described in Section~\ref{subsec:overview}. Hence, we seek the optimal history-dependent policy of an arm that maximizes
\begin{equation} \label{eq:D-obj}
 {D}^i_{\lambda, x^i_0}({\pi}^i) := \mathbb{E}\Bigg[ \Bigg. U^i\left(\sum_{t = 0}^{T-1} r_t^i\left(X^i_t, A^i_t\right)\right) - \dfrac{\lambda}{T} \sum_{t=0}^{T-1} A^i_t \bigg| X^i_0 = x^i_0 \Bigg. \Bigg]
\end{equation}
among all policies in $\Pi_H^i$.

\NimaEdits{To solve~\eqref{eq:D-obj}, we adopt} the steps presented in \citet{bauerle2011markov} and introduce a new \textit{augmented arm} risk-neutral MDP, which is equivalent to the risk-aware MDP above and can be solved using dynamic programming. Specifically, for each arm $i \in {\cal N}$, 
we construct an augmented time-dependent MDP $(\{\BX^i_t\}_{t=0}^{T-1}, \mathcal{A}, \{\BP^i_t(a)\}_{a \in \{0, 1\}, t \in \mathcal{T}}, \{\Br_t^i\}_{t \in \mathcal{T}}, \Bx^i_0)$ with 
$\BX^i_t:=\mathcal{X}^i\times \mathcal{S}^i_t$ where $\mathcal{S}^i_t :=\{\sum_{t'=0}^{t-1} r_{t'}^i(x_{t'},a_{t'}):x_{t'}\in\mathcal{X}^i,a_{t'}\in\mathcal{A},\forall\, 0\leq t'\leq t-1\}\subseteq \mathbb{R}$ denotes the space of possibly realized accumulated rewards at time $t$, with $|\mathcal{S}_t^i| \leq(|\mathcal{X}^i||\mathcal{A}|)^t$, where \(\Bx_0^i:=(x_0^i,0)\), and where
\[
\BP^i_t(x',s'|x,s,a):=P^i_t(x'|x,a)\mathbb{I}(s',s+r^i_{t}(x,a)),\forall x,x'\in\mathcal{X}^i,s\in\mathcal{S}_t^i, a\in\mathcal{A},s'\in\mathcal{S}_{t+1}^i
\]
\[
\Br_{\lambda, t}^i(x,s,a):=\mathbb{I}(t,T-1)U^i(s+r^i_t(x,a)) - (\lambda/T)a,\forall x\in\mathcal{X}^i,s\in\mathcal{S}_t^i, a\in\mathcal{A}.
\]

This reduces to \citet{bauerle2011markov} when $\lambda = 0$.
Then, the following result can be immediately derived from Theorem 1 of \citet{bauerle2014more}.
\begin{proposition}\label{thm:optimalityAugmentedMDP}
Let $f_{\lambda}^{i*} \coloneqq \{f_{ \lambda, t}^{i*}\}_{t \in \mathcal{T}}$ be an optimal Markovian policy for the \textit{augmented arm} risk-neutral MDP. Then, one can construct an optimal policy for the relaxation of Problem \ref{prob:ns-risk-aware} using:
\[\bar{\pi}_{\lambda, t}^{i*}(x_{0:t},a_{0:t-1}):=f_{\lambda, t}^{i*}(x_t, \sum_{t'=0}^{t-1} r^i_{t'}(x_{t'},a_{t'})).\]
Namely, $\max_{\pi^i_\lambda\in \Pi_H^i} {D}^i_{\lambda, x^i_0}(\pi^i_\lambda) = {D}^i_{\lambda, x^i_0}(\bar{\pi}_{\lambda}^{i*})$.
\end{proposition}
This implies that the properties of $\bar{\pi}_{\lambda, t}^{i*}$ can be effectively studied by analyzing the optimal Markovian policy $f_{\lambda}^{i*}$.

\subsection{Two Classes of Indexable Arms for Problem \ref{prob:ns-risk-aware}}

We extend the definition of indexibility to Problem \ref{prob:ns-risk-aware}, which not necessarily admits an optimal policy that is Markovian.

\begin{definition}[Indexability of Problem \ref{prob:ns-risk-aware}] \label{def:idxbl_RFNRB}
An RFNRB is \emph{indexable} if, for all $i \in \mathcal{N}$, the passive set
\[
\mathcal{W}^i_{\lambda} := \bigl\{ (t,x_{0:t},a_{0:t-1})) \in \cup_{t\in\mathcal{T}}\{t\}\times{\mathcal{X}^i}^{t+1}\times\mathcal{A}^{t} : \bar{\pi}^{i*}_{\lambda,t}(x_{0:t},a_{0:t-1}) = 0 \bigr\},
\]
is non-decreasing in $\lambda$ for some sequence of optimal history dependent policy $\{\bar{\pi}^{i*}_{\lambda}\}_{\lambda\geq0}$. For an indexable RFNRB, the \emph{Whittle index} is~$w^i(x_{0:t},a_{0:t-1}, t):= \inf \left\{ \lambda \in \mathbb{R}_+: (t,x_{0:t},a_{0:t-1}) \in \mathcal{W}^i_{\lambda} \right\}$.
\end{definition}

We establish a first set of sufficient conditions under which Problem \ref{prob:ns-risk-aware} is indexable. \NimaEdits{Let us} consider the following Assumptions.

\begin{assumption}\label{ass:mdp}
The utility function $U^i(\cdot)$ is convex and non-decreasing and the MDP of each satisfies the following conditions for all $t \in \mathcal{T}$:
\begin{enumerate}[label=\ref*{ass:mdp}.\alph*]
    \item \label{ass:ordered}$\mathcal{X}$ is a totally ordered set, or equivalently takes the form $\mathcal{X}:=\{1,\dots,|\mathcal{X}|\}$ 
    \item \label{ass:r1} The reward function $r^i_t(x,a)$ is non-decreasing in $x$ for all $a \in \mathcal{A}$.
    \item \label{ass:r2} The rewards are action-independent: $r^i_t(x,1) = r^i_t(x,0)$ for all $x \in \mathcal{X}$.
    \item \label{ass:p1} For all $x,k$: $\sum_{x'\geq k} P^i_t(x'|x,1) \geq \sum_{x'\geq k} P^i_t(x'|x,0)$.
    \item \label{ass:p2} For all $x_1 \geq x_2$ and all $k$: $\sum_{x'\geq k} P^i_t(x'|x_1,a) \geq \sum_{x'\geq k} P^i_t(x'|x_2,a)$.
    \item \label{ass:p3} For all $x_1 \geq x_2$ and all $k$: 
    \[\sum_{x'\geq k} P^i_t(x'|x_1,1) - \sum_{x'\geq k} P^i_t(x'|x_1,0) \geq \sum_{x'\geq k} P^i_t(x'|x_2,1) - \sum_{x'\geq k} P^i_t(x'|x_2,0).\]
\end{enumerate}
\end{assumption}

We then establish the monotonicity of the optimal policy under the assumption:

\begin{lemma}[Monotone Policy]\label{lem:monotone-policy}
If a restless bandit arm $i$ satisfies Assumption~\ref{ass:mdp}, then there exists a family of optimal policies $\{f^{i*}_\lambda\}_{\lambda\geq 0}$, for its augmented arm risk-neutral MDP, that is non-increasing with respect to $\lambda$.
\end{lemma}

We can now state our main indexability result.

\removed{
\begin{theorem}[Indexability of Problem \ref{prob:ns-risk-aware}]\label{thm:ns-indexable}
An augmented restless bandit arm with finite-horizon MDP satisfying Assumptions~\ref{ass:mdp} is indexable.
\end{theorem}

\textbf{\textit{Proof Sketch.}}
\NimaEdits{By Lemma~\ref{lem:monotone-policy}, the optimal policy is non-decreasing in $(x,s,\phi)$. 
Since the activation penalty $\lambda$ enters the Bellman equation equivalently to $-\phi$, with $\phi$ as a  (a subsidy for passivity), the optimal policy is non-increasing in $\lambda$. Therefore, for $\lambda_1 \leq \lambda_2$, we have $f^{i *}_{\lambda_1,t}(x,s) \geq f^{i *}_{\lambda_2,t}(x,s)$ for all $(x,s,t)$, which implies $\mathcal{W}_{\lambda_1} \subseteq \mathcal{W}_{\lambda_2}$, establishing indexability.}
\Halmos}

\begin{theorem}[Indexability of Problem \ref{prob:ns-risk-aware}]\label{thm:ns-indexable}
Problem \ref{prob:ns-risk-aware} is indexable if all restless bandit arms  satisfy  Assumption~\ref{ass:mdp}.
\end{theorem}

\textbf{\textit{Proof.}}
By Lemma~\ref{lem:monotone-policy}, we have that, for each arm $i$, there exists a family of optimal policies $\{f_\lambda^{i*}\}_{
\lambda\geq 0}$ that is non-increasing in $\lambda$. According to Proposition \ref{thm:optimalityAugmentedMDP}, each element, indexed by $\lambda$, of this  sequence  can be used to construct an element of a family $\{\bar{\pi}_{\lambda}\}_{t\in\mathcal{T}}$ that defines some $\mathcal{W}_\lambda^i$. Indexability can be confirmed by verifying that for any $\lambda_1 \leq \lambda_2$, any $(t,x_{0:t},a_{0:t-1})$, we have
\begin{align*}
&(t,x_{0:t},a_{0:t-1})\in \mathcal{W}_{\lambda_1}^i \Rightarrow \bar{\pi}_{\lambda_1,t}^{i*}(x_{0:t},a_{0:t-1})=0 \Rightarrow f_{\lambda_1, t}^{i*}(x_t, \sum_{t'=0}^{t-1} r^i_{t'}(x_{t'},a_{t'}))=0\\
&\quad\Rightarrow f_{\lambda_2, t}^{i*}(x_t, \sum_{t'=0}^{t-1} r^i_{t'}(x_{t'},a_{t'}))\leq f_{\lambda_2, t}^{i*}(x_t, \sum_{t'=0}^{t-1} r^i_{t'}(x_{t'},a_{t'}))=0 \Rightarrow \bar{\pi}_{\lambda_2,t}^{i*}(x_{0:t},a_{0:t-1})=0,
\end{align*}
which finally implies that $(t,x_{0:t},a_{0:t-1})\in \mathcal{W}_{\lambda_2}^i$.
Thus $\mathcal{W}_{\lambda_1}^i \subseteq \mathcal{W}_{\lambda_2}^i$ and Problem \ref{prob:ns-risk-aware} is indexable.
\Halmos

\begin{remark}
The conditions in Assumption~\ref{ass:mdp} have natural interpretations in applications. Condition \ref{ass:r2} implies that rewards depend only on the state, not the action taken. Condition \ref{ass:p1} indicates that active actions lead to stochastically better states. Conditions \ref{ass:p2} and \ref{ass:p3} capture that better states have better transition prospects, with condition \ref{ass:p3} ensuring that the advantage of active actions increases with state quality.
\end{remark}

Next, we present a second set of sufficient conditions for indexability of Problem \ref{prob:ns-risk-aware}.

\begin{assumption}\label{ass:mdp2}
The MDP satisfies the condition that for all 
\NimaEdits{$t \in \{0, \ldots, T-3\}$} and $x \in {\mathcal X}^i$:
    \[
    \|P^i_t(\cdot|x, 0) - P^i_t(\cdot|x, 1)\|_{TV} := \frac{1}{2}\sum_{x' \in {\mathcal X}^i} |P^i_t(x'|x, 0) - P^i_t(x'|x, 1)| \leq \frac{1}{2(T - t - 2)}.
    \]
\end{assumption}

\begin{theorem}
\label{thm:ns-indexable2}
Problem \ref{prob:ns-risk-aware} is indexable if all restless bandit arms  satisfy  condition \ref{ass:r2} and Assumption~\ref{ass:mdp2}.
\end{theorem}

\textit{Proof Sketch.}
To establish indexability, we need to show that the set of states where the passive action (action 0) is optimal expands monotonically as the penalty $\lambda$ increases. This is equivalent to proving that the difference between the value of taking the passive action and the active action is non-decreasing in $\lambda$. Using backward induction, we first establish that the rate of change of the value function with respect to $\lambda$ is bounded within $[-(T-1-t)/T, 0]$ at each time $t$. The total variation condition then ensures that the transition probability differences are sufficiently small, guaranteeing that this difference function increases with $\lambda$, thus establishing indexability. See Section \ref{app:ns-indexable2} in Electronic Companion for more details.
\Halmos

There are notable practical implications with respect to the assumptions of Theorem \ref{thm:ns-indexable}. In the machine maintenance application~\citep{glazebrook2006some,akbarzadeh2019restless,abbou2019group}, they imply that the state of each machine has a direct impact on its performance and so, a better state results in a better productivity. In addition, the machines deteriorate stochastically if not maintained, with a chance of worsening to any state. However, the chance of visiting a worse state is higher if the current state is worse, which is often the case in practice. Also, on the other hand, repairs may not always fully fix the issue but the chance is higher in better states. \NimaEdits{Note that stochastic deterioration and imperfect repairs generalize the transition dynamics in \cite{glazebrook2006some}.}

\NimaEdits{Such conditions are also particularly relevant in the context of} patient scheduling in hospitals. These assumptions imply that patients' conditions can deteriorate stochastically if not attended to, with worse conditions increasing the urgency for intervention. Treatments, while beneficial, may not always fully restore health immediately, reflecting a probabilistic transition towards recovery. This probability is higher if the patient is in a better state. This reflects realistic clinical scenarios where early and continuous care increases the likelihood of reaching an optimal health state~\citep{starfield2005health}

Theorem~\ref{thm:ns-indexable} requires convex utility, which implies risk-seeking preferences and might not fit to some real-world applications. However, Theorem~\ref{thm:ns-indexable2} takes a different approach: it drops the convexity requirement and the ordering conditions, replacing them with a single bounded total variation constraint. These alternative sufficient conditions suggest that indexability can arise from different problem structures. Finally, while Theorems~\ref{thm:ns-indexable} and \ref{thm:ns-indexable2} provide sufficient conditions for indexability, arms may still be indexable even when these conditions are violated. The key requirement is that the optimal policy be monotone in the penalty parameter $\lambda$. 
\NimaEdits{As shown in \citet[Table 2]{nino2007dynamic}, the number of nonindexable restless bandits are very limited. In addition, we did not encounter any nonindexable instances in our numerical analysis.}

\subsection{Computation of Whittle Indices} \label{subsec:compute}
Primarily, there are three approaches to consider for computing Whittle indices. One approach is problem-specific where the Whittle index formula is computed exactly~\citep{jacko2012opportunistic,borkar2017opportunistic,yu2018deadline,wang2019opportunistic}. The other one is the \textit{modified adaptive greedy algorithm} \citep{akbarzadeh2022conditions}, which works for any indexable RB in discounted and average reward setups. The last one is numerical search, which is either through \textit{adaptive greedy} \citep{nino2007dynamic}, binary search~\citep{qian2016restless,akbarzadeh2019dynamic}, or brute-force search.

In this work, we adapt the numerical search approach described in \citet{akbarzadeh2019dynamic} to compute the Whittle indices using binary search. This approach works for any indexable restless bandits instance and the key difference in our setup is the risk-aware MDP solver. The procedure of the binary search algorithm \NimaEdits{is presented in Algorithms \ref{alg:IDX} and \ref{alg:CRC}. Based on the equivalence established in Proposition \ref{thm:optimalityAugmentedMDP}, the Whittle index $w^i$ for a history $(x_{0:t},a_{0:t-1})$ is recovered from the computed index $\hat{w}$ by mapping the history to its corresponding augmented state $(x_t, s_t)$, where $s_t = \sum_{t'=0}^{t-1} r^i_{t'}(x_{t'},a_{t'})$ is the accumulated reward:
\(
w^i(x_{0:t},a_{0:t-1}, t) := \hat{w}\left(x_t, \sum_{t'=0}^{t-1} r^i_{t'}(x_{t'},a_{t'})\right)
\).}
For simplicity, we drop the superscript~$i$ in the pseudo-code.

The computational complexity of the Whittle index calculation is primarily determined by the backward induction over \(T\) time steps for a Markov decision process with two actions. The state space includes dimensions $\mathcal{X}$ and $\mathcal{S}_t$, for each $t$, with transitions governed by two  matrices of size \(|\mathcal{X}| \times |\mathcal{X}|\). The size of \(|\mathcal{S}|\) depends critically on the per-step reward structure. Since the per-step rewards are discrete values, we assume they are multiples of some $\delta > 0$ and span a range $[k_{\min}\delta, k_{\max}\delta]$, then the number of distinct accumulated reward values is $O(T(k_{\max}-k_{\min}))$. Thus, the complexity for solving the MDP for each specific value of $\lambda$ is \(O(|\mathcal{X}|^2 (k_{\max}-k_{\min}) |\mathcal{A}|T^2)\). Additionally, the critical penalty search contributes a logarithmic factor of \(O\left(\log \left((UB - LB)/\epsilon\right)\right)\) and must be done at most $O(|\mathcal{X}| (k_{\max}-k_{\min}) T)$ times. Thus, the overall complexity of the algorithm is \(O\left(|\mathcal{X}|^3 (k_{\max}-k_{\min})^2 |\mathcal{A}|T^3 \log \left((UB - LB)/\epsilon\right)\right)\).

\begin{algorithm}[t!]
\small
\caption{Whittle Index Calculation for an arm} \label{alg:IDX}
\begin{algorithmic}[1]
\State \textbf{Input:} $\mathtt{LB}$ (initial lower bound), $\mathtt{UB}$ (initial upper bound), $\epsilon$ (tolerance)
\NimaEdits{\State $\mathtt{LB}$ and $\mathtt{UB}$ should be such that $f_{\mathtt{LB},t}^*(x,s)=1$ and $f_{\mathtt{UB},t}^*(x,s)=0$ for all $(t,x,s)$
\State $k \leftarrow 1$, $f^*_{1,t}(x,s)\leftarrow f_{\mathtt{LB},t}^*(x,s)$ for all $(t,x,s)$
\Repeat
\State Set $k \leftarrow k+1$
  \State Compute 
    \(
      (\lambda_{k}, f^*_{k})
      = \Lambda^{\epsilon}(\lambda_k,UB)
    \)
    using Algorithm~\ref{alg:CRC}.
  \State Set
    \(
      \hat w(x,s,t)\gets\lambda_{k}
    \), for all $(x,s,t):
        f^*_{k-1,t}(x,s)=1
        \wedge
        f_{k,t}(x,s)=0
      \bigr\}$ 
\Until{$f_{k}^*=0$}}
\State \textbf{return} $\hat w$
\end{algorithmic}
\end{algorithm}

\begin{algorithm}[h!]
\small
\caption{Next Critical Penalty Finder: $\Lambda^{\epsilon}(\lambda^-, \lambda^+)$\EDcomment{Nima, can you help me fix the errors}\NimaResponse{Couldn't!}} \label{alg:CRC}
\begin{algorithmic}[1]
\State \textbf{Input:} $\alpha_l$ (current lower bound), $\alpha_u$ (current upper bound), $\epsilon$ (tolerance level)
\State Compute $f^*_{\alpha_{u}}$
\While{$\lambda^+ - \lambda^- \ge \epsilon$}
\State Set $\lambda \gets (\lambda^- + \lambda^+) / 2$
\State Compute $f^*_{\lambda}$
\If{$f^*_{\lambda}=f^*_{\lambda^-}$} Update $\lambda^- \gets \lambda$.
\Else{}  Update $\lambda^+ \gets \lambda$.
\EndIf
\EndWhile
\State \textbf{return} $\lambda^+, f^*_{\lambda^+}$.
\end{algorithmic}
\end{algorithm}

\section{Learning Policy} \label{sec:unknown}
In this section, we propose a learning algorithm, called \textit{Learning Risk-Aware Policy using Thompson Sampling} (\texttt{LRAP-TS}), to address Problem~\ref{prob:learning}. In this setup, a learning agent aims to minimize the regret in \NimaEdits{\eqref{eq:reg-def}} during $K$ episodes of interaction with the system by balancing the \textit{exploration-exploitation trade-off} \citep{auer2002finite}.

We first introduce additional parameters and variables. Let $\theta^{\star i} \in \Theta^i$ represent the unknown parameters that define the transition probabilities of arm $i \in {\mathcal N}$, where $\Theta^i$ is a compact set. Additionally, we assume that ${\theta^{\star i}}$ are independent of each other and $\bs{\theta}^\star = (\theta^{\star 1}, \ldots, \theta^{\star N})$. Let $\phi^i_1$ capture the prior on $\theta^i$ for each arm $i \in \mathcal{N}$. Furthermore, we let $h^i_k$ be the history of states and actions of arm $i$ during episode $k$, and $\phi^i_k$ be the posterior distribution on $\theta^{\star i}$ given $(h_1^i,\dots,h^i_{k-1})$. Then, upon applying action $a$ at state~$x$ and observing the next state~$x'$ for arm~$i$, the posterior distribution $\phi^i_{k+1}$ can be computed using Bayes rule as 
\begin{equation} \label{eqn:phi-update}
\phi^i_{k+1}(d \theta) = \dfrac{P^i_{\theta}(x' | x, a) \phi^i_{k}(d \theta)}{\int P^i_{\tilde \theta}(x' | x, a) \phi^i_{k}(d \tilde \theta)}.
\end{equation}
If the prior is a conjugate distribution on $\Theta^i$, then the posterior can be updated in closed form. We note that our algorithm and regret analysis do not depend on the specific structure of the prior and posterior update rules.

\subsection{\textsc{LRAP-TS} Algorithm}

Algorithm \texttt{LRAP-TS} operates in episodes each with length~$T$.
It maintains a posterior distribution~$\phi^i_k$ on the dynamics of arm~$i$ and keeps track of \( N^{i}_{k, t}(x, a) = \sum_{\kappa = 1}^{k} \sum_{\tau = 0}^{t} \IND\{(X^i_{\kappa, \tau}, A^i_{\kappa, \tau}) = (x, a)\} \) and \( N^{i}_{k, t}(x, a, x_{+}) = \sum_{\kappa = 1}^{k} \sum_{\tau = 0}^{t} \IND\{(X^i_{k, t}, A^i_{k, t}, X^i_{k, t+1}) = (x, a, x_{+})\}. \)

In particular, we use Dirichlet priors for the transition probabilities, which provide conjugate updates. The posterior after observing a transition from $x$ to $x'$ under action $a$ at time $t$ in episode $k$ is:
\begin{equation*}
    P^i_{k, t}(\cdot | x, a) \sim \text{Dirichlet}(N^i_{k, t}(x,a,1)+1, \ldots, N^i_{k, t}(x,a,{|\mathcal{X}^i|})+1).
\end{equation*}

At the beginning of episode~$k$, one starts by sampling for each arm $i \in \mathcal{N}$ a set of parameters $\theta^i_k$ from the posterior distribution~$\phi^i_{k}$. The optimal policy for problem \ref{prob:learning} under $\{\theta^i_k\}_{i\in\mathcal{N}}$ is then identified and implemented until the end of the horizon. The transitions and rewards observed along the trajectories are finally used to update the posteriors $\phi^i_{k+1}$ for each arm. The algorithm is described in Algorithm~\ref{alg:LRAP-TS}.

\begin{algorithm}[h!]
\small
\caption{\textsc{LRAP-TS}\label{alg:LRAP-TS}}
\begin{algorithmic}[1]
\State Input: Initial states $ \{x^i_0\}_{i \in \mathcal{N}} $, priors $ \{\phi^i_1\}_{i \in \mathcal{N}} $.
\For{$k = 1, 2, \ldots, K$}
\State Sample $\theta^i_{k} \sim \phi^i_{k}$ for arm $i \in \mathcal{N}$ and compute the estimated risk-aware policy.
\For{$t \in \mathcal{T}$}
\State Implement actions based on the estimated risk-aware policy.
\EndFor
\State Update $\phi^i_{k+1}$ according to equation
\eqref{eqn:phi-update} for arm $i \in \mathcal{N}$.
\EndFor
\end{algorithmic} 
\label{Alg:ISRB-TSDE}
\end{algorithm}

\subsection{Regret Bound}
With a slight abuse of notation, let $(\bs{P}^{\star}, \bs{V}^\star, \bs{\pi}^\star)$ denote the transition probability matrix, the optimal value function, and the optimal policy for the overall system parameterized by the true parameters $\theta^\star$ and let $(\bs{P}_k, \bs{V}_k, \bs{\pi}_k)$ denote the transition probability matrix, the optimal value function, and the optimal policy for the overall system parameterized by the estimated parameters ${\theta_k}$ in episode~$k$.

We first bound the expected error in estimation of the unknown transition probabilities over all arms and rounds.

\begin{lemma} \label{lemma:rl-Pdiff}
Let $\bar{|\mathcal{X}|} := \max_{i} |\mathcal{X}^i|$. \EDcomment{check whether all definition are $:=$ and statement of equality $=$.}\NimaResponse{Checked and updated the main text.}\EDcomment{What about appendices?}\NimaResponse{Done.}Then we have 
\begin{equation*}
\sum_{k = 1}^{K} \sum_{t = 0}^{T-1} \EXP\Bigl[ \bigl\lVert \bs{P}^{\star}_t(\cdot   | \bs{X}_{k, t}, \bs{A}_{k, t}) - \bs{P}_{k, t}(\cdot   | \bs{X}_{k, t}, \bs{A}_{k, t}) \bigr\rVert_1 \Bigr] \\ \le 12 N \bar{|\mathcal{X}|} \sqrt{KT (1 + \log (KT))}. 
\end{equation*}
\end{lemma}

Finally, we bound the expected regret defined in \eqref{eq:reg-def}.
\begin{theorem} \label{thm:regret}
The expected regret defined in \eqref{eq:reg-def} under the LRAP-TS algorithm is bounded by:
\[{\REGRET}(K) \leq 12 N^2 T r_{\max}\bar{|\mathcal{X}|} \sqrt{KT (1 + \log (KT))}.\]
\end{theorem}
This result shows that the regret accumulated by LRAP-TS is sublinear in the number of episodes~$K$ and quadratic in the number of arms~$N$. We refer the reader to Section~\ref{app:lem-pdiff} in Electronic Companion and \ref{app:thm-reg} for the proofs of Lemma \ref{lemma:rl-Pdiff} and Theorem \eqref{thm:regret}.

\section{Infinite-Horizon Stationary RB (ISRB)}\label{sec:infiniteH}

In this section, we extend the RB framework to the infinite-horizon setting. In contrast to the finite-horizon formulation where the planning period is fixed to $T$ time steps, the infinite-horizon formulation considers the evolution of the system over an unbounded time period. This setting is particularly useful when the system is expected to operate indefinitely or when one is interested in the long-term behavior of the policies. To handle the infinite horizon, a discount factor $\beta \in (0,1)$ is introduced as a typical approach for convergence \NimaEdits{\citep{puterman2014markov}}. In the following, we describe the infinite-horizon restless bandit process, formulate the associated planning problems in both risk-neutral and risk-aware settings, and finally introduce the learning problem for the infinite-horizon risk aware RB.

\subsection{Parameters Description}
An infinite-horizon stationary restless bandit arm is defined by the tuple
\(
(\beta, \mathcal{X}, \mathcal{A}, \{P(a)\}_{a\in\{0,1\}}, r, x_0),
\)
where \NimaEdits{all the parameters are already defined}. An ISRB is a collection of N independent infinite-horizon stationary restless bandit processes where the \NimaEdits{action space} is $\bs{\mathcal{A}}(M)$ as defined for FNRB.
In the infinite-horizon setting, the performance of a policy is measured by the discounted cumulative reward.

\subsection{Infinite-Horizon Planning Problem} \label{subsec:inf-problem}
In the infinite-horizon setting, the objective is to maximize the long-run discounted reward. In the risk-neutral case, the cumulative reward under a stationary Markov policy $\bs{\pi}\in\bs{\Pi}^{\infty}_M$, where $\bs{\Pi}^{\infty}_M$ is the set of all such policies, is defined as
\[
\bs{J}_{\bs{x}_0}(\bs{\pi}) := \sum_{i\in\mathcal{N}} \sum_{t=0}^{\infty} \beta^t r^i\bigl(X^i_t, A^i_t\bigr) \biggm|_{\bs{\pi}, {\bs X}_0=\bs{x}_0}.
\]
The risk-neutral infinite-horizon RB problem is then defined as follows.

\begin{ISRB}[(Infinite-horizon Stationary RB)]\label{prob:inf-risk-neutral}
Given a set of $N$ arms 
\(
(\beta, \mathcal{X}^i,\mathcal{A},\{P^i(a)\}_{a\in\{0,1\}}, r^i, x^i_0)
\), $i\in\mathcal{N}$
where at most $M$ arms can be activated at each time step, find a time-dependent deterministic policy $\bs{\pi}\in\bs{\Pi}^{\infty}_{M}$ that maximizes 
\(
\mathbb{E}[\bs{J}_{\bs{x}_0}(\bs{\pi})].
\)
\end{ISRB}

To incorporate risk-sensitivity, we consider a set of Lipschitz continuous non-decreasing utility functions $\{U^i\}_{i\in\mathcal{N}}$. In the risk-aware setting, the performance of a policy is measured by applying the utility function to the discounted cumulative reward of each arm as follows:

\[
\bs{D}_{{\bs x}_0}(\bs{\pi}) := \sum_{i\in\mathcal{N}} U^i\Bigl(\sum_{t=0}^{\infty} \beta^t r^i\bigl(X^i_t,A^i_t\bigr)\Bigr) \biggm|_{\bs{\pi}, {\bs X}_0=\bs{x}_0}.
\]

Let $\bs{\Pi}^{\infty}_H$ be the set of all history-dependent policies. Finally, the infinite-horizon risk-aware RB problem is defined as follows.

\begin{RISRB}[(Risk-Aware ISRB)]\label{prob:inf-risk-aware}
Given a set of Lipschitz continuous non-decreasing utility functions $\{U^i\}_{i\in\mathcal{N}}$, and a set of $N$ arms 
\(
(\beta, \mathcal{X}^i,\mathcal{A},\{P^i(a)\}_{a\in\{0,1\}}, r^i, x^i_0)
\), $i\in\mathcal{N}$
where at most $M$ arms can be activated at each time step, find a history-dependent policy $\bs{\pi}\in\bs{\Pi}^{\infty}_H$ that maximizes $\mathbb{E}\left[ \bs{D}_{{\bs x}_0}(\bs{\pi}) \right]$.
\end{RISRB}

\NimaEdits{To find the solution to the infinite-horizon problem, we first consider the finite-horizon discounted setup. Although discounted finite-horizon RB can be modeled as a special case of FNRB, that is, by the assumption that the transition probabilities are time-invariant and the per-step reward is set as $r_t(x, a) = \beta^t r(x, a)$, we will need to formally define the value function and Bellman equation
to facilitate the infinite-horizon analysis.}

\subsubsection{The Case of the Risk Neutral ISRB} \label{subsec:inf-overview}
As in previous case, we apply a relaxation to the constraint of activating at most $M$ arms at a time as follows:
\begin{align*} 
\max_{{\bs \pi} \in {\bs \Pi}^{\infty}_M} \mathbb{E}\left[{\bs J}_{{\bs x}_0}(\bs{\pi}) \right] \text{ s.t. } \lim_{T\to\infty} \EXP\biggl[ \sum_{t = 0}^{T} \beta^t \lVert {\bs A}_t \rVert_1 \bigg| {\bs X}_0 = \bs{x}_0 \biggr] \leq \dfrac{M}{1 - \beta}. 
\end{align*}
This relaxation decouples the problem into $N$ independent optimization problems via Lagrangian relaxation with multiplier $\lambda \in \mathbb{R}_+$:
\begin{equation}
\max_{{\bs \pi} \in {\bs \Pi}^{\infty}_M} \mathbb{E}\left[{\bs J}_{{\bs x}_0}(\bs{\pi}) \right] - \lambda \left(\lim_{T\to\infty} \EXP\biggl[ \sum_{t = 0}^{T} \beta^t \lVert {\bs A}_t \rVert_1 \bigg| {\bs X}_0 = \bs{x}_0 \biggr] - \dfrac{M}{1 - \beta}\right) = \sum_{i=1}^N \max_{\pi^i \in \Pi^{\infty}_M} \bar{J}^{i}_{\lambda, x^i_0}(\pi^i) + \lambda \dfrac{M}{1 - \beta},
\label{eq:inf-lagr}
\end{equation}
where each policy function $\pi^i \in \Pi^{\infty}_M$ is now stationary, with $\pi^i : {\mathcal X}^i \to {\mathcal A}$ and where
\[
\bar{J}^{i}_{\lambda, x^i_0}(\pi^i) := \lim_{T\to\infty} \mathbb{E}\left[\sum_{t = 0}^{T} \beta^t \left( r^i\left(X^i_t, A^i_t\right) - \lambda  A^i_t \right) \biggm| {X^i_0=x^i_0}\right].
\]
\NimaEdits{Note $\lambda M/(1 - \beta)$ is a constant and can be dropped. Then, each} arm’s optimal policy $\pi^{i*}_\lambda$ is derived by dynamic programming on an MDP parameterized by $(\mathcal{X}^i, \mathcal{A}, \{P^i(a)\}_{a \in \{0, 1\}}, \bar{r}^{i}_{\lambda}, x_0^i)$, with $\bar{r}^i_{\lambda}(x,a) := r^i(x,a) - \lambda a$. One then assembles these to form a solution to \eqref{eq:inf-lagr} as $\bs{\pi}^{*}_\lambda :=(\pi_\lambda^{1*}, \dots, \pi_\lambda^{N*})$.

\subsubsection{Revisiting the Finite-Horizon Risk-Aware Setting}
For ease of notation, we drop the superscript $i$ in our analysis. In the risk-aware setting, we first consider a finite time horizon \(T\) hence employ the following objective for the arm:
\[
J_{\lambda, T}(x_0) := \max_{\pi\in\Pi_H^T}\mathbb{E}\Biggl[ U\Bigl(\sum_{t=0}^{T-1}\beta^t r(X_t,A_t)\Bigr)
-\lambda \sum_{t=0}^{T-1}\beta^t A_t \Bigg| X_0=x_0 \Biggr].
\]
where $\Pi^T_H$ is the set of history-dependent policies. In a similar fashion as in \cite{bauerle2014more}, we consider an augmented state space for the MDP to cast the problem as a risk-neutral MDP. Namely, consider
\[
\hat{\mathcal{X}} = \mathcal{X}\times \mathbb{R}_+ \times (0,1],
\]
where, for a state \((x,y,z)\in\hat{\mathcal{X}}\), $x$ captures the initial state of the MDP, $y$ captures some accumulated (discounted) rewards, and $z$ captures some accumulated discount. Specifically, we define for any \(n=0,1,\dots, T\):
\begin{equation}
    V_{\lambda, n}(x,y,z) := \max_{\pi\in{\Pi}^n_H} \mathbb{E}^{\pi}_x \Biggl[ U\Bigl(y+\sum_{k=0}^{n-1}z\beta^k r(X_k,A_k)\Bigr) - \lambda \sum_{k=0}^{n-1}z\beta^k A_k \Biggr]\,,\,\forall (x,y,z)\in\hat{\mathcal{X}},\label{eq:infDiscVDef}
\end{equation} 
where $\mathbb{E}^{\pi}_x[h(X_{0:n-1},A_{0:n-1})]:=\mathbb{E}[h(X_{0:n-1},A_{0:n-1})|X_0=x]$ with all actions are drawn from $\pi$, and define \(V_{\lambda, T}(x,0,1)=J_{\lambda, T}(x)\) for all $x \in \mathcal{X}$.

Define the Bellman operators, for all $v:\hat{\mathcal{X}}\rightarrow \mathbb{R}$ and 
decision rule $f:\hat{\mathcal{X}}\to \mathcal{A}$:
\begin{align*}
    (\BellOp_{f,\lambda} v)(x,y,z) &:= \sum_{x'\in\mathcal{X}} \Bigl[ v\Bigl(x',  y+z r(x,f(x,y,z)),  z\beta\Bigr) - \lambda z f(x,y,z) \Bigr] P(x'|x,f(x,y,z)),\\[1mm]
    (\BellOp_\lambda v)(x,y,z) &:= \max_{a\in\mathcal{A}} \sum_{x'\in\mathcal{X}} \Bigl[ v\Bigl(x',  y+z r(x,a),  z\beta\Bigr) - \lambda z a \Bigr] P(x'|x,a).
\end{align*}
Given any $v:\hat{\mathcal{X}}\rightarrow \mathbb{R}$, we call a decision rule $f^*_\lambda$ a maximizer of $\BellOp_\lambda v$ if $\BellOp_{f^*,\lambda} v=\BellOp_\lambda v$.

\begin{theorem}\label{thm:discounted-finite}
The following hold:
\begin{enumerate}
    \item For \(n=1,\dots,T\), we have that $V_{\lambda, n} = \BellOp_\lambda V_{\lambda, n-1}$, with $V_{\lambda, 0}(x,y,z)=U(y)$.
    \item For \(n=1,\dots,T\), let \((f_{\lambda, 1}^*,f_{\lambda, 2}^*,\dots,f_{\lambda, n}^*)\) be any sequence of decision rules such that $\BellOp_{f_{\lambda, k}^*,\lambda} V_{\lambda, k-1}=V_{\lambda, k}$, for $k=1,\dots,n$. Given a reference $(\hat{x},\hat{y},\hat{z})$, the history dependent policy $\pi_\lambda^*=(\pi_{\lambda, 0}^*,\pi_{\lambda, 1}^*,\dots,\pi_{\lambda, n-1}^*)$ constructed via
    \begin{align*}
        \pi_{\lambda, 0}^*(\hat{x}) &:= f_{\lambda, n}^*(\hat{x},\hat{y},\hat{z}),\\[1mm]
        \pi_{\lambda, k}^*(h_k) &:= f_{\lambda, n-k}^*\Bigl(x_k,  \hat{y}+\hat{z}\sum_{{k'}=0}^{k-1}\beta^{k'} r(x_{k'},a_{k'}), \hat{z} \beta^k\Bigr),
    \end{align*}
    where $h_k$ is short for $(x_{0:k}, a_{0:k-1})$, achieves optimality in the definition of  $V_{\lambda, n}(\hat{x},\hat{y},\hat{z})$ (see \eqref{eq:infDiscVDef}).
\end{enumerate}
\end{theorem}

See Section \ref{app:discounted-finite} in Electronic Companion for the proof.

\subsubsection{Infinite-Horizon Extension} \label{subsec:inf-extension}

We consider the restless bandit problem with infinite time horizon and discount factor \(\beta\in(0,1)\) and a finite $\lambda \in \mathbb{R}_+
$
\begin{equation}\label{eq:infHorizonProb}
J_{\lambda, \infty}(x):=\sup_{\pi\in\Pi^{\infty}_H}\lim_{T\rightarrow \infty}\mathbb{E}_x^\pi\Biggl[ U\Bigl(\sum_{t=0}^{T-1}\beta^t r(X_t,A_t)\Bigr) - \lambda\sum_{t=0}^{T-1}\beta^t a_t  \Biggr].
\end{equation}
Similar as the previous section and the  steps in \cite{bauerle2014more}, we augment the state space only to track the current state, the accumulated (discounted) rewards, and the current accumulated discount multiplier. 

Let 
\begin{equation}
V_{\lambda, \infty}(x,y,z):=\sup_{\pi\in\Pi^{\infty}_H}\lim_{n\rightarrow \infty}\mathbb{E}_x^\pi\Biggl[ U\Bigl(y+z\sum_{t=0}^{n-1}\beta^t r(X_t,A_t)\Bigr) - \lambda\sum_{t=0}^{n-1}z\beta^t a_t  \Biggr].    \label{eq:infHorizonProbExt}
\end{equation}
Then, we have that \(V_{\lambda, \infty}(x,0,1)\equiv J_{\lambda, \infty}(x)\), in value and set of optimizers.

\begin{theorem} \label{thm:inf-convergence-policy2}
The function $V_{\lambda, \infty}(x,y,z)$ is such that both $\BellOp V_{\lambda, \infty} = V_{\lambda, \infty}$ and 
\[
(\BellOp^n V_{\lambda, 0})(x,y,z)\rightarrow V_{\lambda, \infty}(x,y,z),\quad\forall (x,y,z)\in\hat{\mathcal{X}}.\]
with $V_{\lambda, 0}(x,y,z)=U(y)$.
Moreover, for any decision rule \(f_\lambda^*\) such that  \(\BellOp_\lambda V_{\lambda, \infty}=\BellOp_{f_\lambda^*,\lambda}V_{\lambda, \infty}\), and given a reference $(\hat{x},\hat{y},\hat{z})\in\hat{\mathcal{X}}$, the history dependent policy \(\pi_\lambda^*=(\pi_{\lambda, 0}^*,\pi_{\lambda, 1}^*,\dots)\) defined by
\[
\pi_{\lambda, 0}^*(\hat{x})=f_\lambda^*(\hat{x},\hat{y},\hat{z}),\qquad \pi_{\lambda, n}^*(h_n)=f_\lambda^*\Bigl(x_n, \hat{y}+\hat{z} \sum_{t=0}^{n-1}\beta^t r(x_t,a_t),  \hat{z}\beta^n\Bigr) \quad \text{for } n\ge1,
\]
achieves optimality in \eqref{eq:infHorizonProbExt}.\end{theorem}

\NimaEdits{The proof relies on showing that the Bellman operator $\BellOp_\lambda$ is a $\beta$-contraction on the augmented state space $\hat{\mathcal{X}}$ (w.r.t. the $L_\infty$-norm). This is established by leveraging the Lipschitz continuity of the utility function $U$. Notably, this is a weaker requirement than the concavity assumed in prior work of \cite{bauerle2014more}, making our result more general. The Banach fixed-point theorem then applies, guaranteeing both the existence and the convergence of value iteration to a unique fixed point $V_{\lambda, \infty}$. The optimality of the greedy policy $\pi_\lambda^*$ with respect to $V_{\lambda, \infty}$ is a standard result of this framework. Full details are deferred to Section \ref{app:discounted-infinite} in Electronic Companion.}

For each individual arm $i$, by initializing the reference state in the value function definition \eqref{eq:infHorizonProbExt} as $(\hat{x}, \hat{y}, \hat{z}) = (x^i_0, 0, 1)$, $V_{\lambda, \infty}(x^i_0,0,1)$ yields $J_{\lambda, \infty}(x^i_0)$, which is the optimal value of \eqref{eq:infHorizonProbExt} for that arm given its initial state $x^i_0$. The history-dependent policy $\pi^*_{\lambda}$ provided by the theorem is optimal for this individual arm problem. Furthermore, the theorem implies that the structural properties of the potentially complex history-dependent policy $\pi^*_{\lambda}$ can be conveniently studied by analyzing the simpler, stationary decision rule $f^*_{\lambda}$ that operates on the augmented state space $\hat{\mathcal{X}}$.

\subsubsection{Two Classes of Indexable Arms for ISRB}

We start by extending the definition of indexibility presented in Definition \ref{def:idxbl_RFNRB} to an infinite horizon.

\begin{definition}[Indexability of Problem \ref{prob:inf-risk-aware}] \label{def:idxbl_RINRB}
An RISRB is \emph{indexable} if for all $i \in \mathcal{N}$, the passive set
\[
\mathcal{W}^i_{\lambda} := \bigl\{ (t,x_{0:t},a_{0:t-1})) \in \cup_{t=0}^\infty\{t\}\times{\mathcal{X}^i}^{t+1}\times\mathcal{A}^{t} : \bar{\pi}^{i*}_{\lambda,t}(x_{0:t},a_{0:t-1}) = 0 \bigr\},
\]
is non-decreasing in $\lambda$ for some sequence of optimal history dependent policy $\{\bar{\pi}^{i*}_{\lambda}\}_{\lambda\geq0}$. For an indexable RISRB, the \emph{Whittle index} is~$w^i(x_{0:t},a_{0:t-1}, t):= \inf \left\{ \lambda \in \mathbb{R}_+: (t,x_{0:t},a_{0:t-1}) \in \mathcal{W}^i_{\lambda} \right\}$.
\end{definition}

We extend the sufficient conditions for indexability to infinite-horizon problems. In an infinite horizon problem with history-dependent optimal policies, an arm can be considered indexable if the passive set or equivalently the optimal policy is non-decreasing in $\lambda$ for each possible history. In view of what was established in Theorem \ref{thm:inf-convergence-policy2} and similarly as in Section \ref{prob:ns-risk-aware}, this is the case if there exists a family of decision rules $\{f_\lambda^*\}_{\lambda\in[0,\infty)}$ that satisfies $\BellOp_\lambda V_{\lambda, \infty}=\BellOp_{f_\lambda^*,\lambda} V_{\lambda, \infty}$, and is monotone in $\lambda$. As established in Lemma 4.7.1 of \cite{puterman2014markov}, the latter can be straightforwardly confirmed when the Q-function of the augmented MDP is superadditive.

\begin{lemma}[Lemma 4.7.1 of \cite{puterman2014markov}]\label{thm:monotoneSuperAdd}
    If for all $(x,y,z)\in\hat{\mathcal{X}},0\leq \lambda_1\leq \lambda_2$, we have that
\[Q_{\lambda, \infty}(x,y,z,a):=- \lambda z a + \sum_{x'\in\mathcal{X}} \Bigl[ V_{\lambda, \infty}\Bigl(x',  y+z r(x,a),  z\beta\Bigr)  \Bigr] P(x'|x,a)\]    
    satisfies
    \begin{equation}
      Q_{\infty,\lambda_1}(x,y,z,1)-Q_{\infty,\lambda_1}(x,y,z,0)\geq Q_{\infty,\lambda_2}(x,y,z,1)-Q_{\infty,\lambda_2}(x,y,z,0),  \label{cond:superAdditiveQ}
    \end{equation}
    then there exists a family of decision rules $\{f_\lambda^*\}_{\lambda\in[0,\infty)}$ that satisfies $\BellOp_\lambda V_{\lambda, \infty}=\BellOp_{f_\lambda^*,\lambda} V_{\lambda, \infty}$, and is monotone in $\lambda$.
\end{lemma}

\begin{theorem}[Indexability of Problem \ref{prob:inf-risk-aware}]\label{thm:inf-indexable}
Problem \ref{prob:inf-risk-aware} is indexable if all restless bandit arms  satisfy  Assumption~\ref{ass:mdp}.
\end{theorem}

\textbf{\textit{Proof Sketch.}} Our proof relies on verifying that, for each restless bandit arm, the $Q_{\lambda,\infty}(x,y,z,a)$ satisfies condition \eqref{cond:superAdditiveQ}. Lemma \ref{thm:monotoneSuperAdd} then ensures the existence of a monotone family of optimal decision rules $\{f_\lambda^*\}_{\lambda\in[0,\infty)}$ for the augmented risk neutral MDP associated to the arm. The rest of the proof follows exactly as for the proof of Theorem \ref{thm:ns-indexable}.
Check Section \ref{app:proofThm6} in Electronic Companion for details.

Next, we present another set of sufficient conditions for indexability of Problem \ref{prob:inf-risk-aware}.

\begin{assumption}\label{ass:mdp3}
The MDP satisfies the condition that 
    for all $x \in \mathcal{X}$:
    \[
    \|P(\cdot|x, 0) - P(\cdot|x, 1)\|_{TV} \leq \frac{1 - \beta}{2\beta}
    \]
\end{assumption}

\begin{theorem}
\label{thm:inf-indexable2}
Problem \ref{prob:inf-risk-aware} is indexable if all restless bandit arms  satisfy condition \ref{ass:r2} and Assumption~\ref{ass:mdp3}.
\end{theorem}

\textbf{\textit{Proof Sketch.}}
As already discussed in the proof sketch of Theorem~\ref{thm:ns-indexable2}, showing indexability is equivalent to proving that the difference between the value of taking the passive action and the active action is non-decreasing in $\lambda$. In the infinite-horizon setting with discount factor $\beta$, we establish that the rate of change of value with respect to $\lambda$ lies in the interval $[-z/(1-\beta), 0]$ for any state $(x,y,z)$. The total variation bound ensures that the passive set expands monotonically with $\lambda$ and establishes indexability.

\subsection{Infinite-Horizon Learning Problem} \label{sec:inf-learning}

In this section, we extend our learning framework to infinite-horizon restless bandits with discounted rewards. While the finite-horizon setting in Section~\ref{sec:unknown} provides theoretical guarantees, the infinite-horizon discounted setting presents unique challenges for learning risk-aware policies. In fact, existing theoretical frameworks for Thompson Sampling with dynamic episodes \citep{ouyang2017learning,akbarzadeh2023learning} focus on average reward criteria and do not directly extend to discounted objectives.

\NimaEdits{To address these challenges, we propose \textit{Thompson Sampling with Dynamic Episodes} (TSDE) for infinite-horizon Risk-aware Whittle Index Policy (I-RAWIP), which adapts the dynamic episode structure from \cite{ouyang2017learning} \NimaEdits{and} \cite{akbarzadeh2023learning} to learn risk-aware policies in the discounted infinite-horizon setting.}

The algorithm maintains posterior distributions $\phi^i_k$ over the transition dynamics of each arm $i \in \mathcal{N}$. At the beginning of each episode $k$, the algorithm samples parameters $\theta^i_k \sim \phi^i_k$ for all arms and computes the corresponding risk-aware Whittle indices. These indices are then used to make risk-aware decisions throughout the episode.

\NimaEdits{The episode length is determined dynamically. Each episode $k$ continues until one of two stopping triggers is met:
\begin{enumerate}
    \item \textit{Time-based trigger}: The episode ends if its current length ($t - t_k$) exceeds the length of the previous episode ($T_{k-1}$). \EDcomment{I dont get this. The first episodes are obviously short then this criteria prevents them from increasing... the opposite of what you write. Do you mean that you stop when both of these conditions are met rather than when the first one occurs?}\NimaResponse{These conditions are according to \url{https://arxiv.org/pdf/1709.04570}. This condition ensures that the episode length grows at a linear rate. If any condition is met (any first one), the end of an episode is triggered.}
    \item \textit{Visit-based trigger}: The episode ends if the total number of visits to any state-action pair $(x,a,i)$ doubles from its count at the episode's start, i.e., $N^i_t(x,a) \geq 2N^i_{t_k}(x,a)$ for some $(x,a,i)$.
\end{enumerate}}

These criteria ensure that the algorithm collects sufficient data within each episode while adapting to the learning progress. The complete algorithm is presented in Algorithm~\ref{Alg:ISRB-TSDE-Corrected}.

\begin{algorithm}[h!]
\small
\caption{I-RAWIP/WIP - TSDE}
\begin{algorithmic}[1]
\State \textbf{Input:} Initial states $\{x^i_0\}_{i \in \mathcal{N}}$, priors $\{\phi^i_1\}_{i \in \mathcal{N}}$
\State \textbf{Initialize:} $t \leftarrow 1$, $t_1 \leftarrow 1$, $T_0 \leftarrow 0$, $N^i(x,a) \leftarrow 0$ for all $i,x,a$
\For{episodes $k = 1, 2, \ldots$}
    \State $t_k \leftarrow t$
    \State Sample $\theta^i_k \sim \phi^i_k$ for each arm $i \in \mathcal{N}$
    \State Compute RAWIP/WIP $w^i_k(\cdot)$ using sampled MDP parameters
    \State $N_{t_k}^i(x,a) \leftarrow N^i(x,a)$ for all $(x,a,i)$
    
    \While{$t \leq t_k + T_{k-1}$ AND $N^i(x,a) < \max(1, 2 \cdot N_{t_k}^i(x,a))$ for all $(x,a,i)$}
        \State Coordinator collects Whittle indices $w^i_k(\cdot)$ from all arms
        \State Activate $M$ arms with highest indices: $a^i_t = 1$ if $i \in \text{top-}M$, else $a^i_t = 0$
        \State Observe transitions: $x^i_{t+1} \sim P^i_{\theta^{\star i}}(\cdot | x^i_t, a^i_t)$ for all $i$
        \State Update visit counts: $N^i(x^i_t, a^i_t) \leftarrow N^i(x^i_t, a^i_t) + 1$
        \State $t \leftarrow t + 1$
    \EndWhile
    
    \State $T_k \leftarrow t - t_k$
    \State Update posteriors $\phi^i_{k+1}$ using observed transitions via equation \eqref{eqn:phi-update}
\EndFor
\end{algorithmic}
\label{Alg:ISRB-TSDE-Corrected}
\end{algorithm}

Similar to the non-stationary setting, we use Dirichlet priors for the transition probabilities, which provide conjugate updates. For each arm $i$ and state-action pair $(x,a)$, we maintain counts $N^i(x,a,x')$ of observed transitions. The posterior after observing a transition from $x$ to $x'$ under action $a$ is:
\begin{equation*}
    \hat{P}^i_t(\cdot | x, a) \sim \text{Dirichlet}(N^i_t(x,a,1), \ldots, N^i_t(x,a,{|\mathcal{X}^i|}))
\end{equation*}
Given sampled parameters $\theta^i_k$, we compute the risk-aware Whittle indices by solving the auxiliary optimization problem described in Section~\ref{subsec:inf-problem}.

Finally, as episodes progress, the posterior distributions concentrate around the true parameters. The dynamic episode structure ensures that early episodes are short (allowing rapid initial learning) while later episodes become longer (enabling exploitation of learned models). This adaptive behavior is particularly important in the risk-aware setting, where accurate estimation of the reward distribution is crucial for making appropriate risk-sensitive decisions.

\NimaEdits{While we cannot provide regret guarantees for the infinite horizon discounted setting, the dynamic episodes offer two key advantages: (1) they allow the agent to commit to a policy for multiple time steps, reducing the variance in policy updates, and (2) they naturally balance exploration and exploitation by ensuring sufficient data collection before policy changes. Additionally, the algorithm} provides a practical approach for learning risk-aware policies in restless bandit problems. The algorithm naturally extends the finite-horizon approach by maintaining the exploration-exploitation balance through Thompson sampling while adapting to the infinite-horizon nature through dynamic episodes.

\section{Numerical Analysis} \label{sec:numerical}
We evaluate the risk-aware Whittle index policy through numerical experiments to validate the robustness and efficacy of our models in both planning and learning contexts for both finite and infinite horizons. Code is available at \cite{Anonymous}.

\begin{figure*}[h!]
\centering
\begin{subfigure}{0.32\textwidth}
\centering
\includegraphics[width=\linewidth]{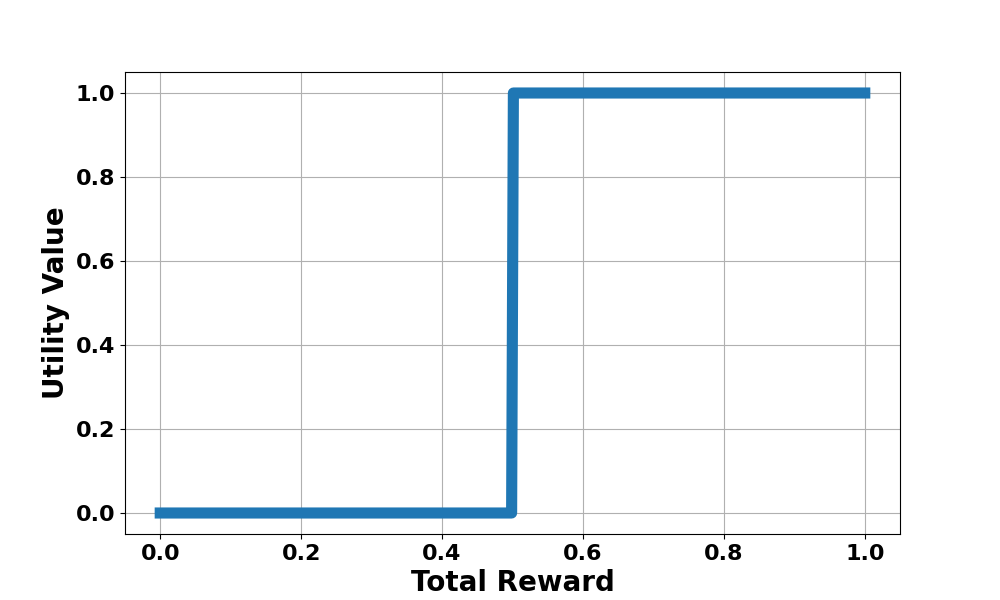}
\caption{\tiny{$(\alpha, \tau) = (1, 0.5)$}}
\end{subfigure}
\hfill
\begin{subfigure}{0.32\textwidth}
\centering
\includegraphics[width=\linewidth]{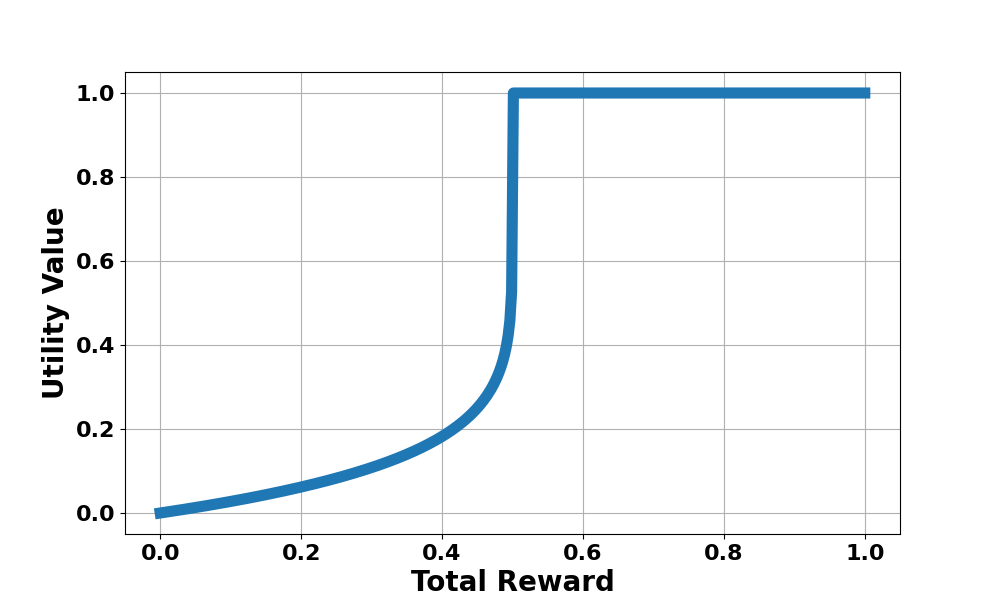}
\caption{\tiny{$(\alpha, \tau, o) = (2, 0.5, 8)$}}
\end{subfigure}
\hfill
\begin{subfigure}{0.32\textwidth}
\centering
\includegraphics[width=\linewidth]{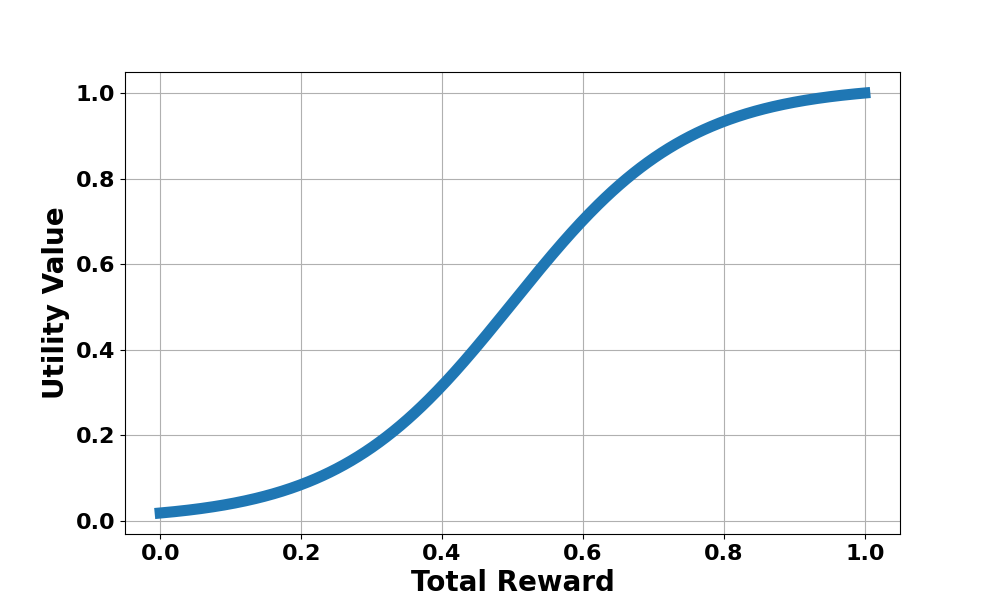}
\caption{\tiny{$(\alpha, \tau, o) = (3, 0.5, 8)$}}
\end{subfigure}
\caption{Sample plots of the three utility functions.}
\label{fig:utility}
\end{figure*}

\subsection{Planning} \label{sec:numexp:planning}

\NimaEdits{In this section, we consider a machine maintenance planning problem where each arm represents a machine that can be in one of $|\mathcal{X}|$ states, ordered from the worst 
to the best. 
This problem naturally fits the restless bandit framework: machines deteriorate over time regardless of whether they receive maintenance, and the decision maker faces a budget constraint that limits the number of machines that can be actively serviced at each time step \citep{glazebrook2005index}. At each time step, the decision maker must select which machines to intervene on (e.g., through preventive maintenance or repair) versus which one to passively monitor, balancing the immediate costs of intervention against the long-term benefits of maintaining machine reliability. The resource constraint reflects real-world limitations such as maintenance crew capacity, budget restrictions, or service time constraints. Risk-awareness is particularly critical in this setting because machine failures can lead to severe consequences. A risk-neutral approach that maximizes expected reward might recommend deferring maintenance on multiple machines simultaneously, potentially exposing the system to severe scenarios where many machines fail concurrently. In contrast, a risk-aware objective penalizes such downside risks, favoring policies that maintain a more robust operational state \citep{calabro2024emerging}.

The per-step reward for each arm is a function of both the current state and time, designed to reflect the discounted value of operating a machine in a particular condition. Specifically, rewards are structured as $r_t(x) = (1-\beta)\beta^t \rho(x) / (1 - \beta^T)$ for finite-horizon and as $r(x) = (1-\beta) \rho(x)$ for infinite-horizon, where $\rho(x)$ is linearly increasing from 0 to 1 across states. This normalization ensures that rewards remain comparable across different horizon lengths and properly account for the time-value of future returns. The transition dynamics for each arm are governed by structured probability matrices $\{P(a;p^i)\}_{a \in \{0, 1\}}$, parameterized by a scalar $p^i \in [0, 1/|\mathcal{X}^i|]$, capturing the persistence of machine $i$ conditions under different actions. The active action tends to maintain or improve machine condition, while the passive action allows for degradation toward the worst state. The detailed structure of these transition matrices is provided in Section ~\ref{app:models} in Electronic Companion and can be confirmed to  adhere to the assumptions in Theorem~ \ref{thm:ns-indexable}. Across the $N$ arms, we vary the parameter $p^i$ by linearly spacing values within $[0.1/|\mathcal{X}|, 1/|\mathcal{X}|]$, creating heterogeneity in transition dynamics while maintaining structural similarity—reflecting realistic scenarios where multiple machines exhibit similar degradation patterns but differ in their specific transition rates.}

Inspired by the work of \cite{Tversky:1979:Prospect}, we consider three S-shaped utility functions for the arms, also shown in Fig.~\ref{fig:utility}. For arm~$i$, they are
\begin{equation*}
U^i_{\alpha, \tau, o}(J^i) = 
\begin{cases} 
\mathbb{I}(J^i-\tau, 0) & \text{if } \alpha = 1 \\
1-\tau^{-1/o} \max(0, \tau-J^i)^{1/o} & \text{if } \alpha = 2 \\
(1 + e^{-o(1 - \tau)}) / (1 + e^{-o(J^i - \tau)}) & \text{if } \alpha = 3
\end{cases}
\end{equation*}
where the risk attitude is the same across machines and parametrized with model type $\alpha$, target value~$\tau$, and the order of non-linearity~$o$. All three functions capture a certain urge to reach a targeted productivity level $\tau$ (convexity below $\tau$ and to secure it (concavity above $\tau$). In particular, the first function reduces to maximizing the average probability,  across machines, that a machine reaches $\tau$. 
While $U^i_{\alpha, \tau, o}(J^i)$ violates Assumption \ref{ass:mdp}, we observed empirically that all the instances of Problem \ref{prob:ns-risk-aware} produced in our experiments were confirmed to be indexible. This provides strong support to our conjecture that $P(a)$ plays a stronger role than $U^i$ in giving rise to this important property.

We explored $1134$ setups to analyze the behavior of our risk-aware policy in context of finite-horizon. The instances have been created out of all combinations of the following parameters:
time horizon $T = 5$, discount factor $\beta \in \{0.8, 0.9, 0.99\}$,
state space size $|\mathcal{X}| \in \{3, 4, 5\}$,
number of arms $N \in \{3|\mathcal{X}|, 4|\mathcal{X}|, 5|\mathcal{X}|\}$,
utility functions in $\{(\alpha=1), (\alpha=2, o=4), (\alpha=2, o=8), (\alpha=2, o=16), (\alpha=3, o=4), (\alpha=3, o=8), (\alpha=3, o=16)\}$,
threshold $\tau \in \{0.5, 0.6, 0.7\}$,
number of arms to be activated in $\{\lfloor.1 N\rfloor , \lfloor.3 N\rfloor\}$, and the size of the augmented state space is set to $50$.

\NimaEdits{For infinite-horizon case, we analyzed $1134$ setups with the same sets of parameters as in finite-horizon case except for the threshold set which is set to $\tau \in \{0.3, 0.4, 0.5\}$. In addition, due to computational reasons, in evaluation simulations \EDcomment{Nima, please confirm} \NimaResponse{Confirmed!} we limit the time horizon to $100$ and also set the size of the augmented state space for the cumulated discount to $100$ as well.}

\begin{figure}[b!]
\centering
\begin{subfigure}[b]{0.45\textwidth}
\centering
\includegraphics[width=\textwidth]{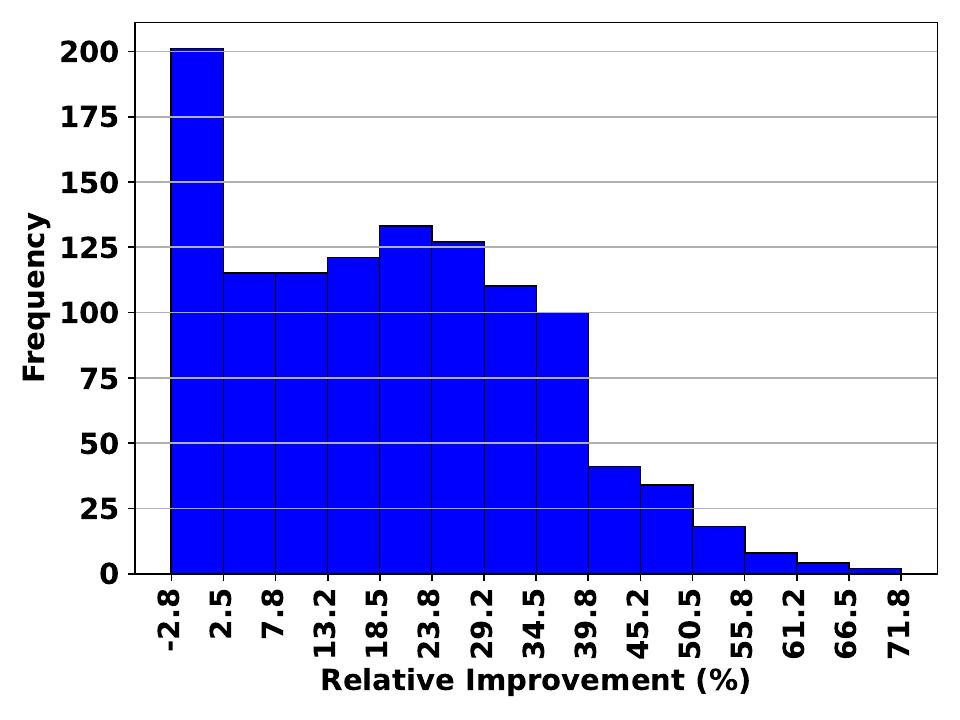}
\caption{Finite-horizon}
\end{subfigure}
\hfill
\begin{subfigure}[b]{0.45\textwidth}
\centering
\includegraphics[width=\textwidth]{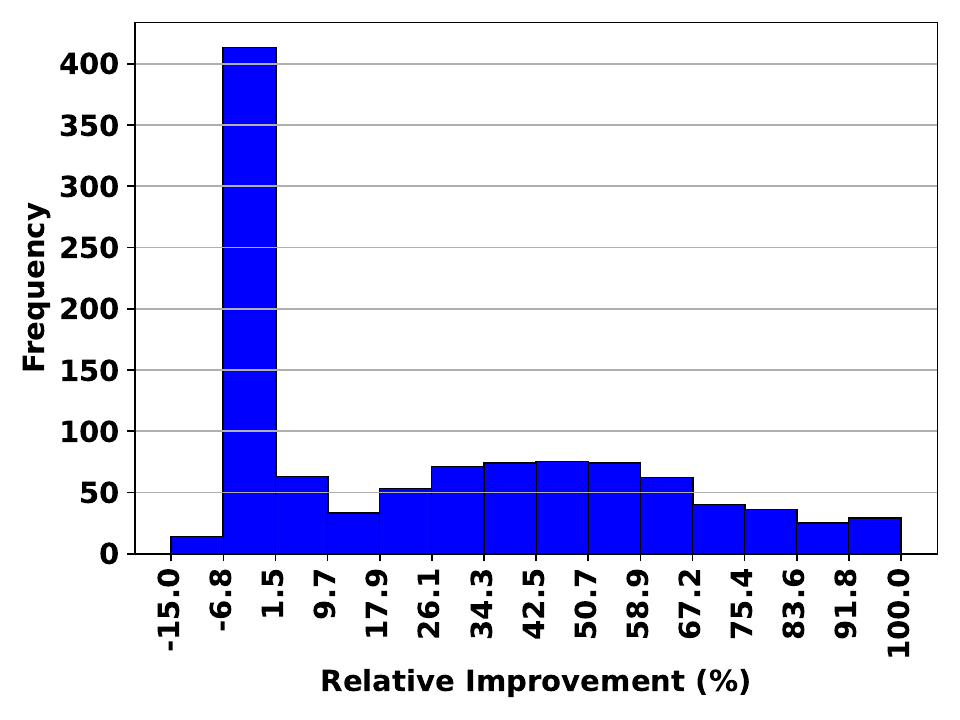}
\caption{Infinite-horizon}
\end{subfigure}
\caption{Distribution of relative improvements in the objective function achieved by our proposed policy compared to the risk-neutral one in $1134$ different setups for (a) finite-horizon and (b) infinite-horizon settings. For presentation clarity, the distributions are truncated on the x-axis up to $75$ and $100$ for finite-horizon and infinite-horizon, respectively.}
\label{fig:results}
\end{figure}

In all experiments, each policy's performance was evaluated using Monte Carlo simulations, averaging over 200 sample paths. We measure the performance of our proposed risk-aware Whittle index policy relative to the risk-neutral version.

Fig.~\ref{fig:results} summarizes the relative improvement in the objective function achieved by our policy in histograms for both finite-horizon and infinite-horizon problems. Some statistics are reported in Tables~\ref{tab:comparison} and \ref{tab:comparison-inf}.

\cite{mate2021risk} addressed a binary state partially-observable restless bandit problem, while our work considers a fully-observable setup. Although the two setups differ fundamentally, we include a numerical baseline inspired by \NimaEdits{\cite{mate2021risk}} who maximize the expected sum of stage-wise utilities instead of the expected utility of the sum of stage-wise performances. This baseline (named \textit{Sum of Stage-wise Utilities Policy} (SSUP)) employs objective 
\(\sum_{t=0}^{T-1} \frac{1}{T} U^i_{\alpha, \tau, o}(T r_t)\)
for finite-horizon setting and is modified to: 
\[\sum_{t=0}^{\infty} \beta^t (1 - \beta) U^i_{\alpha, \tau, o}(r_t / (1-\beta))\]
for infinite-horizon setting.

Table \ref{tab:comparison} indicates that our risk-aware Whittle index policy outperforms the baseline in all the provided statistics, highlighting its superior performance in finite-horizon setups. Similar results are shown in Table \ref{tab:comparison-inf} for infinite-horizon setting.

\begin{table*}[t!]
\centering
\small 
\begin{tabular}{@{}lcccc@{}}
\toprule
Policy & Min (\%) & Max (\%) & Avg (\%) & \% Above 0 \\ \midrule
RAWIP & -2.82 & 182.6 & 20.3 & 95.3 \\ 
SSUP & -28.7& 132.2 & 0.48 & 49.1 \\
\bottomrule
\end{tabular}
\caption{Relative improvement of \NimaEdits{RAWIP and SSUP} compared to the risk-neutral Whittle index policy in finite-horizon.}
\label{tab:comparison}
\end{table*}

\begin{table*}[t!]
\centering
\small 
\begin{tabular}{@{}lcccc@{}}
\toprule
Policy & Min (\%) & Max (\%) & Avg (\%) & \% Above 0 \\ \midrule
I-RAWIP & -14.9 & 202.6 & 32.87 & 64.5 \\ 
SSUP & -10.4 & 9.62 & 0.01 & 55.4 \\
\bottomrule
\end{tabular}
\caption{Relative improvement of \NimaEdits{I-RAWIP and SSUP} compared to the risk-neutral Whittle index policy in infinite-horizon.}
\label{tab:comparison-inf}
\end{table*}

\begin{table}[h!]
\centering
\small 
\begin{tabular}{@{}lcccc@{}}
\hline
\textbf{Parameter Value} & \textbf{Relative Improvement (finite)} & \textbf{Relative Improvement (infinite)} \\
\hline
\(\alpha=1\) & 40.16\% & 61.68\% \\
\hline
\((\alpha, o) = (2, 4)\) & 19.58\% & 49.76\% \\
\((\alpha, o) = (2, 8)\) & 25.85\% & 54.05\% \\
\((\alpha, o) = (2, 16)\) & 30.48\% & 60.42\% \\
\hline
\((\alpha, o) = (3, 4)\) & 0.82\% & -1.37\% \\
\((\alpha, o) = (3, 8)\) & 5.78\% & -2.49\% \\
\((\alpha, o) = (3, 16)\) & 19.51\% & 8.02\% \\
\hline
\end{tabular}
\caption{The effect of the utility function on the relative improvement of RAWIP with respect to the risk-neutral Whittle index policy}
\label{tab:utility-performance}
\end{table}

\NimaEdits{Table \ref{tab:utility-performance} illustrates the performance of our proposed methods in finite and infinite-horizon settings for each choice of utility function ($\alpha$), and their order of non-linearity ($o$). The table presents the average performance of our algorithm (computed across multiple problem setups and iterations) as parameters change. For utility functions 2 and 3, the results show that increasing the order—moving further from the linear case—highlights the superior performance of the risk-aware policy. This is expected, as risk-aware and risk-neutral policies perform identically under a linear utility function.}

\subsection{Learning}

\NimaEdits{In this section, we evaluate the performance of our learning algorithms when the true transition probabilities are unknown. We consider two problem domains: the machine maintenance problem from Section \ref{sec:numexp:planning}, and a healthcare application involving patient treatment scheduling for advanced breast cancer. For the finite-horizon setup, we employ the model-based risk-aware Whittle index policy (Section~\ref{subsec:compute}) as reference policy to measure regret. Our learning algorithm, which combines Risk-Aware Whittle Index Policy with Thompson Sampling and Dynamic Episodes, is denoted (\textsc{RAWIP-TSDE}). For the infinite-horizon setup, given the lack of guarantees on regret, we rather directly compare the performance of our learning algorithm (I-RAWIP-TSDE, Section \ref{sec:inf-learning}) to the model-based risk-aware Whittle index policy (I-RAWIP, Section~\ref{subsec:inf-problem}) and our learning algorithm where the risk-aware Whittle index policy is replaced with the risk-neutral one (I-WIP-TSDE). 

\textbf{Machine Maintenance.} The problem setup is identical to the previous section. 
The key difference is that the true transition probabilities are now unknown to the learner, who must simultaneously learn the system dynamics while making resource allocation decisions.

\textbf{Patient Treatment Scheduling.} We also evaluate our algorithms on a healthcare application involving treatment scheduling for patients with advanced breast cancer, using dynamic progression models derived from real-world clinical data in \cite{le2016structural}. Each arm represents an individual patient whose disease state evolves over time according to a Markov chain. This problem naturally fits the restless bandit framework: patients' conditions progress regardless of treatment decisions, and healthcare providers face strict resource constraints—such as limited availability of expensive therapies, clinical staff capacity, or hospital bed constraints—that restrict the number of patients who can receive intensive treatment at each time step. Risk-awareness is particularly crucial in healthcare settings where the stakes involve patient survival and quality of life. A risk-neutral policy that maximizes expected outcomes might allocate treatments in ways that expose some patients to high probabilities of severe disease progression or mortality. In contrast, a risk-aware objective explicitly penalizes such downside risks, favoring treatment allocations that provide more equitable protection against worst-case outcomes across the patient population.

In the models from \cite{le2016structural}, patient states represent disease progression: DECEASED, PROGRESSING, RESPONDING, and STABLE, though some models include only three states (excluding RESPONDING). Patients can be treated with either Capecitabine (standard therapy) or Lapatinib+Capecitabine (a more expensive combination therapy), with the choice of treatment affecting transition probabilities between disease states. Rewards reflect the clinical value of different health states: $r(\text{DECEASED}) = 0$, $r(\text{PROGRESSING}) = 1$, $r(\text{STABLE}) = 2$ for three-state models, and $r(\text{DECEASED}) = 0$, $r(\text{PROGRESSING}) = 1$, $r(\text{RESPONDING}) = 2$, $r(\text{STABLE}) = 3$ for four-state models. These rewards are then normalized so that the accumulated reward within a horizon is bounded to $[0, 1]$ for each arm, ensuring comparability across different model structures.

The study by \cite{le2016structural} provides empirical ranges for each transition probability in the Markov chains based on clinical trial data. For our experiments, we generate problem instances by randomly sampling transition probabilities within these ranges for each patient (arm) and normalizing to ensure valid probability distributions. Importantly, for the learning problem, these true transition probabilities are unknown to the algorithm, which must learn them through interaction while making treatment allocation decisions under the budget constraint.}

\begin{figure}[h!]
\centering
\begin{subfigure}{0.325\textwidth}
\centering
\includegraphics[width=\linewidth]{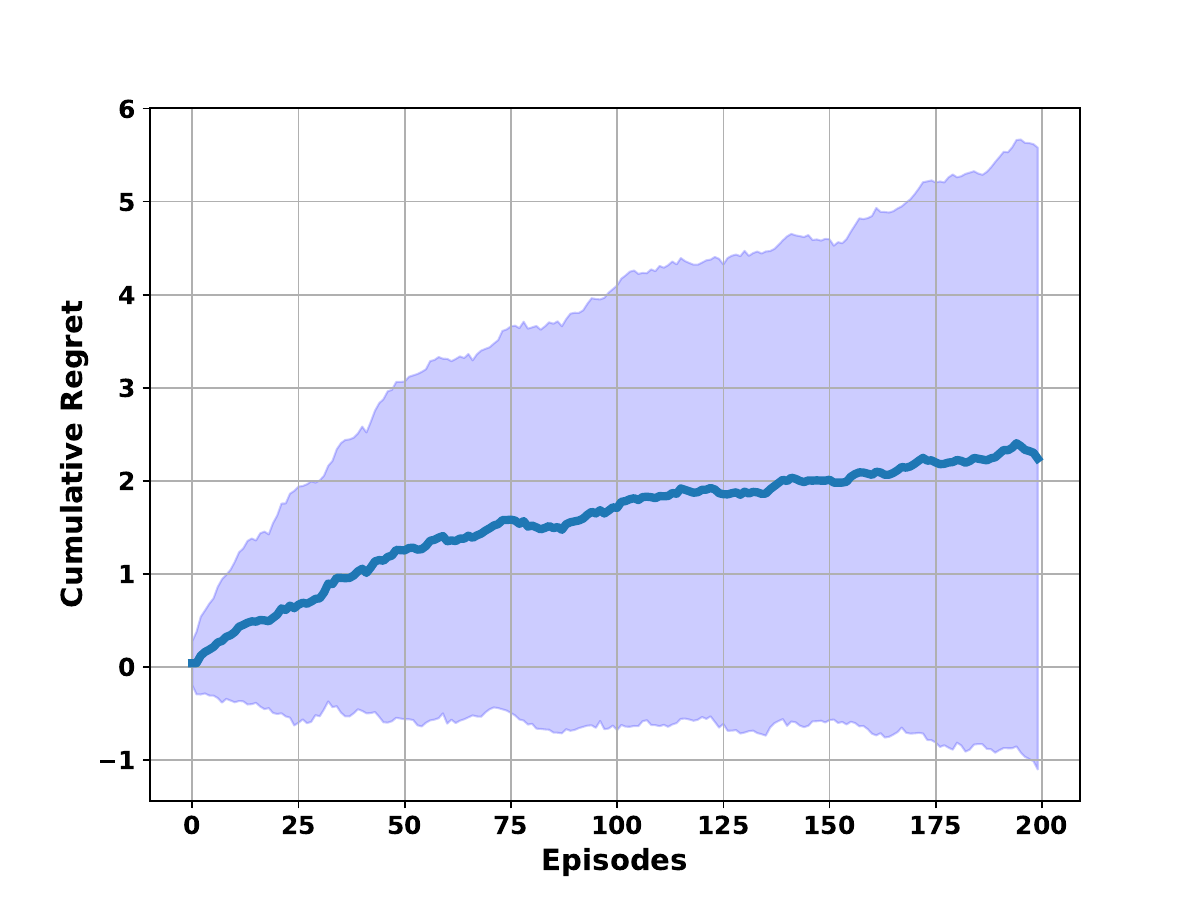}
\caption{\small $|{\cal X}| = 3$ (Model 1)}
\end{subfigure}
\centering
\begin{subfigure}{0.325\textwidth}
\centering
\includegraphics[width=\linewidth]{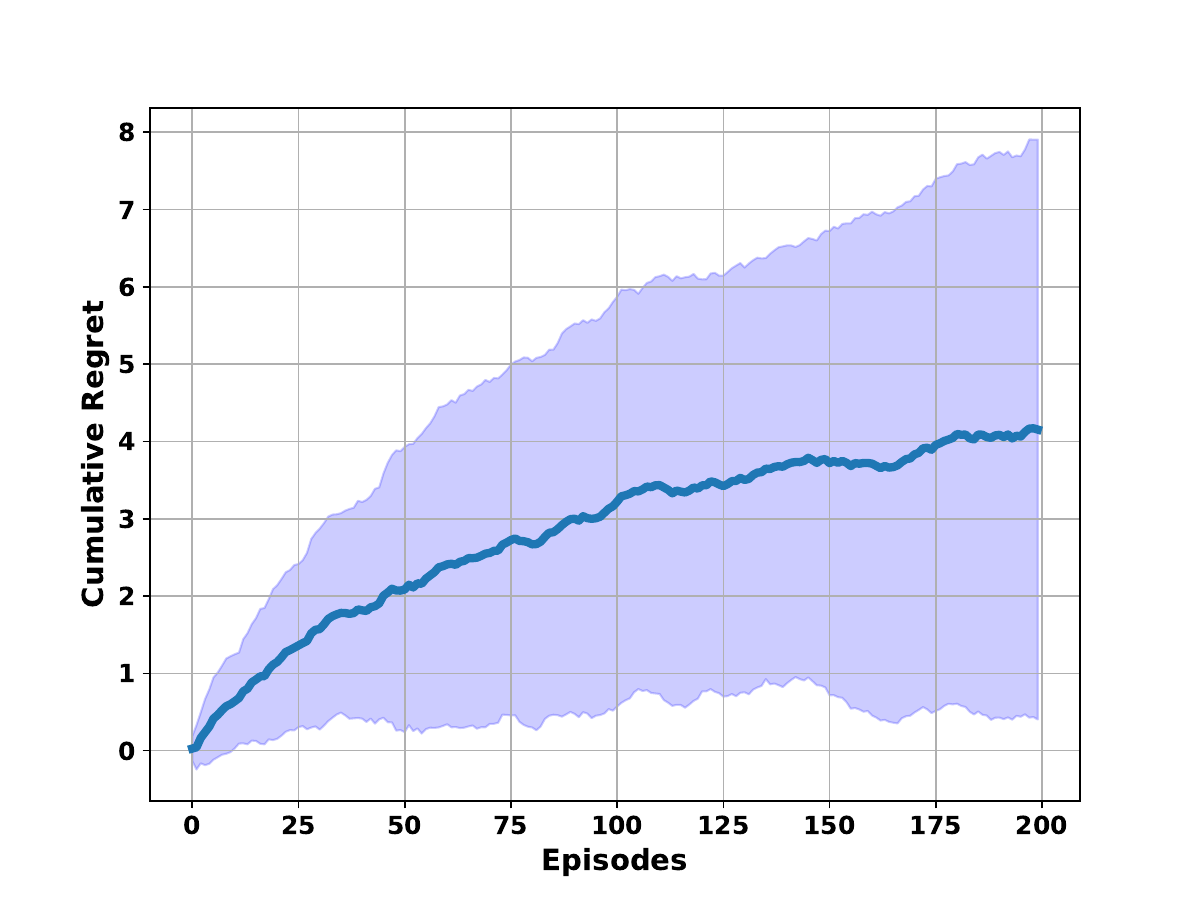}
\caption{\small $|{\cal X}| = 4$ (Model 2)}
\end{subfigure}
\centering
\begin{subfigure}{0.325\textwidth}
\centering
\includegraphics[width=\linewidth]{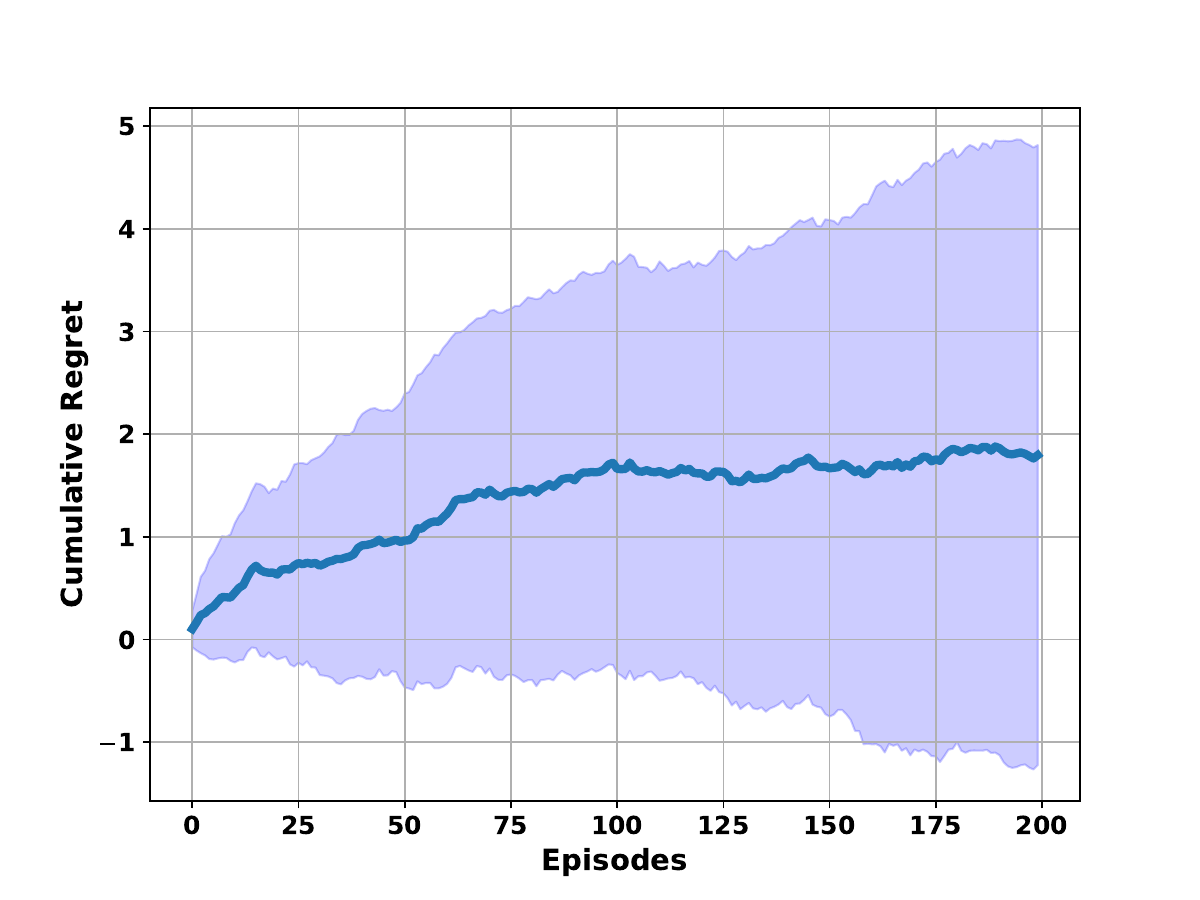}
\caption{\small $|{\cal X}| = 3$ (Model 3)}
\end{subfigure}
\caption{Cumulative regret in finite horizon for \textsc{RAWIP-TS} under \NimaEdits{three} different setups where the utility function is $(\alpha = 3, o=8, \tau=0.5)$. For all these experiments, $T=4$, $\beta=0.99$ $M = 1$, and $N = 5$, the augmented state's discretization is of size $10$, and $100$ iterations are run. Solid curve presents the expected cumulated regret.\EDcomment{I edited the caption,  please check that all is coherent. }\NimaResponse{Agreed.}}
\label{fig:regret}
\end{figure}

\begin{figure}[h!]
\centering

\begin{subfigure}{0.325\textwidth}
    \centering
    \includegraphics[width=\linewidth]{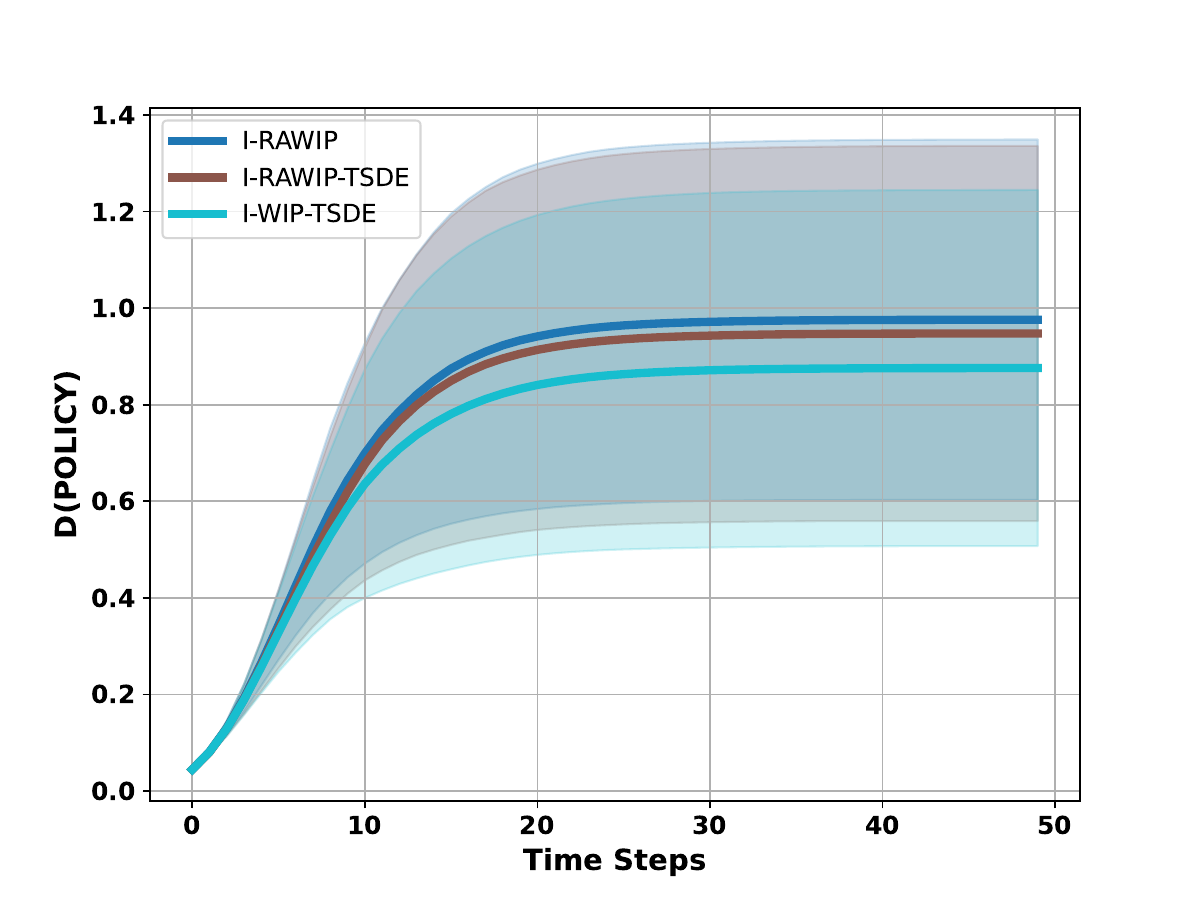}
    \caption{\small $|\mathcal{X}| = 4$ (Model 1)}
\end{subfigure}
\hfill 
\begin{subfigure}{0.325\textwidth}
    \centering
    \includegraphics[width=\linewidth]{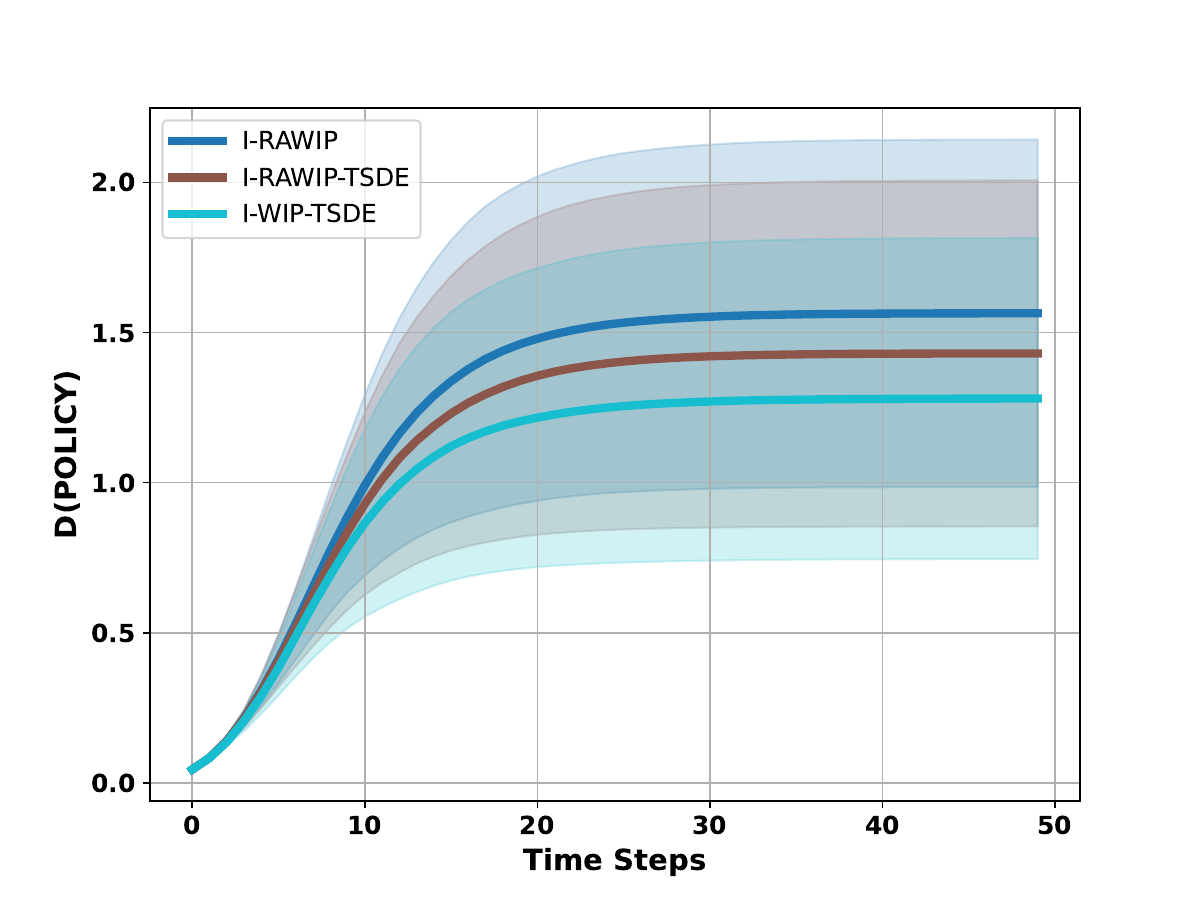}
    \caption{\small $|\mathcal{X}| = 3$ (Model 2)}
\end{subfigure}
\hfill
\begin{subfigure}{0.325\textwidth}
    \centering
    \includegraphics[width=\linewidth]{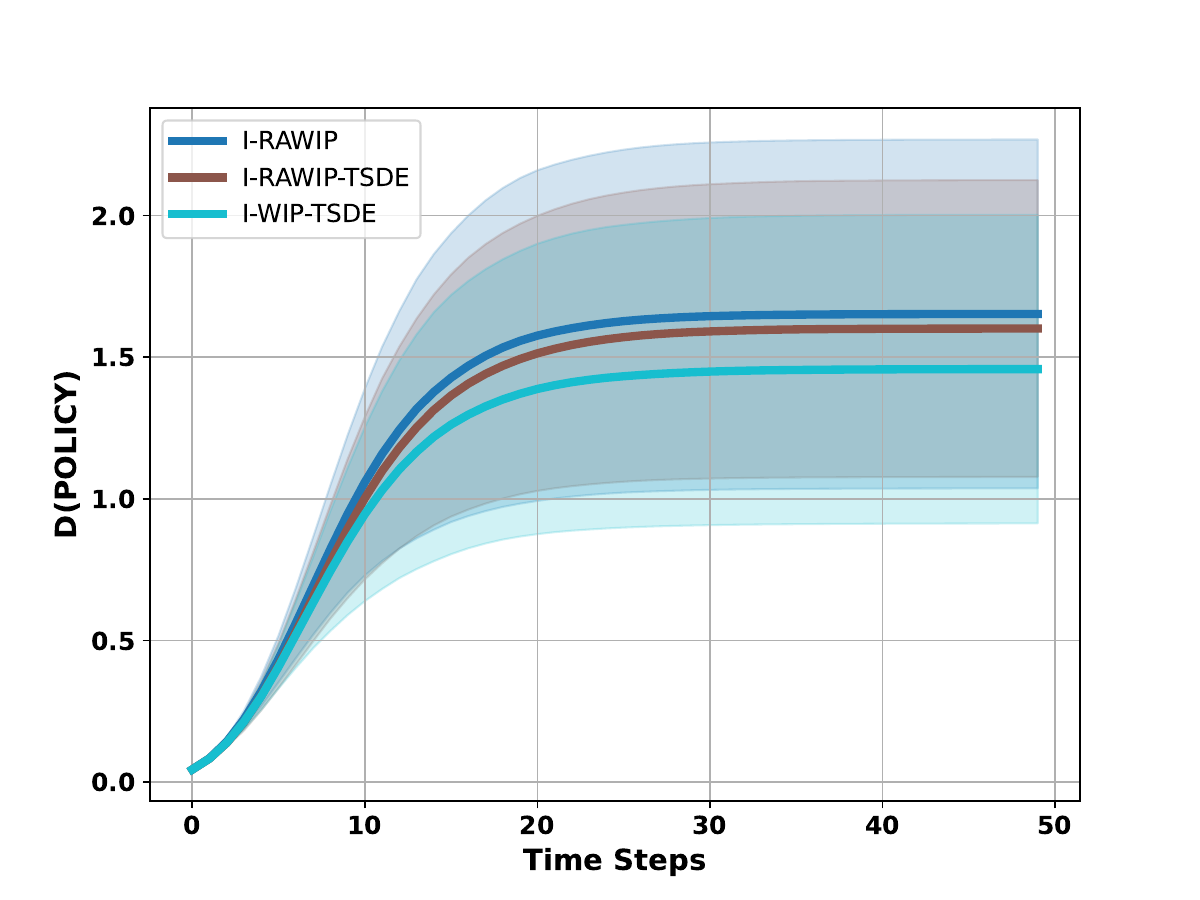}
    \caption{\small $|\mathcal{X}| = 3$ (Model 3)}
\end{subfigure} 

\caption{Cumulated performance of the infinite horizon approaches \textsc{I-RAWIP}, \textsc{I-RAWIP-TSDE}, and \textsc{I-WIP-TSDE} for \NimaEdits{three} different setups with the utility function  $(\alpha = 3, o=8, \tau=0.8)$. For all these experiments, $\beta=0.9$ $M = 3$, $N = 10$, the augmented states discretization are of size $10$ and $50$, and $200$ sample paths are run. 
\EDcomment{Why $T=50$? This is infinite horizon... Also $x$-axis says episode but the episode are dynamic so unclear how relevant the figure would be. }\NimaResponse{The x-axis naming should be updated to time-steps! Also, T=50 is the same reason as in planning...}
\EDcomment{I edited the caption,  please check that all is coherent. }\NimaResponse{Agreed.}}
\label{fig:regret-inf}
\end{figure}

At the start of each episode, the \textsc{RAWIP} is computed based on estimated parameters. During the episode, the learner observes a sample path of either finite-horizon or (a portion of) infinite-horizon rounds. State-action-state observations for each arm are counted, and Dirichlet distributions are updated over the unknown parameters. 

In Figure~\ref{fig:regret}, we illustrate $\REGRET{(k)}$ for a learner implementing \textsc{RAWIP-TS}, compared to an oracle with full knowledge of the model (i.e. transition probabilityies) who implements the finite-horizon \textsc{RAWIP}. The results suggest the learning mechanism is effective. Note as the \textsc{RAWIP} is generally suboptimal, the \textsc{RAWIP-TS} may outperform it and hence, the slope of cumulative regret can become negative. However, the learner's policy eventually converges to the \textsc{RAWIP}.

In Figure~\ref{fig:regret-inf}, we illustrate the risk-aware objective $\bs{D}(\bs{\pi}):= \mathbb{E}[\sum_{i\in\mathcal{N}} U^i(\sum_{t=0}^{\infty} \beta^t r^i(X^i_t,A^i_t)) | {\bs X}_0=\bs{x}_0]$ 
for a learner implementing \textsc{I-RAWIP-TSDE}, compared to an oracle with full knowledge of the true parameters who implements \textsc{I-RAWIP} and a baseline learner who implements \textsc{I-WIP-TSDE}, i.e., the risk-neutral learning policy from \cite{akbarzadeh2023learning}. The result indicates that \textsc{I-RAWIP-TSDE} closely follows \textsc{I-RAWIP} and outperforms \textsc{I-WIP-TSDE} in \NimaEdits{all} scenarios.

\section{Conclusion} \label{sec:conclusion}
Our study extends the traditional RB by incorporating risk-awareness, providing a robust framework for risk-aware decision-making. We establish indexability conditions for risk-aware objectives and propose a Thompson sampling approach that achieves bounded regret, scaling sublinearly with episodes and quadratically with arms. Rigorous experiments on numerous setups \ confirm the potential of our methodology in practical applications to effectively control risk exposure.
Future work could explore alternative episode structures in \NimaEdits{the learning problem of infinite-horizon stationary restless bandits} specifically designed for discounted objectives or investigate whether modified regret bounds can be established for this setting. Additionally, incorporating function approximation could enable the algorithm to handle larger state spaces, making it applicable to more complex resource allocation problems.

\section*{Acknowledgements}
Nima Akbarzadeh was partially funded by GERAD and FRQNT [https://doi.org/10.69777/352729].
Erick Delage was partially supported by the Canadian Natural Sciences and Engineering Research Council [Grant RGPIN-2022-05261] and by the Canada Research Chair program [950-230057].

\bibliographystyle{apalike}
\bibliography{mybibfile}

\removed{\section{Proof of Lemma \ref{lem:ns-conditions}} \label{app:ns-thm-0}
We start by characterizing the sufficient conditions under which the optimal
policy is increasing in the state space of the non-stationary finite horizon MDP~$(\mathcal{T}, {\cal X}, {\cal A}, P, r, \mu)$. 
\begin{proposition} \label{prop:puterman-ns}
Suppose for $t \in \mathcal{T}$ that
\begin{enumerate}
\item $r_t(x, a)$ is non-decreasing in $x, \forall a \in \mathcal{A}$, 
\item $q_t(x'|x, a)$ is non-decreasing in $x, \forall x' \in \mathcal{X}$ and $\forall a \in \mathcal{A}$, 
\item $r_t(x, a)$ is superadditive on $\mathcal{X} \times \mathcal{A}$, 
\item $q_t(x'|x, a)$ is superadditive on $\mathcal{X}^i \times \mathcal{A}$ for every $x' \in \mathcal{X}^i$,
\item $r_T(x)$ is non-decreasing in $x$,
\end{enumerate}
where $q_t(x'|x, a):= \sum_{z \geq x'} P_t^i(z|x, a)$.
Then there exist an optimal policy which is non-decreasing in $x \in \mathcal{X}$, $\forall t \in \mathcal{T}$.
\end{proposition}
See \citet[Theorem 4.7.4]{puterman2014markov}.}

\section{Proof of Lemma \ref{lem:monotone-policy}} \label{app:thm-ns-indexable}

Given any restless bandit arm, one can write the value function associated to the augmented arm risk neutral MDP as the solution to the following Bellman equation:
\[
V_{\lambda,t}(x, s) = \max_{a \in \{0,1\}} \left\{ -\frac{\lambda}{T} a + \sum_{x'} P_t(x'|x, a) V_{\lambda,t+1}(x', s + r_t(x)) \right\} \forall (x,s)\in \mathcal{X}\times \mathcal{S}_t, t\in\mathcal{T}/\{T-1\}, \lambda\in\mathbb{R}_+,
\]
with $V_{\lambda,T-1}(x,s) = \max_{a\in\mathcal{A}}U(s+r(x))$ for all $(x,s)\in \mathcal{X}\times \mathcal{S}_{T-1}$ and $\lambda\in\mathbb{R}_+$, where we dropped indexing $i$ to simplify presentation and replaced $r_t(x,a)$ with $r_t(x)$ due to Assumption \ref{ass:mdp}. To help with the presentation of our analysis and make our findings applicable to the infinite horizon setting, we consider a sequence of W functions defined recursively. Starting with $W_{T}^\beta(x,s,\phi):=U(s)$, and for $t\leq T-1$:
\[
W_t^\beta(x, s, \phi) := \max_{a \in \{0,1\}} \left\{ \frac{\phi\beta^t}{T} a + \sum_{x'} P_t(x'|x, a) W_{t+1}^\beta(x', s + \beta^t r_t(x), \phi) \right\}  \forall (x,s)\in \mathcal{X}\times \mathcal{S}_t, t\in\mathcal{T}, \phi\in\mathbb{R}_-,
\] 
Clearly, $V_{\lambda,t}(x,s)=W_t^1(x, s, -\lambda)$ for all $t\in\mathcal{T}$, 
with the set of maximizers in the definition for $W_t^1(x, s, -\lambda)$ being the same as for the set in the Bellman equation of $V_t(x,s,\lambda)$.

In proving Lemma \ref{lem:monotone-policy}, we will make use of 6 useful properties that  can be satisfied by $W_t^\beta(x,s,\phi)$. 

\begin{definition}[Value Function Conditions]\label{ass:value-properties}
We define the following conditions for a value function $W_t^\beta(x,s,\phi)$ with  $t\in\mathcal{T}$:
\begin{enumerate}[label=\ref*{ass:value-properties}.\alph*]
    \item \label{prop:w1} Non-decreasing with respect to $x$.
    \item \label{prop:w2} Convex and non-decreasing with respect to $s$.
    \item \label{prop:w3} Non-decreasing with respect to $\phi$.
    \item \label{prop:w4} Super-additive with respect to $(x,s)$.
    \item \label{prop:w5} Super-additive with respect to $(x,\phi)$.
    \item \label{prop:w6} Super-additive with respect to $(s,\phi)$.
\end{enumerate}
\end{definition}

In particular, a subset of these conditions will directly lead to the existence of an optimal policy for the augmented restless bandit arm that is monotone at time $t-1$.



\begin{lemma}\label{thm:monotonePolicy}
If the MDP satisfies Assumption \ref{ass:mdp} and $W_{t+1}^\beta(x,s,\phi)$ satisfies conditions \ref{prop:w1}, \ref{prop:w4}, \ref{prop:w5}, then there is an optimal policy for the augmented restless bandit arm that is non-decreasing with respect to $(x,s,\phi)$, with $\phi=-\lambda$ at time $t$.
\end{lemma}

Informed with Lemma \ref{thm:monotonePolicy}, the claim made in Lemma \ref{lem:monotone-policy} therefore can be reduced to establishing that Assumption \ref{ass:mdp} is sufficient to ensure that the conditions \ref{prop:w1}, \ref{prop:w4}, \ref{prop:w5} are satisfied at all $t\in\mathcal{T}$. Our work will actually demonstrate a stronger statement.

\begin{lemma}\label{thm:WtProperties}
    Given that the MDP satisfies Assumption \ref{ass:mdp}, the value function $W_t^\beta(x,s,\phi)$ satisfies the conditions \ref{prop:w1}-\ref{prop:w6} for all  $t\in\mathcal{T}$.
\end{lemma}

In the remaining subsections, we first present useful technical lemmas and follow with the proof of lemmas \ref{thm:monotonePolicy} and \ref{thm:WtProperties}.

\subsection{Technical Lemmas}

The key technical results that enable the proof are the following lemmas about a property of convex function and about conditions under which different form of super-additivity applies.

\begin{lemma}\label{thm:increasingSecant}
    Given a convex function $f(x)$, for all $\Delta\geq 0$ and $x_1\geq x_2$, we have that:
    \[f(x_2+\Delta)-f(x_2) \leq f(x_1+\Delta)-f(x_1).\]
\end{lemma}
\textit{\textbf{Proof.}}

We start by showing that given $x_1\geq x_2\geq x_3$, we must have that:
\[s_3:=\frac{f(x_2)-f(x_3)}{x_2-x_3}\leq s_2:=\frac{f(x_1)-f(x_3)}{x_1-x_3}\leq s_1:=\frac{f(x_1)-f(x_2)}{x_1-x_2}.\]
One can first use the convexity of $f(x)$ to confirm:
\begin{align*}
f(x_2)-f(x_3) &= f(\frac{x_2-x_3}{x_1-x_3}x_1+\left(1-\frac{x_2-x_3}{x_1-x_3}\right)x_3)-f(x_3)\\
&\leq \frac{x_2-x_3}{x_1-x_3}f(x_1)+\left(1-\frac{x_2-x_3}{x_1-x_3}\right)f(x_3)-f(x_3)\\
&=     \frac{x_2-x_3}{x_1-x_3}(f(x_1)-f(x_3)).
\end{align*}
Hence, $s_3\leq s_2$. Now, given that $s_2(x_1-x_3)=s_3(x_2-x_3)+s_1(x_1-x_2)$, we must have:
\begin{align*}
    &s_2(x_1-x_3)= s_3(x_2-x_3)+s_1(x_1-x_2)\leq  s_2(x_2-x_3)+s_1(x_1-x_2)\\
    &\qquad\qquad\qquad\;\Rightarrow\;s_2(x_1-x_2)\leq s_1(x_1-x_2)\;\Rightarrow\;s_2\leq s_1.
\end{align*}
Similarly
\begin{align*}
    s_3(x_1-x_3) \leq s_2 (x_1-x_3) = s_3(x_2-x_3)+s_1(x_1-x_2) \;\Rightarrow\;s_3(x_1-x_2)\leq s_1(x_1-x_2) \;\Rightarrow\; s_3\leq s_1. 
\end{align*}

Now getting back at our claim, one can identify two cases. In the first case, we have that $x_3\leq x_3+\Delta \leq x_2 \leq x_2 + \Delta$. Applying the ordering of secants established above twice over this sequence, we get:
\[\frac{f(x_3+\Delta)-f(x_3)}{\Delta}\leq \frac{f(x_2)-f(x_3+\Delta)}{x_2 -(x_3+\Delta)}\leq  \frac{f(x_2+\Delta)-f(x_2)}{\Delta},\]
concluding that $f(x_3+\Delta)-f(x_3) \leq f(x_2+\Delta)-f(x_2)$ since $\Delta\geq 0$.

Alternatively, in the second case we have $x_3 \leq x_2 \leq x_3+\Delta \leq x_2 + \Delta$. A similar argument leads to:
\[\frac{f(x_3+\Delta)-f(x_3)}{\Delta}\leq \frac{f(x_3+\Delta)-f(x_2)}{x_3+\Delta - x_2}\leq  \frac{f(x_2+\Delta)-f(x_2)}{\Delta}.\]
Hence, we have again that $f(x_3+\Delta)-f(x_3) \leq f(x_2+\Delta)-f(x_2)$.
\Halmos

\begin{lemma}[Super-additivity Preservation I]\label{lem:sa1}
Given probability functions $p_1(x), p_2(x)$ satisfying $\sum_{x\geq k} p_1(x) \geq \sum_{x\geq k} p_2(x)$ for all $k$, if $f(x,y)$ is super-additive in $(x,y)$, then for all $y_1 \geq y_2$:
\[\mathbb{E}_{p_1}[f(X,y_1)]- \mathbb{E}_{p_1}[f(X,y_2)] \geq \mathbb{E}_{p_2}[f(X,y_1)]- \mathbb{E}_{p_2}[f(X,y_2)].\]
\end{lemma}

\textbf{\textit{Proof.}} Let $q_j(k) := \sum_{x \geq k} p_j(x)$ with $j\in\{1,2\}$. The proof is as follows:
\begin{align*}
    \mathbb{E}_{p_1}&[f(X,y_1)- f(X,y_2)] \\&=  (f(1,y_1)- f(1,y_2)) + \sum_{k=2}^{|\mathcal{X}|} q_1(k) (f(k,y_1)- f(k,y_2) - (f(k-1,y_1)- f(k-1,y_2))\\
    &\geq (f(1,y_1)- f(1,y_2)) + \sum_{k=2}^{|\mathcal{X}|} q_2(k) (f(k,y_1)- f(k,y_2) - (f(k-1,y_1)- f(k-1,y_2))\\        
    & = \mathbb{E}_{p_2}[f(X,y_1)- f(X,y_2)],    
\end{align*}
where we exploited the fact that $q_1(k)\geq q_2(k)$ and the super-additivity of $f(x,y)$, which implies that 
\(f(k,y_1)- f(k,y_2)\geq f(k-1,y_1)- f(k-1,y_2).\Halmos\)

\begin{lemma}[Super-additivity Preservation II]\label{lem:sa2}
Given probability functions $p_{1,1}, p_{1,2}, p_{2,1}, p_{2,2}$ satisfying $\sum_{x\geq k} (p_{1,1}(x) - p_{1,2}(x)) \geq \sum_{x\geq k} (p_{2,1}(x) - p_{2,2}(x))$ for all $k$, if $f(x)$ is non-decreasing, then:
\[\mathbb{E}_{p_{1,1}}[f(X)] - \mathbb{E}_{p_{1,2}}[f(X)] \geq \mathbb{E}_{p_{2,1}}[f(X)] - \mathbb{E}_{p_{2,2}}[f(X)].\]
\end{lemma}


\textbf{\textit{Proof.}}  Let $q_{i,j}(k) := \sum_{x \geq k} p_{i,j}(x)$ with $(i,j)\in\{1,2\}^2$. This result follows from:
\begin{align*}
    \mathbb{E}_{p_{1,1}}&[f(X)] - \mathbb{E}_{p_{1,2}}[f(X)] \\&=  f(1) + \sum_{k=2}^{|\mathcal{X}|} (q_{1,1}(k)(f(k)- f(k-1)) - \left(f(1) + \sum_{k=2}^{|\mathcal{X}|} (q_{1,2}(k)(f(k)- f(k-1))\right)\\
    &= \sum_{k=2}^{|\mathcal{X}|} (q_{1,1}(k) - q_{1,2}(k)) (f(k)- f(k-1))\\
    &\geq  \sum_{k=2}^{|\mathcal{X}|} (q_{2,1}(k) - q_{2,2}(k)) (f(k)- f(k-1))\\
    &=\mathbb{E}_{p_{2,1}}[f(X)] - \mathbb{E}_{p_{2,2}}[f(X)]    ,
\end{align*}
where we exploited the fact that $f(k)\geq f(k-1)$ and the super-additivity of $q_{i,j}(k)$ in $(i,j)$, which implies that
\(q_{1,1}(k) - q_{1,2}(k) \geq q_{2,1}(k) - q_{2,2}(k). \Halmos\)

\subsection{Proof of Lemma \ref{thm:monotonePolicy}}

Define the $Q_t^\beta$ function as:
\[Q_t^\beta(x,s,\phi,a):=\frac{\phi\beta^t}{T} a + \sum_{x'} P_t(x'|x, a) W_{t+1}^\beta(x', s + \beta^tr_t(x), \phi).\]
One can show that $Q_t^\beta$ is super-additive with respect to $(x,a)$, $(s,a)$, $(\phi,a)$.
Starting with $(x,a)$, we have for $x_1\geq x_2$:
\begin{align*}
    Q_t^\beta&(x_1,s,\phi,1)-Q_t^\beta(x_1,s,\phi,0) = \frac{\phi\beta^t}{T}+ \mathbb{E}_{p_{1,1}}[W_{t+1}^\beta(x', s + \beta^t r_t(x_1), \phi)]-\mathbb{E}_{p_{1,2}}[W_{t+1}^\beta(x', s + \beta^t r_t(x_1), \phi)]\\
    &\geq \frac{\phi\beta^t}{T}+ \mathbb{E}_{p_{2,1}}[W_{t+1}^\beta(x', s + \beta^t r_t(x_1), \phi)]-\mathbb{E}_{p_{2,2}}[W_{t+1}^\beta(x', s + \beta^t r_t(x_1), \phi)]\\
    &\geq \frac{\phi\beta^t}{T}+ \mathbb{E}_{p_{2,1}}[W_{t+1}^\beta(x', s + \beta^t r_t(x_2), \phi)]-\mathbb{E}_{p_{2,2}}[W_{t+1}^\beta(x', s + \beta^t r_t(x_2), \phi)]\\
    &= Q_t^\beta(x_2,s,\phi,1)-Q_t^\beta(x_2,s,\phi,0), 
\end{align*}    
where $p_{1,1}(x):= P_t(x|x_1, 1)$, $p_{1,2}(x):= P_t(x|x_1, 0)$, $p_{2,1}(x):= P_t(x|x_2, 1)$, and $p_{2,2}(x):= P_t(x|x_2, 0)$. The first inequality follows from Lemma \ref{lem:sa2} given property \ref{prop:w1} and super-additivity of $p_{ij}$ (i.e. Assumption \ref{ass:p3}). The second inequality follows from Lemma \ref{lem:sa1} when we exploit Assumption \ref{ass:r1} and \ref{ass:p1} and property \ref{prop:w4}.

Following with $(s,a)$, we have for $s_1\geq s_2$:
\begin{align*}
    Q_t^\beta(x,s_1,\phi,1)-Q_t^\beta(x,s_1,\phi,0) &= \frac{\phi\beta^t}{T}+ \mathbb{E}_{p_{1}}[W_{t+1}^\beta(x', s_1 + \beta^t r_t(x), \phi)]-\mathbb{E}_{p_{2}}[W_{t+1}^\beta(x', s_1 + \beta^t r_t(x), \phi)]\\
    &\geq \frac{\phi\beta^t}{T}+ \mathbb{E}_{p_{1}}[W_{t+1}^\beta(x', s_2 + \beta^t r_t(x), \phi)]-\mathbb{E}_{p_{2}}[W_{t+1}^\beta(x', s_2 + \beta^t r_t(x), \phi)]\\ 
    &= Q_t^\beta(x,s_2,\phi,1)-Q_t^\beta(x,s_2,\phi,0), 
\end{align*} 
where $p_1(x'):=P(x'|x,1)$ and $p_2(x'):=P(x'|x,0)$ and the inequality follows from Lemma \ref{lem:sa1} using Assumption \ref{ass:p1} and property \ref{prop:w4}.

Next, for $(\phi,a)$, we have for $\phi_1\geq \phi_2$:
\begin{align*}
    Q_t^\beta(x,s,\phi_1,1)&-Q_t^\beta(x,s,\phi_1,0) \\
    &= \frac{\phi_1\beta^t}{T}+ \mathbb{E}_{p_{1}}[W_{t+1}^\beta(x', s + \beta^t r_t(x), \phi_1)]-\mathbb{E}_{p_{2}}[W_{t+1}^\beta(x', s + \beta^t r_t(x), \phi_1)]\\
    &\geq \frac{\phi_2\beta^t}{T}+ \mathbb{E}_{p_{1}}[W_{t+1}^\beta(x', s + \beta^t r_t(x), \phi_2)]-\mathbb{E}_{p_{2}}[W_{t+1}^\beta(x', s + \beta^t r_t(x), \phi_2)]\\ 
    &= Q_t^\beta(x,s,\phi_2,1)-Q_t^\beta(x,s,\phi_2,0), 
\end{align*} 
again using  Lemma \ref{lem:sa1} with Assumption \ref{ass:p1} and property \ref{prop:w5}.

The three-way super-additivity of $Q$ ensures that the policy 
\[\pi(x,s,\phi):=\min\arg\max_a Q_t^\beta(x,s,\phi,a)\]
is non-decreasing with respect to $(x,s,\phi)$. Indeed, if $\pi(x_2,s_2,\phi_2)=1$, then $Q_t(x_2,s_2,\phi_2,1)> Q_t(x_2,s_2,\phi_2,0)$, based on the definition of $\pi(x,s,\phi)$. Moreover, if $x_1\geq x_2$, $s_1\geq s_2$, and $\phi_1\geq \phi_2$, then  
\[ 0 > Q_t(x_2,s_2,\phi_2,0)- Q_t(x_2,s_2,\phi_2,1)\geq Q_t(x_1,s_1,\phi_1,0)- Q_t(x_1,s_1,\phi_1,1).\]
This implies that $Q_t(x_1,s_1,\phi_1,1)>Q_t(x_1,s_1,\phi_1,0)$ thus $\pi(x_1,s_1,\phi_1)=1$. 
\Halmos

\subsection{Proof of Lemma \ref{thm:WtProperties}}

We now turn to  showing that Assumption~\ref{ass:mdp} is sufficient for the properties \ref{prop:w1}-\ref{prop:w6} to hold for all $t\in\mathcal{T}$. We do  starting from $t=T$ and then using an inductive argument from $T-1,\dots,0$.

\begin{lemma}\label{thm:W0-inf}
    If $U(y)$ is convex and non-decreasing, then $W_0^\beta(x,y,z,\phi) = U(y)$ satisfies conditions~\ref{prop:w1}-\ref{prop:w6}. 
\end{lemma}
\textbf{\textit{Proof.}} 
Indeed, it is constant with respect to $x$ and $\phi$ (properties \ref{prop:w1} and \ref{prop:w3}). It satisfies property \ref{prop:w2} by our assumption on $U(s)$. Finally, the follow confirm the three super-additivity properties:
\begin{align*}
(\ref{prop:w4})\quad& W_T^\beta(x_1,s_1,\phi)-W_T^\beta(x_1,s_2,\phi)=U(s_1)-U(s_2) = W_T^\beta(x_2,s_1,\phi)-W_T^\beta(x_2,s_2,\phi)\\
(\ref{prop:w5}) \quad& W_T^\beta(x_1,s,\phi_1)-W_T^\beta(x_1,s,\phi_2)=U(s)-U(s) = W_T^\beta(x_2,s,\phi_1)-W_T^\beta(x_2,s,\phi_2)\\
(\ref{prop:w6})\quad& W_T^\beta(x,s_1,\phi_1)-W_T^\beta(x,s_1,\phi_2)=0 = W_T^\beta(x,s_2,\phi_1)-W_T^\beta(x,s_2,\phi_2)    \Halmos
\end{align*}

The above Lemma validates that  the properties hold for time~$T$. Now, consider for $t < T$ next.

\begin{lemma}
    If the MDP satisfies Assumption \ref{ass:mdp} and $W_{t+1}^\beta$ satisfies conditions (\ref{prop:w1}-\ref{prop:w6}), $W_{t}(x,s,\phi)$ satisfies properties \ref{prop:w1}-\ref{prop:w6} for all $t\in\mathcal{T}$.
\end{lemma}

\textit{\textbf{Proof.}}
We prove this inductively. First, Lemma \ref{lem:monotone-policy} ensures that it is the case for $W_T^\beta$. Now, given that the properties are satisfied at $t+1$, we wish to show that they also apply at $t$. 

\textbf{$W_t^\beta$ is non-decreasing in $x$:} For $x_1\geq x_2$,
\begin{align*}
    W_t^\beta&(x_1,s,\phi) = \max_{a \in \{0,1\}} \left\{ \frac{\phi\beta^t}{T} a + \sum_{x'} P_t(x'|x_1, a) W_{t+1}^\beta(x', s + \beta^t r_t(x_1), \phi) \right\}\\
    &\geq \max_{a \in \{0,1\}} \left\{ \frac{\phi\beta^t}{T} a + \sum_{x'} P_t(x'|x_1, a) W_{t+1}^\beta(x', s + \beta^t r_t(x_2), \phi) \right\}\\
    &= \max_{a \in \{0,1\}} \Big\{ \frac{\phi\beta^t}{T} a + W_{t+1}^\beta(1, s + \beta^t r_t(x_2), \phi) \\
    & \qquad + \sum_{k \geq 2} \sum_{x'\geq k} P_t(x'|x_1, a) (W_{t+1}^\beta(k, s + \beta^t r_t(x_2), \phi) - W_{t+1}^\beta(k-1, s + r_t(x_2), \phi)) \Big\} \\
    &\geq \max_{a \in \{0,1\}} \Big\{ \frac{\phi\beta^t}{T} a + W_{t+1}^\beta(1, s + \beta^t r_t(x_2), \phi) \\
    & \qquad + \sum_{k \geq 2} \sum_{x'\geq k} P_t(x'|x_2, a) (W_{t+1}^\beta(k, s + \beta^t r_t(x_2), \phi) - W_{t+1}^\beta(k-1, s + r_t(x_2), \phi)) \Big\} \\
    &=W_t^\beta(x_2,s,\phi),  
\end{align*}

where we first exploited that $W_{t+1}^\beta$ is non-decreasing in $s$, then exploited the fact that $\mathcal{X}$ is totally ordered (see Assumption \ref{ass:ordered}). Finally, we made use of Assumption \ref{ass:p2} and the non-decreasingness of $W_{t+1}^\beta$ in $x$ (i.e. property \ref{prop:w1}). 

\textbf{$W_t^\beta$ is non-decreasing and convex in $s$:}  For $s_1\geq s_2$,
\begin{align*}
    W_t^\beta(x,s_1,\phi) &= \max_{a \in \{0,1\}} \left\{ \frac{\phi\beta^t}{T} a + \sum_{x'} P_t(x'|x, a) W_{t+1}^\beta(x', s_1 + \beta^t r_t(x), \phi) \right\}\\
&\geq \max_{a \in \{0,1\}} \left\{ \frac{\phi\beta^t}{T} a + \sum_{x'} P_t(x'|x, a) W_{t+1}^\beta(x', s_2 + \beta^t r_t(x), \phi) \right\}\\        
    &=W_t^\beta(x,s_2,\phi), 
\end{align*}    
given that $W_{t+1}^\beta$ is non-decreasing in $s$. Convexity follows from the convexity of $W_{t+1}^\beta$ and the fact that $W_t^\beta(x,s,\phi)$ is the maximum of the sum of convex functions in $s$.

\textbf{$W_t^\beta$ is non-decreasing with respect to $\phi$:} For $\phi_1\geq \phi_2$,
\begin{align*}
    W_t^\beta(x,s,\phi_1) &= \max_{a \in \{0,1\}} \left\{ \frac{\phi_1\beta^t}{T} a + \sum_{x'} P_t(x'|x, a) W_{t+1}^\beta(x', s + \beta^t r_t(x), \phi_1) \right\}\\
&\geq \max_{a \in \{0,1\}} \left\{ \frac{\phi_2\beta^t}{T} a + \sum_{x'} P_t(x'|x, a) W_{t+1}^\beta(x', s + \beta^t r_t(x), \phi_2) \right\}\\        
    &=W_t^\beta(x,s,\phi_2), 
\end{align*}    
given that $W_{t+1}^\beta$ is non-decreasing in $\phi$ and $a\geq 0$. 

\textbf{Super-additive with respect to $(x,s)$:} For $x_1\geq x_2$ and $s_1\geq s_2$, let $a_{1,1},a_{1,2},a_{2,1}$, and $a_{2,2}$ be an action that achieves optimality in $W_t^\beta(x_1,s_1,\phi)$, $W_t^\beta(x_1,s_2,\phi)$, $W_t^\beta(x_2,s_1,\phi)$, and $W_t^\beta(x_2,s_2,\phi)$ respectively. Based on Lemma \ref{lem:monotone-policy}, there exists a monotone policy at time $t$. This implies that a tuple of optimal actions is necessarily in the following list:
\[(a_{1,1},a_{1,2},a_{2,1},a_{2,2})\in\{(1,1,1,1), (1, 1, 1, 0), (1,1,0,0), (1,0,0,0), (0,0,0,0), (1,0,1,0)\}.\]
We will show that in each case, we have that $W_{t}$ is super-additive with respect to $(x,s)$.

For $(1,1,1,1)$, when denoting $p_{1,1}(x):= P_t(x|x_1, 1)$, $p_{1,2}(x):= P_t(x|x_1, 0)$, $p_{2,1}(x):= P_t(x|x_2, 1)$, and $p_{2,2}(x):= P_t(x|x_2, 0)$, we can see that:\EDcomment{Replace $r(x)$ with $r_t(x)$. Also check that $p(\cdot|x,a)$ is ed by $t$ in the right places.} \NimaResponse{Checked and applied.}
\begin{align*}
    W_t^\beta(x_1,s_1,\phi)-W_t^\beta(x_1,s_2,\phi) &= \mathbb{E}_{p_{1,1}}[W_{t+1}^\beta(x',s_1+\beta^t r_t(x_1),\phi)-W_{t+1}^\beta(x',s_2+\beta^t r_t(x_1),\phi)]\\
    &\geq \mathbb{E}_{p_{1,1}}[W_{t+1}^\beta(x',s_1+\beta^t r_t(x_2),\phi)-W_{t+1}^\beta(x',s_2+\beta^t r_t(x_2),\phi)]\\
    &\geq \mathbb{E}_{p_{2,1}}[W_{t+1}^\beta(x',s_1+\beta^t r_t(x_2),\phi)-W_{t+1}^\beta(x',s_2+\beta^t r_t(x_2),\phi)]\\
    &=W_t^\beta(x_2,s_1,\phi)-W_t^\beta(x_2,s_2,\phi)
\end{align*}
where we first exploit property \ref{prop:w2}, i.e. the convexity of $W_{t+1}^\beta$ and property that $r(x_2)\leq r(x_1)$, which implies, due to Lemma \ref{thm:increasingSecant}, that 
\begin{multline*}
    W_{t+1}^\beta(x',(s_1-s_2)+s_2+\beta^t r_t(x_1),\phi)-W_{t+1}^\beta(x',s_2+\beta^t r_t(x_1),\phi) \geq \\
    W_{t+1}^\beta(x',(s_1-s_2)+s_2+\beta^t r_t(x_2),\phi)-W_{t+1}^\beta(x',s_2+\beta^t r_t(x_2),\phi).
\end{multline*}

We then employ Lemma \ref{lem:sa1} using Assumption \ref{ass:p2} and the super-additivity of $W_{t+1}^\beta$ with respect to $(x,s)$ (i.e. property \ref{prop:w4}).

For $(1,1,0,0)$, we can see that:
\begin{align*}
    W_t^\beta(x_1,s_1,\phi)-W_t^\beta(x_1,s_2,\phi) &= \mathbb{E}_{p_{1,1}}[W_{t+1}^\beta(x',s_1+\beta^t r_t(x_1),\phi)-W_{t+1}^\beta(x',s_2+\beta^t r_t(x_1),\phi)]\\
    &\geq  \mathbb{E}_{p_{2,1}}[W_{t+1}^\beta(x',s_1+\beta^t r_t(x_2),\phi)-W_{t+1}^\beta(x',s_2+\beta^t r_t(x_2),\phi)]\\
    &\geq  \mathbb{E}_{p_{2,2}}[W_{t+1}^\beta(x',s_1+\beta^t r_t(x_2),\phi)-W_{t+1}^\beta(x',s_2+\beta^t r_t(x_2),\phi)]\\
    &=W_t^\beta(x_2,s_1,\phi)-W_t^\beta(x_2,s_2,\phi)
\end{align*}
where we first employed a step derived for $(1,1,1,1)$, and then Lemma \ref{lem:sa1} using Assumption \ref{ass:p1} and the super-additivity of $W_{t+1}^\beta$ with respect to $(x,s)$ (i.e. property \ref{prop:w4}).

The case of $(0,0,0,0)$ is similar as we have
\begin{align*}
    W_t^\beta(x_1,s_1,\phi)-W_t^\beta(x_1,s_2,\phi) &= \mathbb{E}_{p_{1,2}}[W_{t+1}^\beta(x',s_1+\beta^t r_t(x_1),\phi)-W_{t+1}^\beta(x',s_2+\beta^t r_t(x_1),\phi)]\\
    &\geq \mathbb{E}_{p_{1,2}}[W_{t+1}^\beta(x',s_1+\beta^t r_t(x_2),\phi)-W_{t+1}^\beta(x',s_2+\beta^t r_t(x_2),\phi)]\\
    &\geq \mathbb{E}_{p_{2,2}}[W_{t+1}^\beta(x',s_1+\beta^t r_t(x_2),\phi)-W_{t+1}^\beta(x',s_2+\beta^t r_t(x_2),\phi)]\\
    &=W_t^\beta(x_2,s_1,\phi)-W_t^\beta(x_2,s_2,\phi)
\end{align*}
using again first the convexity of $W_{t+1}^\beta$ in $s$ (see Lemma \ref{thm:increasingSecant}), followed with Lemma \ref{lem:sa1} using Assumption \ref{ass:p2} and the super-additivity of $W_{t+1}^\beta$ with respect to $(x,s)$.

The case $(1,0,0,0)$, follows straightforwardly from:
\begin{align*}
    W_t^\beta(x_1,s_1,\phi) &-W_t^\beta(x_1,s_2,\phi) \\
    &= \frac{\phi\beta^t}{T}+\mathbb{E}_{p_{1,1}}[W_{t+1}^\beta(x',s_1+\beta^t r_t(x_1),\phi)]-\mathbb{E}_{p_{1,2}}[W_{t+1}^\beta(x',s_2+\beta^t r_t(x_1),\phi)]\\
    &\geq \mathbb{E}_{p_{1,2}}[W_{t+1}^\beta(x',s_1+\beta^t r_t(x_1),\phi)-W_{t+1}^\beta(x',s_2+\beta^t r_t(x_1),\phi)]\\
    &\geq \mathbb{E}_{p_{2,2}}[W_{t+1}^\beta(x',s_1+\beta^t r_t(x_2),\phi)-W_{t+1}^\beta(x',s_2+\beta^t r_t(x_2),\phi)]\\
    &=W_t^\beta(x_2,s_1,\phi)-W_t^\beta(x_2,s_2,\phi)
\end{align*}
where the first inequality comes from optimality of $a_{11}=1$ over $a_{11}=0$, while the second comes from the derivations for $(0,0,0,0)$.

Finally, $(1,0,1,0)$ is obtained through
\begin{align*}
    W_t^\beta(x_1,s_1,\phi)&-W_t^\beta(x_1,s_2,\phi) \\
    & = \frac{\phi\beta^t}{T}+\mathbb{E}_{p_{1,1}}[W_{t+1}^\beta(x',s_1+\beta^t r_t(x_1),\phi)]-\mathbb{E}_{p_{1,2}}[W_{t+1}^\beta(x',s_2+\beta^t r_t(x_1),\phi)]\\
    &\geq \frac{\phi\beta^t}{T}+\mathbb{E}_{p_{2,1}}[W_{t+1}^\beta(x',s_1+\beta^t r_t(x_1),\phi)]-\mathbb{E}_{p_{2,1}}[W_{t+1}^\beta(x',s_2+\beta^t r_t(x_1),\phi)]\\&\qquad+\mathbb{E}_{p_{1,1}}[W_{t+1}^\beta(x',s_2+\beta^t r_t(x_1),\phi)]-\mathbb{E}_{p_{1,2}}[W_{t+1}^\beta(x',s_2+\beta^t r_t(x_1),\phi)]\\
    &\geq \frac{\phi\beta^t}{T}+\mathbb{E}_{p_{2,1}}[W_{t+1}^\beta(x',s_1+\beta^t r_t(x_2),\phi)]-\mathbb{E}_{p_{2,1}}[W_{t+1}^\beta(x',s_2+\beta^t r_t(x_2),\phi)]\\&\qquad+\mathbb{E}_{p_{1,1}}[W_{t+1}^\beta(x',s_2+\beta^t r_t(x_1),\phi)]-\mathbb{E}_{p_{1,2}}[W_{t+1}^\beta(x',s_2+\beta^t r_t(x_1),\phi)]\\    
    &\geq \frac{\phi\beta^t}{T}+\mathbb{E}_{p_{2,1}}[W_{t+1}^\beta(x',s_1+\beta^t r_t(x_2),\phi)]-\mathbb{E}_{p_{2,1}}[W_{t+1}^\beta(x',s_2+\beta^t r_t(x_2),\phi)]\\&\qquad+\mathbb{E}_{p_{1,1}}[W_{t+1}^\beta(x',s_2+\beta^t r_t(x_2),\phi)]-\mathbb{E}_{p_{1,2}}[W_{t+1}^\beta(x',s_2+\beta^t r_t(x_2),\phi)]\\&\qquad+\mathbb{E}_{p_{1,2}}[W_{t+1}^\beta(x',s_2+\beta^t r_t(x_1),\phi)]-\mathbb{E}_{p_{1,2}}[W_{t+1}^\beta(x',s_2+\beta^t r_t(x_1),\phi)]\\   
        &= \frac{\phi\beta^t}{T}+\mathbb{E}_{p_{2,1}}[W_{t+1}^\beta(x',s_1+\beta^t r_t(x_2),\phi)]-\mathbb{E}_{p_{2,1}}[W_{t+1}^\beta(x',s_2+\beta^t r_t(x_2),\phi)]\\&\qquad+\mathbb{E}_{p_{1,1}}[W_{t+1}^\beta(x',s_2+\beta^t r_t(x_2),\phi)]-\mathbb{E}_{p_{1,2}}[W_{t+1}^\beta(x',s_2+\beta^t r_t(x_2),\phi)]\\   
    &\geq \frac{\phi\beta^t}{T}+\mathbb{E}_{p_{2,1}}[W_{t+1}^\beta(x',s_1+\beta^t r_t(x_2),\phi)]-\mathbb{E}_{p_{2,1}}[W_{t+1}^\beta(x',s_2+\beta^t r_t(x_2),\phi)]\\&\qquad+\mathbb{E}_{p_{2,1}}[W_{t+1}^\beta(x',s_2+\beta^t r_t(x_2),\phi)]-\mathbb{E}_{p_{2,2}}[W_{t+1}^\beta(x',s_2+\beta^t r_t(x_2),\phi)]\\
    &= \frac{\phi\beta^t}{T}+\mathbb{E}_{p_{2,1}}[W_{t+1}^\beta(x',s_1+\beta^t r_t(x_2),\phi)]-\mathbb{E}_{p_{2,2}}[W_{t+1}^\beta(x',s_2+\beta^t r_t(x_2),\phi)]\\
    &=W_t^\beta(x_2,s_1,\phi)-W_t^\beta(x_2,s_2,\phi)
\end{align*}
where we first employ Lemma \ref{lem:sa1} with super-additivity of $W_{t+1}^\beta$ with respect to $(x,s)$ and Assumption \ref{ass:p2}. We then exploit the convexity of $W_{t+1}^\beta$ in $s$ (see  Lemma \ref{thm:increasingSecant}), followed with Lemma \ref{lem:sa1} using super-additivity of $W_{t+1}^\beta$ in $(x,s)$ and Assumption \ref{ass:p1}. The fourth inequality employs Lemma \ref{lem:sa2} using Assumption \ref{ass:p3} and the non-decreasingness of $W_{t+1}^\beta$ in $x$ (i.e. property \ref{prop:w1}).  

The case $(1,1,1,0)$ follows from:
\begin{align*}
    W_t^\beta(x_1,s_1,\phi)-W_t^\beta(x_1,s_2,\phi) &= Q_t^\beta(x_1,s_1,\phi,1)-Q_t^\beta(x_1,s_2,\phi,1) \\
    &= \mathbb{E}_{p_{1,1}}[W_{t+1}^\beta(x',s_1+\beta^t r_t(x_1),\phi)-W_{t+1}^\beta(x',s_2+\beta^t r_t(x_1),\phi)] \\
    &\geq \mathbb{E}_{p_{1, 1}}[W_{t+1}^\beta(x',s_1+\beta^t r_t(x_2),\phi)-W_{t+1}^\beta(x',s_2+\beta^t r_t(x_2),\phi)] \\
    &\geq \mathbb{E}_{p_{2, 1}}[W_{t+1}^\beta(x',s_1+\beta^t r_t(x_2),\phi)-W_{t+1}^\beta(x',s_2+\beta^t r_t(x_2),\phi)] \\
    &= Q_t^\beta(x_2,s_1,\phi,1) - Q_t^\beta(x_2,s_2,\phi,1) \\
    & \geq Q_t^\beta(x_2,s_1,\phi,1) - Q_t^\beta(x_2,s_2,\phi,0) \\
    & = W_t^\beta(x_2,s_1,\phi)-W_t^\beta(x_2,s_2,\phi)
\end{align*}
where the last inequality holds because action $a_{2,2}=0$ is optimal at $(x_2,s_2)$.

\textbf{Super-additive with respect to $(x,\phi)$:} For $x_1\geq x_2$ and $\phi_1\geq \phi_2$, let $a_{1,1},a_{1,2},a_{2,1}$, and $a_{2,2}$ be an action that achieves optimality in $W_t^\beta(x_1,s,\phi_1)$, $W_t^\beta(x_1,s,\phi_2)$, $W_t^\beta(x_2,s,\phi_1)$, and $W_t^\beta(x_2,s,\phi_2)$ respectively. Based on Lemma \ref{lem:monotone-policy}, there exists a monotone policy at time $t$. This implies that a tuple of optimal actions is necessarily in the following list:
\[(a_{1,1},a_{1,2},a_{2,1},a_{2,2})\in\{(1,1,1,1), (1, 1, 1, 0), (1,1,0,0), (0,0,0,0), (1,0,0,0), (1,0,1,0)\}.\]
We will show that in each case, we have that $W_{t}$ is super-additive with respect to $(x,\phi)$.

For $(1,1,1,1)$, we can see that:
\begin{align*}
    W_t^\beta(x_1,s,\phi_1)-W_t^\beta(x_1,s,\phi_2) &= \mathbb{E}_{p_{1,1}}[W_{t+1}^\beta(x',s+\beta^t r_t(x_1),\phi_1)-W_{t+1}^\beta(x',s+\beta^t r_t(x_1),\phi_2)]\\
    &\geq \mathbb{E}_{p_{1,1}}[W_{t+1}^\beta(x',s+\beta^t r_t(x_2),\phi_1)-W_{t+1}^\beta(x',s+\beta^t r_t(x_2),\phi_2)]\\
    &\geq \mathbb{E}_{p_{2,1}}[W_{t+1}^\beta(x',s+\beta^t r_t(x_2),\phi_1)-W_{t+1}^\beta(x',s+\beta^t r_t(x_2),\phi_2)]\\
    &=W_t^\beta(x_2,s,\phi_1)-W_t^\beta(x_2,s,\phi_2)
\end{align*}
where we first exploit property \ref{prop:w6}, i.e. the super-additivity of $W_{t+1}^\beta$ in $(s,\phi)$.
We then employ Lemma \ref{lem:sa1} using Assumption \ref{ass:p2} and the super-additivity of $W_{t+1}^\beta$ with respect to $(x,\phi)$ (i.e. property \ref{prop:w5}).

For $(1,1,0,0)$, we can see that:
\begin{align*}
    W_t^\beta(x_1,s,\phi_1)-W_t^\beta(x_1,s,\phi_2) &= \mathbb{E}_{p_{1,1}}[W_{t+1}^\beta(x',s+\beta^t r_t(x_1),\phi_1)-W_{t+1}^\beta(x',s+\beta^t r_t(x_1),\phi_2)]\\
    &\geq  \mathbb{E}_{p_{2,1}}[W_{t+1}^\beta(x',s+\beta^t r_t(x_2),\phi_1)-W_{t+1}^\beta(x',s+\beta^t r_t(x_2),\phi_2)]\\
    &\geq  \mathbb{E}_{p_{2,2}}[W_{t+1}^\beta(x',s+\beta^t r_t(x_2),\phi_1)-W_{t+1}^\beta(x',s+\beta^t r_t(x_2),\phi_2)]\\
    &=W_t^\beta(x_2,s,\phi_1)-W_t^\beta(x_2,s,\phi_2)
\end{align*}
where we first employed a step derived for $(1,1,1,1)$, and then Lemma \ref{lem:sa1} using Assumption \ref{ass:p1} and the super-additivity of $W_{t+1}^\beta$ with respect to $(x,\phi)$ (i.e. property \ref{prop:w5}).

The case of $(0,0,0,0)$ is similar as we have
\begin{align*}
    W_t^\beta(x_1,s,\phi_1)-W_t^\beta(x_1,s,\phi_2) &= \mathbb{E}_{p_{1,2}}[W_{t+1}^\beta(x',s+\beta^t r_t(x_1),\phi_1)-W_{t+1}^\beta(x',s+\beta^t r_t(x_1),\phi_2)]\\
    &\geq \mathbb{E}_{p_{1,2}}[W_{t+1}^\beta(x',s+\beta^t r_t(x_2),\phi_1)-W_{t+1}^\beta(x',s+\beta^t r_t(x_2),\phi_2)]\\
    &\geq \mathbb{E}_{p_{2,2}}[W_{t+1}^\beta(x',s+\beta^t r_t(x_2),\phi_1)-W_{t+1}^\beta(x',s+\beta^t r_t(x_2),\phi_2)]\\
    &=W_t^\beta(x_2,s,\phi_1)-W_t^\beta(x_2,s,\phi_2)
\end{align*}
using again first the super-additivity of $W_{t+1}^\beta$ in $(s,\phi)$, followed with Lemma \ref{lem:sa1} using Assumption \ref{ass:p2} and the super-additivity of $W_{t+1}^\beta$ with respect to $(x,\phi)$.

The case $(1,0,0,0)$, follows straightforwardly from:
\begin{align*}
    W_t^\beta(x_1,s,\phi_1) &-W_t^\beta(x_1,s,\phi_2) \\&= \frac{\phi_1 - \phi_2}{T}+\mathbb{E}_{p_{1,1}}[W_{t+1}^\beta(x',s+\beta^t r_t(x_1),\phi_1)]-\mathbb{E}_{p_{1,2}}[W_{t+1}^\beta(x',s+\beta^t r_t(x_1),\phi_2)]\\
    &\geq \frac{(\phi_1 - \phi_2)\beta^t}{T}+\mathbb{E}_{p_{1,2}}[W_{t+1}^\beta(x',s+\beta^t r_t(x_1),\phi_1)-W_{t+1}^\beta(x',s+\beta^t r_t(x_1),\phi_2)]\\
    &\geq \frac{(\phi_1 - \phi_2)\beta^t}{T}+\mathbb{E}_{p_{2,2}}[W_{t+1}^\beta(x',s+\beta^t r_t(x_2),\phi_1)-W_{t+1}^\beta(x',s+\beta^t r_t(x_2),\phi_2)]\\
    &=W_t^\beta(x_2,s,\phi_1)-W_t^\beta(x_2,s,\phi_2)
\end{align*}
where the first inequality comes from optimality of $a_{11}=1$ over $a_{11}=0$, while the second comes from the derivations for $(0,0,0,0)$.

Finally, $(1,0,1,0)$ is obtained through
\begin{align*}
    W_t^\beta(x_1,s,\phi_1)&-W_t^\beta(x_1,s,\phi_2) \\&= \frac{\phi_1\beta^t}{T}+\mathbb{E}_{p_{1,1}}[W_{t+1}^\beta(x',s+\beta^t r_t(x_1),\phi_1)]-\mathbb{E}_{p_{1,2}}[W_{t+1}^\beta(x',s+\beta^t r_t(x_1),\phi_2)]\\
    &\geq \frac{(\phi_1 - \phi_2)\beta^t}{T}+\mathbb{E}_{p_{2,1}}[W_{t+1}^\beta(x',s+\beta^t r_t(x_1),\phi_1)]-\mathbb{E}_{p_{2,1}}[W_{t+1}^\beta(x',s+\beta^t r_t(x_1),\phi_2)]\\&\qquad+\mathbb{E}_{p_{1,1}}[W_{t+1}^\beta(x',s+\beta^t r_t(x_1),\phi_2)]-\mathbb{E}_{p_{1,2}}[W_{t+1}^\beta(x',s+\beta^t r_t(x_1),\phi_2)]\\
    &\geq \frac{(\phi_1 - \phi_2)\beta^t}{T}+\mathbb{E}_{p_{2,1}}[W_{t+1}^\beta(x',s+\beta^t r_t(x_2),\phi_1)]-\mathbb{E}_{p_{2,1}}[W_{t+1}^\beta(x',s+\beta^t r_t(x_2),\phi_2)]\\&\qquad+\mathbb{E}_{p_{1,1}}[W_{t+1}^\beta(x',s+\beta^t r_t(x_1),\phi_2)]-\mathbb{E}_{p_{1,2}}[W_{t+1}^\beta(x',s+\beta^t r_t(x_1),\phi_2)]\\    
    &\geq \frac{(\phi_1 - \phi_2)\beta^t}{T}+\mathbb{E}_{p_{2,1}}[W_{t+1}^\beta(x',s+\beta^t r_t(x_2),\phi_1)]-\mathbb{E}_{p_{2,1}}[W_{t+1}^\beta(x',s+\beta^t r_t(x_2),\phi_2)]\\&\qquad+\mathbb{E}_{p_{1,1}}[W_{t+1}^\beta(x',s+\beta^t r_t(x_2),\phi_2)]-\mathbb{E}_{p_{1,2}}[W_{t+1}^\beta(x',s+\beta^t r_t(x_2),\phi_2)]\\   
    &\geq \frac{(\phi_1 - \phi_2)\beta^t}{T}+\mathbb{E}_{p_{2,1}}[W_{t+1}^\beta(x',s+\beta^t r_t(x_2),\phi_1)]-\mathbb{E}_{p_{2,1}}[W_{t+1}^\beta(x',s+\beta^t r_t(x_2),\phi_2)]\\&\qquad+\mathbb{E}_{p_{2,1}}[W_{t+1}^\beta(x',s+\beta^t r_t(x_2),\phi_2)]-\mathbb{E}_{p_{2,2}}[W_{t+1}^\beta(x',s+\beta^t r_t(x_2),\phi_2)]\\
    &= \frac{(\phi_1 - \phi_2)\beta^t}{T}+\mathbb{E}_{p_{2,1}}[W_{t+1}^\beta(x',s+\beta^t r_t(x_2),\phi_1)]-\mathbb{E}_{p_{2,2}}[W_{t+1}^\beta(x',s+\beta^t r_t(x_2),\phi_2)]\\
    &=W_t^\beta(x_2,s,\phi_1)-W_t^\beta(x_2,s,\phi_2)
\end{align*}
where we first employ Lemma \ref{lem:sa1} with super-additivity of $W_{t+1}^\beta$ with respect to $(x,\phi)$ and Assumption \ref{ass:p2}. We then exploit the super-additivity  of $W_{t+1}^\beta$ with respect to $(s,\phi)$ (i.e. property \ref{prop:w6}), followed with Lemma \ref{lem:sa1} using super-additivity of $W_{t+1}^\beta$ in $(x,s)$ and Assumption \ref{ass:p1}. The fourth inequality employs Lemma \ref{lem:sa2} using Assumption \ref{ass:p3} and the non-decreasingness of $W_{t+1}^\beta$ in $x$ (i.e. property \ref{prop:w1}).  

The case $(1,1,1,0)$ follows from:
\begin{align*}
    W_t^\beta(x_1,s,\phi_1) &-W_t^\beta(x_1,s,\phi_2)
    = Q_t(x_1,s,\phi_1,1)-Q_t(x_1,s,\phi_2,1) \\
    &= \frac{(\phi_1 - \phi_2)\beta^t}{T} + \mathbb{E}_{p_{1, 1}}[W_{t+1}^\beta(x',s+\beta^t r_t(x_1),\phi_1)-W_{t+1}^\beta(x',s+\beta^t r_t(x_1),\phi_2)] \\
    &\geq \frac{(\phi_1 - \phi_2)\beta^t}{T} + \mathbb{E}_{p_{1,1}}[W_{t+1}^\beta(x',s+\beta^t r_t(x_2),\phi_1)-W_{t+1}^\beta(x',s+\beta^t r_t(x_2),\phi_2)] \\
    &\geq \frac{(\phi_1 - \phi_2)\beta^t}{T} + \mathbb{E}_{p_{2, 1}}[W_{t+1}^\beta(x',s+\beta^t r_t(x_2),\phi_1)-W_{t+1}^\beta(x',s+\beta^t r_t(x_2),\phi_2)] \\
    &= Q_t(x_2,s,\phi_1,1) - Q_t(x_2,s,\phi_2,1) \\
    &\geq Q_t(x_2,s,\phi_1,1) - Q_t(x_2,s,\phi_2,0) \\
    & = W_t^\beta(x_2,s,\phi_1)-W_t^\beta(x_2,s,\phi_2)
\end{align*}
where the last inequality holds because action $a_{2,2}=0$ is optimal at $(x_2,s_2)$.

\textbf{Super-additive with respect to $(s,\phi)$:} For $s_1\geq s_2$ and $\phi_1\geq \phi_2$, let $a_{1,1},a_{1,2},a_{2,1}$, and $a_{2,2}$ be an action that achieves optimality in $W_t^\beta(x,s_1,\phi_1)$, $W_t^\beta(x,s_1,\phi_2)$, $W_t^\beta(x,s_2,\phi_1)$, and $W_t^\beta(x,s_2,\phi_2)$ respectively. Based on Lemma \ref{lem:monotone-policy}, there exists a monotone policy at time $t$. This implies that a tuple of optimal actions is necessarily in the following list:
\[(a_{1,1},a_{1,2},a_{2,1},a_{2,2})\in\{(1,1,1,1), (1, 1, 1, 0), (1,1,0,0), (0,0,0,0), (1,0,0,0), (1,0,1,0)\}.\]
We will show that in each case, we have that $W_t^\beta$ is super-additive with respect to $(s,\phi)$.

For $(1,1,1,1)$, we can see that:
\begin{align*}
    W_t^\beta(x,s_1,\phi_1)-W_t^\beta(x,s_1,\phi_2) &= \mathbb{E}_{p_{1}}[W_{t+1}^\beta(x',s_1+\beta^t r_t(x),\phi_1)-W_{t+1}^\beta(x',s_1+\beta^t r_t(x),\phi_2)]\\
    &\geq \mathbb{E}_{p_{1}}[W_{t+1}^\beta(x',s_2+\beta^t r_t(x),\phi_1)-W_{t+1}^\beta(x',s_2+\beta^t r_t(x),\phi_2)]\\
    &=W_t^\beta(x,s_2,\phi_1)-W_t^\beta(x,s_1,\phi_2)
\end{align*}
where we simply exploited property \ref{prop:w6}, i.e. the super-additivity of $W_{t+1}^\beta$ in $(s,\phi)$.

For $(1,1,0,0)$, we can see that:
\begin{align*}
    W_t^\beta(x,s_1,\phi_1)-W_t^\beta(x,s_1,\phi_2) &= \mathbb{E}_{p_{1}}[W_{t+1}^\beta(x',s_1+\beta^t r_t(x),\phi_1)-W_{t+1}^\beta(x',s_1+\beta^t r_t(x),\phi_2)]\\
    &\geq  \mathbb{E}_{p_{1}}[W_{t+1}^\beta(x',s_2+\beta^t r_t(x),\phi_1)-W_{t+1}^\beta(x',s_2+\beta^t r_t(x),\phi_2)]\\
    &\geq  \mathbb{E}_{p_{2}}[W_{t+1}^\beta(x',s_2+\beta^t r_t(x),\phi_1)-W_{t+1}^\beta(x',s_2+\beta^t r_t(x),\phi_2)]\\
    &=W_t^\beta(x,s_2,\phi_1)-W_t^\beta(x,s_2,\phi_2)
\end{align*}
where we first employed a step derived for $(1,1,1,1)$, and then Lemma \ref{lem:sa1} using Assumption \ref{ass:p1} and the super-additivity of $W_{t+1}^\beta$ with respect to $(x,\phi)$ (i.e. property \ref{prop:w4}).

The case of $(0,0,0,0)$ is similar as we have
\begin{align*}
    W_t^\beta(x,s_1,\phi_1)-W_t^\beta(x,s_1,\phi_2) &= \mathbb{E}_{p_{2}}[W_{t+1}^\beta(x',s_1+\beta^t r_t(x),\phi_1)-W_{t+1}^\beta(x',s_1+\beta^t r_t(x),\phi_2)]\\
    &\geq \mathbb{E}_{p_{2}}[W_{t+1}^\beta(x',s_2+\beta^t r_t(x),\phi_1)-W_{t+1}^\beta(x',s_2+\beta^t r_t(x),\phi_2)]\\
    &=W_t^\beta(x,s_2,\phi_1)-W_t^\beta(x,s_2,\phi_2)
\end{align*}
using again first the super-additivity of $W_{t+1}^\beta$ in $(s,\phi)$.

The case $(1,0,0,0)$, follows straightforwardly from:
\begin{align*}
    W_t^\beta(x,s_1,\phi_1) &-W_t^\beta(x,s_1,\phi_2) \\&= \frac{(\phi_1 - \phi_2)\beta^t}{T}+\mathbb{E}_{p_{1}}[W_{t+1}^\beta(x',s_1+\beta^t r_t(x),\phi_1)]-\mathbb{E}_{p_{2}}[W_{t+1}^\beta(x',s_1+\beta^t r_t(x),\phi_2)]\\
    &\geq \frac{(\phi_1 - \phi_2)\beta^t}{T} + \mathbb{E}_{p_{2}}[W_{t+1}^\beta(x',s_1+\beta^t r_t(x),\phi_1)-W_{t+1}^\beta(x',s_1+\beta^t r_t(x),\phi_2)]\\
    &\geq \frac{(\phi_1 - \phi_2)\beta^t}{T} + \mathbb{E}_{p_{2}}[W_{t+1}^\beta(x',s_2+\beta^t r_t(x),\phi_1)-W_{t+1}^\beta(x',s_2+\beta^t r_t(x),\phi_2)]\\
    &=W_t^\beta(x,s_2,\phi_1)-W_t^\beta(x,s_2,\phi_2)
\end{align*}
where the first inequality comes from optimality of $a_{11}=1$ over $a_{11}=0$, while the second comes from the derivations for $(0,0,0,0)$.

Finally, $(1,0,1,0)$ is obtained through
\begin{align*}
    W_t^\beta(x,s_1,\phi_1)&-W_t^\beta(x,s_1,\phi_2) \\&= \frac{(\phi_1 - \phi_2)\beta^t}{T}+\mathbb{E}_{p_{1}}[W_{t+1}^\beta(x',s_1+\beta^t r_t(x),\phi_1)]-\mathbb{E}_{p_{2}}[W_{t+1}^\beta(x',s_1+\beta^t r_t(x),\phi_2)]\\
    &\geq \frac{(\phi_1 - \phi_2)\beta^t}{T}+\mathbb{E}_{p_{1}}[W_{t+1}^\beta(x',s_2+\beta^t r_t(x),\phi_1)]-\mathbb{E}_{p_{1}}[W_{t+1}^\beta(x',s_2+\beta^t r_t(x),\phi_2)]\\&\qquad+\mathbb{E}_{p_{1}}[W_{t+1}^\beta(x',s_1+\beta^t r_t(x),\phi_2)]-\mathbb{E}_{p_{2}}[W_{t+1}^\beta(x',s_1+\beta^t r_t(x),\phi_2)]\\ 
    &\geq \frac{(\phi_1 - \phi_2)\beta^t}{T}+\mathbb{E}_{p_{1}}[W_{t+1}^\beta(x',s_2+\beta^t r_t(x),\phi_1)]-\mathbb{E}_{p_{2}}[W_{t+1}^\beta(x',s_2+\beta^t r_t(x),\phi_2)]\\&\qquad+\mathbb{E}_{p_{2}}[W_{t+1}^\beta(x',s_1+\beta^t r_t(x),\phi_2)]-\mathbb{E}_{p_{2}}[W_{t+1}^\beta(x',s_1+\beta^t r_t(x),\phi_2)]\\
    &= \frac{(\phi_1 - \phi_2)\beta^t}{T}+\mathbb{E}_{p_{1}}[W_{t+1}^\beta(x',s_2+\beta^t r_t(x),\phi_1)]-\mathbb{E}_{p_{2}}[W_{t+1}^\beta(x',s_2+\beta^t r_t(x),\phi_2)]\\
    &=W_t^\beta(x,s_2,\phi_1)-W_t^\beta(x,s_2,\phi_2)
\end{align*}
where we first employ the super-additivity of $W_{t+1}^\beta$ with respect to $(s,\phi)$. We then exploit Lemma \ref{lem:sa1} using super-additivity of $W_{t+1}^\beta$ in $(x,s)$ and Assumption \ref{ass:p1}. 

The case $(1,1,1,0)$ for $(s,\phi)$ follows from:
\begin{align*}
    W_t^\beta(x,s_1,\phi_1) &-W_t^\beta(x,s_1,\phi_2) = Q_t^\beta(x,s_1,\phi_1,1)-Q_t^\beta(x,s_1,\phi_2,1) \\
    &= \frac{(\phi_1 - \phi_2)\beta^t}{T} + \mathbb{E}_{P_t(\cdot|x,1)}[W_{t+1}^\beta(x',s_1+\beta^t r_t(x),\phi_1)-W_{t+1}^\beta(x',s_1+\beta^t r_t(x),\phi_2)] \\
    &\geq \frac{(\phi_1 - \phi_2)\beta^t}{T} + \mathbb{E}_{P_t(\cdot|x,1)}[W_{t+1}^\beta(x',s_2+\beta^t r_t(x),\phi_1)-W_{t+1}^\beta(x',s_2+\beta^t r_t(x),\phi_2)] \\
    &= Q_t^\beta(x,s_2,\phi_1,1) - Q_t^\beta(x,s_2,\phi_2,1) \\
    &\geq Q_t^\beta(x,s_2,\phi_1,1) - Q_t^\beta(x,s_2,\phi_2,0) \\
    & = W_t^\beta(x,s_2,\phi_1)-W_t^\beta(x,s_2,\phi_2)
\end{align*}
where the first inequality follows from the super-additivity of $W_{t+1}^\beta$ in $(s,\phi)$ (property \ref{prop:w6}). The second inequality holds because action $a_{2,2}=0$ is optimal at $(x,s_2,\phi_2)$. \Halmos

\section{Proof of Theorem \ref{thm:ns-indexable2}} \label{app:ns-indexable2}

Recall that when condition \ref{ass:r2} is satisfied, the Bellman equation associated to the augmented arm risk neutral MDP: 
\[
V_{\lambda,t}(x, s) = \max_{a \in \{0,1\}} \left\{ \mathbb{I}(t=T-1)U(s+r_t(x))-\frac{\lambda}{T} a + \sum_{x'} P_t(x'|x, a) V_{\lambda,t+1}(x', s + r_t(x)) \right\}.
\]
for any given state $(x,s)$ and given $\lambda\in\mathbb{R}_+$ and $T\in\mathcal{T}$, and with $V_{\lambda,T}(x,s) = 0$. 

\begin{lemma}[Difference of Value Bound]
For any time $t \in\mathcal{T}$, any state $(x,s)$ and any $\lambda_2\geq \lambda_1\geq 0$. the difference in value is bounded by: 
\[
-\frac{T-1-t}{T}(\lambda_2-\lambda_1) \leq V_{\lambda_2,t}(x, s) - V_{\lambda_1,t}(x, s) \leq 0.
\]
\end{lemma}
\textbf{\textit{Proof.}}
The proof is based on backward induction.

\textbf{Base Case:} At the final decision step, $t = T-1$, the Bellman equation is
\[
V_{\lambda,T-1}(x, s, \lambda) = \max_{a \in \{0,1\}} \left\{ U(s+r_{T-1}(x))-\frac{\lambda}{T} a \right\} = U(s+r_{T-1}(x)).
\]
Hence, we have that $V_{\lambda_2,T-1}(x, s) - V_{\lambda_1,T-1}(x, s)=0$, which is clearly inside $[0,0]$. Thus, the bounds hold.

\textbf{Inductive Step:} Assume the bounds hold for time $t+1$, one can first verify at $t$ that:
\begin{align*}
   &V_{\lambda_2,t}(x, s) - V_{\lambda_1,t}(x, s) =  \max_{a \in \{0,1\}} \left\{ -\frac{\lambda_2}{T} a + \sum_{x'} P_t(x'|x, a) V_{\lambda_2,t+1}(x', s + r_t(x)) \right\} \\&\quad - \max_{a \in \{0,1\}} \left\{ -\frac{\lambda_1}{T} a + \sum_{x'} P_t(x'|x, a) V_{\lambda_1,t+1}(x', s + r_t(x)) \right\}\\
   &\quad\leq \max_{a \in \{0,1\}} \left\{ -\frac{\lambda_2-\lambda_1}{T} a + \sum_{x'} P_t(x'|x, a) (V_{\lambda_2,t+1}(x', s + r_t(x)) - V_{\lambda_1,t+1}(x', s + r_t(x))\right\}\\
   &\quad\leq \max_{a \in \{0,1\}} \left\{ -\frac{\lambda_2-\lambda_1}{T} a\right\} = 0.
\end{align*}
where we used $V_{\lambda_2,t+1}(x, s) - V_{\lambda_1,t+1}(x, s) \leq 0$ for all $(x,s)$. Next, one can see that 
\begin{align*}
   V_{\lambda_2,t}(x, s) &- V_{\lambda_1,t}(x, s)
   \\ & \geq \min_{a \in \{0,1\}} \left\{ -\frac{\lambda_2-\lambda_1}{T} a + \sum_{x'} P_t(x'|x, a) (V_{\lambda_2,t+1}(x', s + r_t(x)) - V_{\lambda_1,t+1}(x', s + r_t(x))\right\}\\
   & \geq \min_{a \in \{0,1\}} \left\{ -\frac{\lambda_2-\lambda_1}{T} a\right\}  -\frac{T-t-2}{T}(\lambda_2-\lambda_1)= -\frac{T-1-t}{T}(\lambda_2-\lambda_1),
\end{align*}
where we used $V_{\lambda_2,t+1}(x, s) - V_{\lambda_1,t+1}(x, s) \geq -(T-1-(t+1))(\lambda_2-\lambda_1)/T$.\Halmos

\begin{lemma}\label{thm:ass2increasingPol}
If a restless bandit arm satisfies condition \ref{ass:r2} and Assumption~\ref{ass:mdp2}, then there exists a family of optimal policies $\{f_\lambda^{*}\}_{\lambda\geq 0}$, for its augmented arm risk neutral MDP, that is non-increasing with respect to $\lambda$.
\end{lemma}

\textbf{\textit{Proof.}}
We start by considering an advantage function 
\[
\Delta_t(x, s,\lambda) := \frac{\lambda}{T} + \sum_{x'} (P_t(x'|x, 0) - P_t(x'|x, 1)) V_{\lambda,t+1}(x', s + r_t(x)).
\]
and showing that it is non-decreasing in $\lambda$. Namely, for $t=T-1$ we have
\[\Delta_{T-1}(x, s,\lambda) = \frac{\lambda}{T} + \sum_{x'} (p_{T-1}(x'|x, 0) - p_{T-1}(x'|x, 1)) V_{\lambda,T}(x', s + r_{T-1}(x)) = \frac{\lambda}{T},\]
since $V_{\lambda,T}(x,s)=0$, thus non-decreasing in $\lambda$. 
\NimaEdits{We next show that $\Delta_{T-2}(x, s, \lambda)$ is non-decreasing in $\lambda$. From previous lemma, we know $V_{\lambda_2,T-1}(x,s) - V_{\lambda_1,T-1}(x,s) = 0$. Therefore, for any $\lambda_2 \geq \lambda_1 \geq 0$, we have:
\begin{align*}
    \Delta_{T-2}(x, s, \lambda_2) - \Delta_{T-2}(x, s, \lambda_1)
    &= \frac{\lambda_2-\lambda_1}{T} + \sum_{x'} [P_{T-2}(x'|x, 0) - P_{T-2}(x'|x, 1)] \\
    & \qquad \qquad \qquad (V_{\lambda_2,T-1}(x', s + r_{T-2}(x)) - V_{\lambda_1,T-1}(x', s + r_{T-2}(x))) \\
    &= \frac{\lambda_2-\lambda_1}{T} + \sum_{x'} [P_{T-2}(x'|x, 0) - P_{T-2}(x'|x, 1)] \cdot (0) \\
    &= \frac{\lambda_2-\lambda_1}{T}.
\end{align*}
Since $\lambda_2 \geq \lambda_1$, this difference is non-negative. Thus, $\Delta_{T-2}(x, s, \lambda)$ is non-decreasing in $\lambda$.}
Whereas when \NimaEdits{$t\leq T-3$}, one can show that if $\lambda_2 \geq \lambda_1\geq 0$, then
\begin{align*}
    \Delta_t(x, s,\lambda_2) &- \Delta_t(x, s,\lambda_1)= \frac{\lambda_2}{T} + \sum_{x'} (P_t(x'|x, 0) - P_t(x'|x, 1)) V_{\lambda_2,t+1}(x', s + r_t(x)) \\
    &- (\frac{\lambda_1}{T} + \sum_{x'} (P_t(x'|x, 0) - P_t(x'|x, 1)) V_{\lambda_1,t+1}(x', s + r_t(x)))\\
    &= \frac{\lambda_2-\lambda_1}{T} + \sum_{x'} [P_t(x'|x, 0) - P_t(x'|x, 1)] (V_{\lambda_2,t+1}(x', s + r_t(x))-V_{\lambda_1,t+1}(x', s + r_t(x)))\\
    &\geq \frac{\lambda_2-\lambda_1}{T} - \sum_{x'} |P_t(x'|x, 0) - P_t(x'|x, 1)| |V_{\lambda_2,t+1}(x', s + r_t(x))-V_{\lambda_1,t+1}(x', s + r_t(x))|\\
    &\geq \frac{\lambda_2-\lambda_1}{T} - \sum_{x'} |P_t(x'|x, 0) - P_t(x'|x, 1)| \frac{T-2-t}{T}(\lambda_2-\lambda_1)\\
    &\geq \frac{\lambda_2-\lambda_1}{T} - \frac{1}{T-t-2} \frac{T-2-t}{T}(\lambda_2-\lambda_1) \geq 0.
\end{align*}

Consider the policy
$f_{\lambda,t}^{*}(x,s):=\min(\arg\max_{a \in \{0,1\}} \left\{ -\frac{\lambda}{T} a + \sum_{x'} P_t(x'|x, a) V_{\lambda,t+1}(x', s + r_t(x)) \right\})$. We will show that $f_{\lambda,t}^{*}(x,s)$ is non-increasing in $\lambda$. Specifically, for $\lambda_2\geq \lambda_1\geq 0$, if $f_{\lambda_1,t}^{*}(x,s)=0$, then 
\[-\frac{\lambda_1}{T}\cdot 1 + \sum_{x'} P_t(x'|x, 1) V_{\lambda_1,t+1}(x', s + r_t(x)) \leq -\frac{\lambda_1}{T}\cdot 0 + \sum_{x'} P_t(x'|x, 0) V_{\lambda_1,t+1}(x', s + r_t(x))\]
by definition of the policy, thus implying that $\Delta_t(x, s,\lambda_1)\geq 0$. Moreover, we have $0\leq \Delta_t(x, s,\lambda_1)\leq \Delta_t(x, s,\lambda_2)$, which implies that $f_{\lambda_2,t}^{*}(x,s)=0$.
\Halmos

The rest of the proof of Theorem \ref{thm:ns-indexable2} follows directly as the proof of Theorem \ref{thm:ns-indexable} with the difference that it is Lemma \ref{thm:ass2increasingPol} that ensures the existence of  a family of optimal policies $\{f_\lambda^{i*}\}_{
\lambda\geq 0}$ that is non-increasing in $\lambda$. \Halmos

\section{Simple Illustrative Models which Satisfy the Assumptions of Theorem~\ref{thm:ns-indexable}} \label{app:models}

A class of models which satisfy the Assumptions of Theorem~\ref{thm:ns-indexable} is presented below:
The state space for each arm is sorted from the worst to the best state. Per-step reward is only a function of the state for each arm and is non-decreasing over the state space. The transition probabilities follow any of the following pairs:
\begin{itemize}
\item Given an arm with $n$ states and a parameter $p \in [0, 1]$, let the transition probability matrix under passive and active actions be:
\begin{equation*}
{\cal P}_{0}(p) = 
\begin{bmatrix}
1 & 0 & 0 & 0   & 0   & 0   & 0 & \dots  & 0 \\
1-p & p & 0 & 0 & 0   & 0   & 0 & \dots  & 0 \\
\vdots & \vdots & \vdots & \vdots & \vdots & \vdots & \vdots & \vdots & \vdots \\
1-p & \dots & 0 & 0 & 0 & 0 & 0 & p & 0 \\
1-p & \dots & 0 & 0 & 0 & 0 & 0 & 0 & p
\end{bmatrix},
\quad 
{\cal P}_{1}(p) = 
\begin{bmatrix}
1 & 0 & 0 & 0   & 0   & 0   & 0 & \dots  & 0 \\
0 & 1 & 0 & 0 & 0   & 0   & 0 & \dots  & 0 \\
\vdots & \vdots & \vdots & \vdots & \vdots & \vdots & \vdots & \vdots & \vdots \\
0      & \ldots & 0 & 0 & 0 & 0 & 0 & 1  	& 0 \\
0      & \ldots & 0 & 0 & 0 & 0 & 0 & 0  & 1
\end{bmatrix}.
\end{equation*}

\item Given an arm with $n$ states and parameter $p1 > p2$, let the transition probability matrix under passive and active actions be:
\begin{equation*}
{\cal P}_{0}(p) = 
\begin{bmatrix}
1 & 0 & 0 & 0   & 0   & 0   & 0 & \dots  & 0 \\
1-p2 & p2 & 0 & 0 & 0   & 0   & 0 & \dots  & 0 \\
\vdots & \vdots & \vdots & \vdots & \vdots & \vdots & \vdots & \vdots & \vdots \\
1-p2 & \dots & 0 & 0 & 0 & 0 & 0 & p2 & 0 \\
1-p2 & \dots & 0 & 0 & 0 & 0 & 0 & 0 & p2
\end{bmatrix}, \quad 
{\cal P}_{1}(p) = 
\begin{bmatrix}
1 & 0 & 0 & 0   & 0   & 0   & 0 & \dots  & 0 \\
1-p1 & p1 & 0 & 0 & 0   & 0   & 0 & \dots  & 0 \\
\vdots & \vdots & \vdots & \vdots & \vdots & \vdots & \vdots & \vdots & \vdots \\
1-p1 & \dots & 0 & 0 & 0 & 0 & 0 & p1 & 0 \\
1-p1 & \dots & 0 & 0 & 0 & 0 & 0 & 0 & p1
\end{bmatrix}.
\end{equation*}

\item Given an arm with $n$ states and a parameter $p \in [0, 0.5]$, let the transition probability matrix under passive and active actions be:
\begin{equation*}
{\cal P}_{0}(p) = 
\begin{bmatrix}
1 & 0 & 0 & 0   & 0   & 0   & 0 & \dots  & 0 \\
1-p & p & 0 & 0 & 0   & 0   & 0 & \dots  & 0 \\
\vdots & \vdots & \vdots & \vdots & \vdots & \vdots & \vdots & \vdots & \vdots \\
0 & \dots & 0 & 0 & 0 & 0 & 1-p & p & 0 \\
0 & \dots & 0 & 0 & 0 & 0 & 0 & 1-p & p
\end{bmatrix}, \quad
{\cal P}_{1}(p) = 
\begin{bmatrix}
1 & 0 & 0 & 0   & 0   & 0   & 0 & \dots  & 0 \\
p & 1-p & 0 & 0 & 0   & 0   & 0 & \dots  & 0 \\
\vdots & \vdots & \vdots & \vdots & \vdots & \vdots & \vdots & \vdots & \vdots \\
0      & \ldots & 0 & 0 & 0 & 0 & p & 1-p  	& 0 \\
0      & \ldots & 0 & 0 & 0 & 0 & 0 & p & 1-p
\end{bmatrix}.
\end{equation*}

\item Given an arm with $n$ states and a parameter $p \in [0, 1/(n-1)]$, let the transition probability matrix under passive action be:
\begin{equation*}
{\cal P}_{0}(p) = 
\begin{bmatrix}
1 & 0 & 0 & 0   & \dots  & 0 \\
1 - (n-1) p & (n-1) p & 0  & 0 & \dots  & 0 \\
1 - (n-1) p & p & (n-2) p  & 0 & \dots  & 0 \\
\vdots & \vdots & \vdots & \vdots & \vdots & \vdots \\
1 - (n-1) p & \dots & p & p & 2 p & 0 \\
1 - (n-1) p & \dots & p & p & p & p
\end{bmatrix},
{\cal P}_{1}(p) = 
\begin{bmatrix}
(n-1) p & 0 & 0    & \dots  & 1 - (n-1)p \\
0 & (n-2) p & 0   & \dots  & 1- (n-2) p \\
\vdots & \vdots & \vdots & \vdots & \vdots \\
0      & \ldots & 0 & \dots  	& 1 - p \\
0      & \ldots & 0 & \dots  & 1
\end{bmatrix}
\end{equation*}
\end{itemize}

\section{Proof of Lemma \ref{lemma:rl-Pdiff}} \label{app:lem-pdiff}
We follow the steps provided in \cite{akbarzadeh2023learning} to prove the lemma.

First, we state some basic properties of Thompson sampling algorithm.
\begin{lemma}[Lemma 1 of \cite{russo2014learning}]\label{lem:TS1}
Suppose the true parameters $\theta^\star$ and the estimated ones $\theta_k$ have the same distribution given the same history~$\mathcal{H}$. For any $\mathcal{H}$-measurable function~$f$, we have 
\( \EXP[f(\theta^\star)|\mathcal{H}] = \EXP[f(\theta_k)|\mathcal{H}]. \)
\end{lemma}

Let 
\[
\hat P^i_{k, t}(x'^{i}_{t+1} | x^i_t, a^i_t) := 
\begin{cases}
\bar{p}^i_0(x'^{i}_{{t+1}} | x^i_t, a^i_t), & \text{ if } N^i_{k, t}(x^i_t, a^i_t) = 0,
\\
N^i_{k, t}(x^i_t, a^i_t, x'^i_{t+1})/ N^i_{k, t}(x^i_t, a^i_t) , & \text{ otherwise, }
\end{cases}
\]

denote the sample mean estimation of $P^i_t(x'^i_{t+1} | x^i_t, a^i_t)$ based on observations up to end of episode~$k$. For the ease of notation, for a given $\delta \in (0,1)$, we define
\begin{equation}\label{eq:rl-eps}
\epsilon^i_\delta(\ell) := 
\begin{cases}
\sqrt{\frac{2 |\mathcal{X}|^i \log(1/\delta)}{\ell}} & \text{ if } \ell \geq 1 \\
\sqrt{2 |\mathcal{X}|^i \log(1/\delta)} & \text{ if } \ell = 0
\end{cases}
.
\end{equation}

\begin{lemma}\label{lem:rl-ineq-ns}
Consider any arm~$i$, episode~$k$, $\delta \in (0,1)$, $\ell > 1$, state-action pair $(x^i_t,a^i_t)$ at time~$t$. 
Define events $\mathcal{E}^i_{k, t, \ell}(x^i_t, a^i_t) := \{ N^i_{k, t}(x^i_t, a^i_t) = \ell \}$ and $\mathcal{F}^{\star i}_{k, t}(x^i_t, a^i_t) := \{ \lVert P^{\star i}_t(\cdot   | x^i_t, a^i_t) - \hat P^i_{k, t}(\cdot   | x^i_t, a^i_t) \rVert_1 \leq \epsilon_\delta(N^i_{k, t}(x^i_t, a^i_t)) \}$, and $\mathcal{F}^{i}_{k, t}(x^i_t, a^i_t) := \{ \lVert P^i_{k, t}(\cdot   | x^i_t, a^i_t) - \hat P^i_{k, t}(\cdot   | x^i_t, a^i_t) \rVert_1 \leq \epsilon_\delta(N^i_{k, t}(x^i_t, a^i_t)) \}$. Then, we have
\begin{align}
\PR\Bigl( \bigl\lVert P^{\star i}_t(\cdot   | x^i_t, a^i_t) - \hat P^i_{k, t}(\cdot   | x^i_t, a^i_t) \bigr\rVert_1
&>
\epsilon^i_\delta(\ell) \Bigm| \mathcal{E}^i_{k, t, \ell}(x^i_t, a^i_t)
\Bigr) \le \delta,
\label{eq:rl-ineq-1} \\
\PR\Bigl( \bigl\lVert P^i_{k, t}(\cdot   | x^i_t, a^i_t) - \hat P^i_{k, t}(\cdot   | x^i_t, a^i_t) \bigr\rVert_1
&>
\epsilon^i_\delta(\ell) \Bigm| \mathcal{E}^i_{k, t, \ell}(x^i_t, a^i_t)
\Bigr) \le \delta, \label{eq:rl-ineq-2}
\end{align}
where $\epsilon^i_{\delta}(\ell)$ is given by~\eqref{eq:rl-eps}. The above inequalities imply that
\begin{align}
&\EXP\Bigl[ \bigl\lVert P^{\star i}_t(\cdot   | x^i_t, a^i_t) - \hat P^i_{k, t}(\cdot   | x^i_t, a^i_t) \bigr\rVert_1 \Bigr] \le
\EXP[\epsilon^i_\delta(N^i_{k, t}(x^i_t, a^i_t)) | \mathcal{F}^{\star i}(x^i_t, a^i_t)] + 2\delta,
\label{eq:rl-ineq-3} \\
&\EXP\Bigl[ \bigl\lVert P^i_{k, t}(\cdot   | x^i_t, a^i_t) - \hat P^i_{k, t}(\cdot   | x^i_t, a^i_t) \bigr\rVert_1 \Bigr] \le
\EXP[\epsilon^i_\delta(N^i_{k, t}(x^i_t, a^i_t)) | \mathcal{F}^{\star i}_{k, t}(x^i_t, a^i_t)] + 2\delta.
\label{eq:rl-ineq-4}
\end{align}
\end{lemma}
\begin{proof}{Proof:}
Given arm~$i$, the state-action pair~$(x^i_t, a^i_t)$ of the arm at time~$t$ and $N^{i}_{k, t}(x^i_t, a^i_t) = \ell$ samples we have
\[ 
\mathbb{P}\Big( \lVert P^{\star i}_t(\cdot   | x^i_t, a^i_t) - \hat P^i_{k, t}(\cdot   | x^i_t, a^i_t) \rVert_1 \geq \varepsilon \Bigm| \mathcal{E}^i_{k, t, \ell} \Big) 
\leq 2^{|\mathcal{X}^i|} \exp\biggl( - \dfrac{\ell \varepsilon^2}{2} \biggr) 
< \exp\biggl( |\mathcal{X}^i| - \dfrac{\ell \varepsilon^2}{2} \biggr).
\]
Therefore, setting $\delta = \exp(|\mathcal{X}^i| - \ell \varepsilon^2/2)$, we get
\[
\PR\Biggl( \bigl\lVert P^{\star i}_t(\cdot   | x^i_t, a^i_t) - \hat P^i_{k, t}(\cdot   | x^i_t, a^i_t) \bigr\rVert_1
>
\sqrt{\frac{2 (|\mathcal{X}^i| + \log(1/\delta))}{1 \vee \ell}} 
\Bigm| \mathcal{E}^i_{k, t, \ell} \Biggr) \le \delta.
\]
This proves~\eqref{eq:rl-ineq-1}. Eq.~\eqref{eq:rl-ineq-2} follows from the Thompson sampling Lemma (Lemma~\ref{lem:TS1}). 
To prove \eqref{eq:rl-ineq-3} and \eqref{eq:rl-ineq-4}, we first show
\begin{align} \label{eqn:rl-PFc_bound}
P((\mathcal{F}^{\star i}_{k, t}(x^i_t, a^i_t))^c) & = \PR\Bigl( \bigl\lVert P^{\star i}_t(\cdot   | x^i_t, a^i_t) - \hat P^i_{k, t}(\cdot   | x^i_t, a^i_t) \bigr\rVert_1 > \epsilon^i_\delta(N^i_{k, t}(x^i_t, a^i_t)) \Bigr) \notag \\
& = \sum_{\ell = 0}^{\infty} \PR\Bigl( \bigl\lVert P^{\star i}_t(\cdot   | x^i_t, a^i_t) - \hat P^i_{k, t}(\cdot   | x^i_t, a^i_t) \bigr\rVert_1 > \epsilon^i_\delta(\ell) \Bigm| \mathcal{E}^i_{k, t, \ell} \Bigr) \PR(\mathcal{E}^i_{k, t, \ell}) \notag \\
& \leq \sum_{\ell = 0}^{\infty} \delta \PR(\mathcal{E}^i_{k, t, \ell}) = \delta.
\end{align}
Now consider
\begin{align*}
\hskip 2em & \hskip -2em
\EXP\Bigl[ \bigl\lVert P^{\star i}_t(\cdot   | x^i_t, a^i_t) - \hat P^i_{k, t}(\cdot   | x^i_t, a^i_t) \rVert_1 \Bigr]
=
\EXP\Bigl[ \bigl\lVert P^{\star i}_t(\cdot   | x^i_t, a^i_t) - \hat P^i_{k, t}(\cdot   | x^i_t, a^i_t) \bigr\rVert_1 \Bigm| \mathcal{F}^{\star i}_{k, t}(x^i_t, a^i_t) \Bigr] \PR(\mathcal{F}^{\star i}_{k, t}(x^i_t, a^i_t)) \notag \\
& + \EXP\Bigl[ \bigl\lVert P^{\star i}_t(\cdot   | x^i_t, a^i_t) - \hat P^i_{k, t}(\cdot   | x^i_t, a^i_t) \bigr\rVert_1 \Bigm| (\mathcal{F}^{\star i}_{k, t}(x^i_t, a^i_t))^c \Bigr] \PR((\mathcal{F}^{\star i}_{k, t}(x^i_t, a^i_t))^c) \notag \\
& \stackrel{(a)}\le 2 P((\mathcal{F})^c) + \EXP[\epsilon^i_\delta(N^i_{k, t}(x^i_t, a^i_t)) | \mathcal{F}^{\star i}_{k, t}(x^i_t, a^i_t)] \notag \\
& \stackrel{(b)}\le 2 \delta + \EXP[\epsilon^i_\delta(N^i_{k, t}(x^i_t, a^i_t)) | \mathcal{F}^{\star i}_{k, t}(x^i_t, a^i_t)],
\end{align*}
where $(a)$ uses $\lVert \cdot \rVert_1 \le 2$ and $P(\mathcal{F}) \le 1$ and $(b)$ uses \eqref{eqn:rl-PFc_bound}. This proves \eqref{eq:rl-ineq-3}. Eq.~\eqref{eq:rl-ineq-4} follows from a similar argument. \Halmos
\end{proof}

\begin{lemma} \label{lemma:rl-Pdiff-ns}
Consider episode~$k$, $\delta \in (0,1)$. Let denote the joint state-action pair of the learning policy for the system at time step~$t$ is $(\bs{X}_{k, t}, \bs{A}_{k, t})$. Define events $\mathcal{F}^{\star i}_{k, t}$ and $\mathcal{F}^{i}_{k, t}$ as in Lemma~\ref{lem:rl-ineq-ns}. Then we have 
\begin{multline*}
\EXP\Bigl[ \bigl\lVert
\bs{P}^{\star}_t( \cdot | \bs{X}_{k, t}, \bs{A}_{k, t})
-
\bs{P}_{k, t}(\cdot | \bs{X}_{k, t}, \bs{A}_{k, t}) 
\bigr\rVert_1 \Bigr] 
\\
\le 4 N \delta + \sum_{i \in \mathcal{N}} \left(  \EXP\bigl[ \epsilon^i_\delta(N^i_{k, t}(X^i_{k, t}, A^i_{k, t})) | \mathcal{F}^{\star i}_{k, t}(X^i_{k, t}, A^i_{k, t}) \bigr] + \EXP\bigl[ \epsilon^i_\delta(N^i_{k, t}(X^i_{k, t}, A^i_{k, t})) | \mathcal{F}^{i}_{k, t}(X^i_{k, t}, A^i_{k, t}) \bigr] \right). 
\end{multline*}
\end{lemma}
\begin{proof}{Proof:}
From \cite[Lemma 13]{jung2019thompson}, we have
\begin{equation*}
\EXP\Bigl[ \bigl\lVert
\bs{P}^{\star}_t( \cdot | \bs{x}, \bs{a})
-
\bs{P}_{k, t}(\cdot | \bs{x}, \bs{a}) 
\bigr\rVert_1 \Bigr] 
\leq \sum_{i \in \mathcal{N}} \EXP\Bigl[ \bigl\lVert P^{\star i}_t(\cdot   | x^i_t, a^i_t) - P^i_{k, t}(\cdot   | x^i_t, a^i_t) \bigr\rVert_1 \Bigr].
\end{equation*}
Then, the rest of the proof is as follows:
\begin{align*}
& \EXP\Bigl[ \bigl\lVert
\bs{P}^{\star}_t( \cdot | \bs{X}_{k, t}, \bs{A}_{k, t})
-
\bs{P}_{k, t}(\cdot | \bs{X}_{k, t}, \bs{A}_{k, t}) 
\bigr\rVert_1 \Bigr] \leq
\sum_{i = 1}^{N} \EXP\Bigl[ \bigl\lVert P^{\star i}_t(\cdot   | X^i_{k, t}, A^i_{k, t}) - P^i_{k, t}(\cdot   | X^i_{k, t}, A^i_{k, t}) \bigr\rVert_1 \Bigr] \\
& \quad \leq \sum_{i = 1}^{N} \EXP\Bigl[ \bigl\lVert P^{\star i}_t(\cdot   | X^i_{k, t}, A^i_{k, t}) - \hat P^i_{k, t}(\cdot   | X^i_{k, t}, A^i_{k, t}) \bigr\rVert_1 + \bigl\lVert P^{i}_k(\cdot   | X^i_{k, t}, A^i_{k, t}) - \hat P^i_{k, t}(\cdot   | X^i_{k, t}, A^i_{k, t}) \bigr\rVert_1 \Bigr] \\
& \quad \le 4 N \delta + \sum_{i = 1}^{N} \EXP\bigl[ \epsilon^i_\delta(N^i_{k, t}(X^i_{k, t}, A^i_{k, t})) | \mathcal{F}^{\star i}_{k, t}(X^i_{k, t}, A^i_{k, t}) \bigr] + \EXP\bigl[ \epsilon^i_\delta(N^i_{k, t}(X^i_{k, t}, A^i_{k, t})) | \mathcal{F}^{i}_{k, t} (X^i_{k, t}, A^i_{k, t})\bigr] ,
\end{align*}
where the second inequality follows from triangle inequality, and the third follows from Lemma~\ref{lem:rl-ineq-ns}. \Halmos
\end{proof}

Finally, to prove the result of Theorem \ref{thm:regret} we have the following:
\begin{multline*}
\sum_{k = 1}^{K} \sum_{t = 0}^{T-1} \EXP\Bigl[ \bigl\lVert \bs{P}^{\star}_{t}(\cdot   | \bs{X}_{k, t}, \bs{A}_{k, t}) - \bs{P}_{k, t}(\cdot   | \bs{X}_{k, t}, \bs{A}_{k, t}) \bigr\rVert_1 \Bigr] \\
\le
\sum_{k = 1}^{K} \sum_{t = 0}^{T-1} 4 n \delta + \sum_{i \in \mathcal{N}} \left(  \EXP\bigl[ \epsilon^i_\delta(N^i_{k, t}(X^i_{k, t}, A^i_{k, t})) | \mathcal{F}^{\star i}_{k, t}(X^i_{k, t}, A^i_{k, t}) \bigr] + \EXP\bigl[ \epsilon^i_\delta(N^i_{k, t}(X^i_{k, t}, A^i_{k, t})) | \mathcal{F}^{i}_{k, t}(X^i_{k, t}, A^i_{k, t}) \bigr] \right).
\end{multline*}

For the first term, $\sum_{k = 1}^{K} \sum_{t = 0}^{T-1} 4 N \delta = 4 N \delta K T$.

For the second term, we follow the steps from \cite{osband2013more}. Therefore, by definition of $\epsilon^i_\delta(\cdot)$, we have 
\begin{align}
& \sum_{k = 1}^{K} \sum_{t = 0}^{T-1}  \sum_{i = 1}^{N} \EXP\bigg[ \sqrt{\dfrac{\left(4 |\mathcal{X}|^i + \log (1/\delta)\right)}{1 \vee N^i_{k, t}(X^i_{k, t}, A^i_{k, t})}} \Biggm| \mathcal{F}^{\star i}_{k, t}(X^i_{k, t}, A^i_{k, t}) \bigg] \\
& \leq \sum_{i=1}^{N} \sqrt{4 \left(|\mathcal{X}|^i + \log (1/\delta)\right)} \sum_{k = 1}^{K} \sum_{t = 0}^{T-1} \EXP\bigg[ \sqrt{\dfrac{1}{1 \vee N^i_{k, t}(X^i_{k, t}, A^i_{k, t})}} \Biggm| \mathcal{F}^{\star i}_{k, t}(X^i_{k, t}, A^i_{k, t}) \bigg] \notag \\
& \leq \sum_{i=1}^{N} \sqrt{4 \left(|\mathcal{X}|^i + \log (1/\delta)\right)} \sum_{k = 1}^{K} \sum_{t = 0}^{T-1} \notag \\
& \quad \left( \IND(N^i_{k, t}(X^i_{k, t}, A^i_{k, t}) \leq 1) + \IND(N^i_{k, t}(X^i_{k, t}, A^i_{k, t}) > 1) \EXP\bigg[ \sqrt{\dfrac{1}{1 \vee N^i_{k, t}(X^i_{k, t}, A^i_{k, t})}} \Biggm| \mathcal{F}^{\star i}_{k, t}(X^i_{k, t}, A^i_{k, t}) \bigg]\right). \notag
\end{align}

Let assume $(X^i_{k, t}, A^i_{k, t}) = (x, a)$. Consider $N^i_{k, t}(x, a) \leq 1$. This can happen fewer than $2$ times per state-action-time tuple. Therefore, $\sum_{k = 1}^{K} \sum_{t = 0}^{T-1}\IND(N^i_{k, t}(x, a) \leq 1) \leq 2 (T 2|\mathcal{X}^i|) = 4T|\mathcal{X}^i|$.

Now, consider $N^i_{k, t}(x, a) > 1$. Hence,
\begin{align*}
& 
\sum_{k = 1}^{K} \sum_{t = 0}^{T-1}  \IND(N^i_{k, t}(X^i_{k, t}, A^i_{k, t}) > 1) \EXP\bigg[ \sqrt{\dfrac{1}{1 \vee N^i_{k, t}(X^i_{k, t}, A^i_{k, t})}} \Biggm| \mathcal{F}^{\star i}_{k, t}(X^i_{k, t}, A^i_{k, t}) \bigg] \\
& \leq 
\sum_{k = 1}^{K} \sum_{t = 0}^{T-1}  \EXP\bigg[ \sqrt{\dfrac{1}{N^i_{k, t}(X^i_{k, t}, A^i_{k, t})}} \Biggm| \mathcal{F}^{\star i}_{k, t}(X^i_{k, t}, A^i_{k, t}) \bigg] \leq
\sum_{(x^i_t, a^i_t, t)} \sum_{j = 1}^{N^i_{K+1, 0}(x^i_t, a^i_t)} \dfrac{1}{\sqrt{j}} \\
& \leq
\sqrt{2|\mathcal{X}^i|T \sum_{(x^i_t, a^i_t, t)} N^i_{K+1, 0}(x^i_t, a^i_t)} = \sqrt{2 |\mathcal{X}^i| TK}.
\end{align*}
Finally, we get
\begin{align} \label{eqn:rl-R2-1}
\sum_{k = 1}^{K} \sum_{t = 0}^{T-1}  \sum_{i = 1}^{N} \EXP\bigg[ \sqrt{\dfrac{\left(4 |\mathcal{X}|^i + \log (1/\delta) \right)}{1 \vee N^i_{k, t}(X^i_{k, t}, A^i_{k, t})}} \Biggm| \mathcal{F}^{\star i}_{k, t}(X^i_{k, t}, A^i_{k, t}) \bigg] & \leq \sum_{i=1}^{N} \sqrt{4 \left(|\mathcal{X}|^i + \log (1/\delta)\right)} \sqrt{2 |\mathcal{X}^i| KT} \notag \\
& = 2\sqrt{2} \sum_{i = 1}^{N} \sqrt{\left(|\mathcal{X}|^i + \log (1/\delta)\right) |\mathcal{X}^i| KT} \notag \\
& = 2\sqrt{2} \sum_{i = 1}^{N} \sqrt{\left((|\mathcal{X}|^i)^2 KT + \log (1/\delta) |\mathcal{X}^i| KT \right) } \notag \\
& \leq 2\sqrt{2} \sum_{i = 1}^{N} |\mathcal{X}|^i \sqrt{KT (1 + \log (1/\delta)) } \notag \\
& \leq 2\sqrt{2} N|\bar{\mathcal{X}}| \sqrt{KT (1 + \log (1/\delta)) }
\end{align}

The same approach works for the third term and we get
\begin{equation} \label{eqn:rl-R2-11}
\sum_{k = 1}^{K} \sum_{t = 0}^{T-1} \sum_{i = 1}^{N} \EXP\bigg[ \sqrt{\dfrac{2\sqrt{2} |\mathcal{X}|^i \log (1/\delta)}{1 \vee N^i_{k, t}(X^i_{k, t}, A^i_{k, t})}} \Biggm| \mathcal{F}^{i}_{k, t}(X^i_{k, t}, A^i_{k, t}) \bigg]
\leq 2\sqrt{2} N|\bar{\mathcal{X}}| \sqrt{KT (1 + \log (1/\delta)) }.
\end{equation}

Finally, by setting $\delta = 1/(KT)$, and substituting the upper-bounds, we have
\begin{multline*}
\sum_{k = 1}^{K} \sum_{t = 0}^{T-1} \EXP\Bigl[ \bigl\lVert
\bs{P}^{\star}_{t}( \cdot | \bs{X}_{k, t}, \bs{A}_{k, t})
-
\bs{P}_{k, t}(\cdot | \bs{X}_{k, t}, \bs{A}_{k, t}) 
\bigr\rVert_1 \Bigr] 
\\ \leq
\sum_{k = 1}^{K} \sum_{t = 0}^{T-1} \sum_{i = 1}^{N} \EXP\Bigl[ \bigl\lVert P^{\star i}_t(\cdot   | X^i_{k, t}, A^i_{k, t}) - P^i_{k, t}(\cdot   | X^i_{k, t}, A^i_{k, t}) \bigr\rVert_1 \Bigr] 
\\ \leq 
2\sqrt{2} N + 8 N|\bar{\mathcal{X}}| \sqrt{KT (1 + \log (KT)) } < 
12 N \bar{|\mathcal{X}|} \sqrt{KT (1 + \log (KT))}
\end{multline*}
Hence, the first term can be upper-bounded by the second term.
\Halmos

\section{Proof of Theorem \ref{thm:regret}} \label{app:thm-reg}

The technique used in Section~\ref{subsec:augment} can also be applied to the overall system prior to the decomposition. Let construct a time-dependent MDP for the overall system as $(\bar{\bs{\mathcal{X}}}, \bs{\mathcal{A}}, \{\bs{\tilde{P}}_t(\bs{a})\}_{\bs{a} \in \bs{\mathcal{A}}, t \in \mathcal{T}}, \{\bs{r}_t\}_{t\in\mathcal{T}}, \bs{x}_0)$ with $\bar{\bs{\mathcal{X}}}:=\bs{\mathcal{X}}\times \bs{\mathcal{S}}$, $\bs{\tilde{P}}_t(\bs{x}',\bs{s}'|\bs{x},\bs{s},\bs{a}):=\bs{P}_t(\bs{x}'|\bs{x},\bs{a})\mathbb{I}(\bs{s}',\bs{s}+\bs{r}_t(\bs{x},\bs{a}))$, $\bs{r}_t(\bs{x},\bs{s},\bs{a}):=\mathbb{I}(t,T-1)U_{\bs{\tau}}(\bs{s}+\bs{r}_t(\bs{x},\bs{a}))$ and $\bar{\bs{x}}_0:=(\bs{x_0},\bs{0})$.

The value function under optimal policy~$\bs{\pi}^\star$ for the overall system with the true parameter set~${\bs \theta}^\star$ is ${\bs V}^{{\bs \pi}^\star, {\bs \theta}^\star}_{t-1} = \langle \tilde{\bs P}_t^{{\bs \pi}^\star, {\bs \theta}^\star}, {\bs V}^{{\bs \pi}^\star, {\bs \theta}^\star}_{t} \rangle$.
The value function under the estimated policy~$\bs{\pi}^k$ for the overall system with the estimated parameter set~${\bs \theta}^k$ is ${\bs V}^{{\bs \pi}^k, {\bs \theta}^k}_{t-1} = \langle \tilde{\bs P}_t^{{\bs \pi}^k, {\bs \theta}^k}, {\bs V}^{{\bs \pi}^k, {\bs \theta}^k}_{k, t} \rangle$. And the value function under the estimated policy~$\bs{\pi}^k$ for the overall system with the true parameter set~${\bs \theta}^\star$ is ${\bs V}^{{\bs \pi}^k, {\bs \theta}^\star}_{t-1} = \langle \tilde{\bs P}_t^{{\bs \pi}^k, {\bs \theta}^\star}, {\bs V}^{{\bs \pi}^k, {\bs \theta}^\star}_{k, t} \rangle$.

Note that in the definition of the regret, the objectives are obtained using the parameters of the true MDP but the policy is the estimated one. Hence, we have
\begin{align}
\REGRET(K) 
& = \sum_{k = 1}^{K} \EXP\biggl[ 
\mathbb{E}\left[ \bs{D}_{{\bs x}_0}(\bs{\pi}^*) \right] 
-
\mathbb{E}\left[ \bs{D}_{{\bs x}_0}(\bs{\pi}_k) \right]
\biggr] \notag = \sum_{k = 1}^{K} 
\EXP\biggl[ {\bs V}^{{\bs \pi}^\star, {\bs \theta}^\star}_{0} ({\bs x}_0, {\bs s}_0) - {\bs V}^{{\bs \pi}^k, {\bs \theta}^\star}_{0}({\bs x}_0, {\bs s}_0)
\biggr] \notag \\
& = \sum_{k = 1}^{K} 
\EXP\biggl[ {\bs V}^{{\bs \pi}^k, {\bs \theta}^k}_{0} ({\bs x}_0, {\bs s}_0) - {\bs V}^{{\bs \pi}^k, {\bs \theta}^\star}_{0}({\bs x}_0, {\bs s}_0)
\biggr] \label{eq:RegDef}
\end{align}
where the last equality holds by Lemma~\ref{lem:TS1}.
Then, we present the main result of this section.

From Eq.~\eqref{eq:RegDef}, by adding and subtracting $\langle \tilde{\bs P}_t^{{\bs \pi}^k, {\bs \theta}^\star} {\bs 1}_{\{{\bs x}_0, {\bs s}_0\}}, {\bs V}^{{\bs \pi}^k, {\bs \theta}^k}_{1} \rangle$ we get 
\begin{align*}
\REGRET(K) 
& = \sum_{k = 1}^{K} \EXP\biggl[ 
{\bs V}^{{\bs \pi}^k, {\bs \theta}^k}_{0} ({\bs x}_0, {\bs s}_0) - {\bs V}^{{\bs \pi}^k, {\bs \theta}^\star}_{0}({\bs x}_0, {\bs s}_0)
\biggr] \\
& = \sum_{k = 1}^{K} \EXP\biggl[ 
\langle \tilde{\bs P}_t^{{\bs \pi}^k, {\bs \theta}^k} {\bs 1}_{\{{\bs x}_0, {\bs s}_0\}}, {\bs V}^{{\bs \pi}^k, {\bs \theta}^k}_{1} \rangle 
- 
\langle \tilde{\bs P}_t^{{\bs \pi}^k, {\bs \theta}^\star} {\bs 1}_{\{{\bs x}_0, {\bs s}_0\}}, {\bs V}^{{\bs \pi}^k, {\bs \theta}^\star}_{1} \rangle 
\\
& \qquad \qquad + 
\langle \tilde{\bs P}_t^{{\bs \pi}^k, {\bs \theta}^\star} {\bs 1}_{\{{\bs x}_0, {\bs s}_0\}}, {\bs V}^{{\bs \pi}^k, {\bs \theta}^k}_{1} \rangle 
- 
\langle \tilde{\bs P}_t^{{\bs \pi}^k, {\bs \theta}^\star} {\bs 1}_{\{{\bs x}_0, {\bs s}_0\}}, {\bs V}^{{\bs \pi}^k, {\bs \theta}^k}_{1} \rangle
\biggr] \\
& = \sum_{k = 1}^{K} \EXP\biggl[ 
\langle (\tilde{\bs P}_t^{{\bs \pi}^k, {\bs \theta}^k} - \tilde{\bs P}_t^{{\bs \pi}^k, {\bs \theta}^\star}) {\bs 1}_{\{{\bs x}_0, {\bs s}_0\}}, {\bs V}^{{\bs \pi}^k, {\bs \theta}^k}_{1} \rangle 
+ 
\langle \tilde{\bs P}_t^{{\bs \pi}^k, {\bs \theta}^\star} {\bs 1}_{\{{\bs x}_0, {\bs s}_0\}}, {\bs V}^{{\bs \pi}^k, {\bs \theta}^k}_{1} - {\bs V}^{{\bs \pi}^k, {\bs \theta}^\star}_{1} \rangle
\biggr].
\end{align*}
Then, by adding and subtracting $\left({\bs V}^{{\bs \pi}^k, {\bs \theta}^k}_{1} - {\bs V}^{{\bs \pi}^k, {\bs \theta}^\star}_{1}\right)({\bs X}_{k, 1}, {\bs S}_{k, 1})$, we get
\begin{align*}
\REGRET(K) 
& = \sum_{k = 1}^{K} \EXP\biggl[ 
\langle (\tilde{\bs P}_t^{{\bs \pi}^k, {\bs \theta}^k} - \tilde{\bs P}_t^{{\bs \pi}^k, {\bs \theta}^\star}) {\bs 1}_{\{{\bs x}_0, {\bs s}_0\}}, {\bs V}^{{\bs \pi}^k, {\bs \theta}^k}_{1} \rangle 
+ 
\langle \tilde{\bs P}_t^{{\bs \pi}^k, {\bs \theta}^\star} {\bs 1}_{\{{\bs x}_0, {\bs s}_0\}}, {\bs V}^{{\bs \pi}^k, {\bs \theta}^k}_{1} - {\bs V}^{{\bs \pi}^k, {\bs \theta}^\star}_{1} \rangle \\
& \qquad \qquad + \left({\bs V}^{{\bs \pi}^k, {\bs \theta}^k}_{1} - {\bs V}^{{\bs \pi}^k, {\bs \theta}^\star}_{1}\right)({\bs X}_{k, 1}, {\bs S}_{k, 1}) - \left({\bs V}^{{\bs \pi}^k, {\bs \theta}^k}_{1} - {\bs V}^{{\bs \pi}^k, {\bs \theta}^\star}_{1}\right)({\bs X}_{k, 1}, {\bs S}_{k, 1})
\biggr] \\
& = \sum_{k = 1}^{K} \EXP\biggl[ 
\langle (\tilde{\bs P}_t^{{\bs \pi}^k, {\bs \theta}^k} - \tilde{\bs P}_t^{{\bs \pi}^k, {\bs \theta}^\star}) {\bs 1}_{\{{\bs x}_0, {\bs s}_0\}}, {\bs V}^{{\bs \pi}^k, {\bs \theta}^k}_{1} \rangle 
+
\left( {\bs V}^{{\bs \pi}^k, {\bs \theta}^k}_{1} - {\bs V}^{{\bs \pi}^k, {\bs \theta}^\star}_{1} \right)({\bs X}_{k, 1}, {\bs S}_{k, 1})
\biggr]
\end{align*}

where we have used the fact that for any arbitrary policy ${\bs \pi}^k$,
\begin{align}
\EXP\biggl[ 
\langle \tilde{\bs P}_t^{{\bs \pi}^k , {\bs \theta}^\star} {\bs 1}_{\{{\bs x}_0, {\bs s}_0\}}, {\bs V}^{{\bs \pi}^k, {\bs \theta}^k}_{1} \rangle
\biggr] & = \EXP[{\bs V}^{{\bs \pi}^k, {\bs \theta}^k}_{1} ({\bs X}_{k, 1}, {\bs S}_{k, 1})], \notag \\
\EXP\biggl[ 
\langle \tilde{\bs P}_t^{{\bs \pi}^k , {\bs \theta}^\star} {\bs 1}_{\{{\bs x}_0, {\bs s}_0\}}, {\bs V}^{{\bs \pi}^k, {\bs \theta}^\star}_{1} \rangle
\biggr] & = \EXP[{\bs V}^{{\bs \pi}^k, {\bs \theta}^\star}_{1} ({\bs X}_{k, 1}, {\bs S}_{k, 1})] \label{eq:Reg2}
\end{align}
and hence,
\[
\EXP\biggl[ 
\langle \tilde{\bs P}_t^{{\bs \pi}^k , {\bs \theta}^\star} {\bs 1}_{\{{\bs x}_0, {\bs s}_0\}}, {\bs V}^{{\bs \pi}^k, {\bs \theta}^k}_{1} - {\bs V}^{{\bs \pi}^k, {\bs \theta}^\star}_{1} \rangle - \left( {\bs V}^{{\bs \pi}^k, {\bs \theta}^k}_{1} - {\bs V}^{{\bs \pi}^k, {\bs \theta}^\star}_{1} \right)({\bs X}_{k, 1}, {\bs S}_{k, 1})
\biggr] = 0.
\]
Now, recursively, the second inner term of \eqref{eq:Reg2} can be decomposed until the end of the finite time horizon.
Therefore, for a horizon of length $T$, we have
\begin{align*}
\REGRET(K) 
& = \sum_{k = 1}^{K} \EXP\biggl[ 
{\bs V}^{{\bs \pi}^k, {\bs \theta}^k}_{0} ({\bs x}_0, {\bs s}_0) - {\bs V}^{{\bs \pi}^k, {\bs \theta}^\star}_{0}({\bs x}_0, {\bs s}_0)
\biggr] \\
& = \sum_{k = 1}^{K} \sum_{t = 0}^{T-1} \EXP\biggl[ 
\langle (\tilde{\bs P}_t^{{\bs \pi}^k, {\bs \theta}^k} - \tilde{\bs P}_t^{{\bs \pi}^k, {\bs \theta}^\star}) {\bs 1}_{\{{\bs X}_{k, t}, {\bs S}_{k, t}\}}, {\bs V}^{{\bs \pi}^k, {\bs \theta}^k}_{k, t+1} \rangle
\biggr].
\end{align*}

As the per-step reward is upper-bounded by $r_{\max}$, then $|{\bs V}^{{\bs \pi}^k, {\bs \theta}^k}_{k, t+1} ({\bs X}_{k, t}, {\bs S}_{k, t}) | \leq N T r_{\max}$ for any $t$,
\begin{align*}
\REGRET(K) & \leq N T r_{\max} \sum_{k = 1}^{K} \sum_{t = 0}^{T-1} \EXP\biggl[ 
\bigl\lVert (\tilde{\bs P}_t^{{\bs \pi}^k, {\bs \theta}^k} - \tilde{\bs P}_t^{{\bs \pi}^k, {\bs \theta}^\star}) {\bs 1}_{\{{\bs X}_{k, t}, {\bs S}_{k, t}\}} \bigr\rVert_1 \biggr] \\
& = N T r_{\max} \sum_{k = 1}^{K} \sum_{t = 0}^{T-1} \EXP\biggl[ 
\bigl\lVert ({\bs P}^{{\bs \pi}^k}_{k, t} - {\bs P}_t^{\star {\bs \pi}^k}) {\bs 1}_{\{{\bs X}_{k, t}\}} \bigl\rVert_1 \biggr] \\ 
& \leq N T r_{\max} \sum_{k = 1}^{K} \sum_{t = 0}^{T-1} \EXP\biggl[ 
\bigl\lVert \bs{P}^{\star}_t(\cdot   | \bs{X}_{k, t}, \bs{A}_{k, t}) - \bs{P}_{k, t}(\cdot   | \bs{X}_{k, t}, \bs{A}_{k, t}) \bigr\rVert_1 \biggr] \\
& \leq 12 N^2 T r_{\max} \bar{|\mathcal{X}|} \sqrt{KT (1 + \log (KT))}.
\end{align*}
where the last inequality holds by Lemma~\ref{lemma:rl-Pdiff-ns}.\Halmos

\section{Proof of Theorem \ref{thm:discounted-finite}} \label{app:discounted-finite}

We prove part (1) in two steps. For \(n=0\), by definition we have
\[
V_{\lambda, 0}(x,y,z)=\max_{\pi\in\Pi_H^\infty} \mathbb{E}^{\pi}_x \Biggl[ U\Bigl(y+\sum_{k=0}^{-1}z\beta^k r(X_k,A_k)\Bigr) - \lambda \sum_{k=0}^{-1}z\beta^k A_k \Biggr]=U(y).
\]
Now turning to an arbitrary $n>0$, we have that
\begin{align*}
    V_{\lambda, n}&(x,y,z)=\max_{\pi\in\Pi_H^\infty} \mathbb{E}^{\pi}_x \Biggl[ U\Bigl(y+\sum_{k=0}^{n-1}z\beta^k r(X_k,A_k)\Bigr) - \lambda \sum_{k=0}^{n-1}z\beta^k A_k \Biggr]\\
&=\max_{\pi\in\Pi_H^\infty} \sum_{x'\in\mathcal{X}} \mathbb{E}_x^\pi\Biggl[ U\Bigl(y+\sum_{k=0}^{n-1}z\beta^k r(X_k,A_k)\Bigr) - \lambda \sum_{k=0}^{n-1}z\beta^k A_k \Biggm| x_1=x' \Biggr]P(x'|x,\pi_0(x))\\    
    &=\max_{a\in\mathcal{A}} \sum_{x'\in\mathcal{X}} \max_{\pi\in\Pi_H^\infty}\mathbb{E}_{x'}^\pi\Biggl[ U\Bigl(y+zr(x,a)+\sum_{k=0}^{n-1}(z\beta)\beta^k r(X_k,A_k)\Bigr) -\lambda za - \lambda \sum_{k=0}^{n-1}(z\beta)\beta^{k} A_k  \Biggr]P(x'|x,a)\\    
&= \max_{a\in\mathcal{A}} \sum_{x'\in\mathcal{X}} \Bigl[ V_{\lambda, n-1}\Bigl(x',  y+z r(x,a),  z\beta\Bigr) - \lambda z a \Bigr] P(x'|x,a)\\
    &=(\BellOp_\lambda v)(x,y,z).
\end{align*}

For the second part, we will prove by induction that for all $n\geq 1$, and all$(x,y,z)\in\hat{\mathcal{X}}$, the policy $\pi_{\lambda|x,y,z,n}^*=(\pi_{\lambda,0|x,y,z,n}^*,\pi_{\lambda,1|x,y,z,n}^*,\dots,\pi_{\lambda,n-1|x,y,z,n}^*)$, where we indicate with  $|x,y,z,n$ the fact that it was constructed for an initialization at $(x,y,z)$ with a horizon of $n$, satisfies
\begin{align}
    \mathbb{E}^{\pi_{\lambda|x,y,z,n}^*}_x \Biggl[ U\Bigl(y+\sum_{k=0}^{n-1}z\beta^k r(X_k,A_k)\Bigr) - \lambda \sum_{k=0}^{n-1}z\beta^k A_k \Biggr] = V_{\lambda, n}(x,y,z).\label{eq:policyMatchN}
\end{align}
Starting at $n=1$, we have that:
\begin{align*}
    \mathbb{E}^{\pi_{\lambda|x,y,z,1}^*}_x &\Biggl[ U\Bigl(y+z r(x,A_0)\Bigr) - \lambda z A_0 \Biggr]\\
    &=\sum_{x'} \Biggl[ V_{\lambda, 0}(x,y+zr(x,f_{\lambda, 1}^*(x,y,z)),\beta z) - \lambda z f_{\lambda, 1}^*(x,y,z)\Biggr] P(x'|x,f_{\lambda, 1}^*(x,y,z)) \\
    &= (\BellOp_{f_{\lambda, 1}^*,\lambda} V_{\lambda, 0})(x,y,z) = V_{\lambda, 1}(x,y,z).
\end{align*}
Similarly, given that condition \eqref{eq:policyMatchN} is satisfied at $n-1$ for all $(x,y,z)$, we can inductively establish at $n$ that:
\begin{align*}
    \mathbb{E}^{\pi_{\lambda|x,y,z,n}^*}_x &\Biggl[ U\Bigl(y+\sum_{k=0}^{n-1}z\beta^k r(X_k,A_k)\Bigr) - \lambda \sum_{k=0}^{n-1}z\beta^k A_k \Biggr]\\
    &=\sum_{x'} \mathbb{E}^{\pi_{\lambda|x,y,z,n}^*}_x \Biggl[ U\Bigl(y+\sum_{k=0}^{n-1}z\beta^k r(X_k,A_k)\Bigr) - \lambda \sum_{k=0}^{n-1}z\beta^k A_k \Biggm| x_1=x' \Biggr] P(x'|x,f_{\lambda, n}^*(x,y,z)) \\
    &=\sum_{x'} \mathbb{E}^{\pi_{\lambda|x',y',z',n-1}^*}_{x'} \Biggl[ U\Bigl(y+zr(x,f_{\lambda, n}^*(x,y,z))+\sum_{k=0}^{n-2}z\beta^{k+1} r(X_k,A_k)\Bigr) \\
    &\qquad-\lambda zf_{\lambda, n}^*(x,y,z) - \lambda \sum_{k=0}^{n-2}z\beta^{k+1} A_k  \Biggr] P(x'|x,f_{\lambda, n}^*(x,y,z)) \\  
    &=\sum_{x'} \Biggl[V_{\lambda, n-1}(x',y+zr(x,f_{\lambda, n}^*(x,y,z)),\beta z) -\lambda zf_{\lambda, n}^*(x,y,z) \Biggr] P(x'|x,f_{\lambda, n}^*(x,y,z)) \\   
    &= (\BellOp_{f_{\lambda, n}^*,\lambda} V_{\lambda, n-1})(x,y,z) = V_{\lambda, n}(x,y,z).
\end{align*}
with $y':=y+zr(x,f_{\lambda, n}^*(x,y,z))$ and $z':=z\beta$, and where we exploit:
\[\pi_{\lambda,k|x,y,z,n}^*(h_k)=f_{n-k,\lambda}^*\Bigl(x_k,  y+z\sum_{k'=0}^{k-1}\beta^{k'} r(X_{k'},A_{k'}), z \beta^k\Bigr)=\pi_{\lambda,k-1|x_1,y+zr(x_0,a_0),\beta z,n-1}^*([x_1, a_1,\dots, x_k]).\Halmos\]

\section{Proof of Theorem \ref{thm:inf-convergence-policy2}} \label{app:discounted-infinite}

We follow similar steps as in the proof of Theorem 3 of \cite{bauerle2014more}.

\textbf{Step 1. $V_{\lambda, \infty}(x,y,z)\in\mathbb{R}$} \\
For any $(x,y,z)\in\hat{\mathcal{X}}$, $\lambda\geq 0$, and $\pi\in\Pi_H^\infty$, we let 
\[V_{\lambda, n}^\pi(x,y,z):=\mathbb{E}_x^\pi\Biggl[ U\Bigl(y+z\sum_{t=0}^{n-1}\beta^t r(X_t,A_t)\Bigr) - \lambda\sum_{t=0}^{n-1}z\beta^t A_t  \Biggr].\]
One can show that the limit of $V_{\lambda, n}^\pi(x,y,z)$ exists as follows. First, for all $n'\geq n\geq 0$, we have that 
\begin{align*}
    V_{\lambda, n}^\pi(x,y,z)&=\mathbb{E}^{\pi}_x\Bigl[ U\Bigl(y+z\sum_{t=0}^{n-1}\beta^t r(X_t,A_t)\Bigr)-\lambda \sum_{t=0}^{n-1}z\beta^t A_t \Bigr]\\
    &\geq \mathbb{E}^{\pi}_x\Bigl[ U\Bigl(y+z\sum_{t=0}^{n'-1}\beta^t r(X_t,A_t) - z\beta^n r_{\max}\frac{1-\beta^{(n'-n)}}{1-\beta}\Bigr)-\lambda\sum_{t=0}^{n'-1}z\beta^t A_t \Bigr]\\
    &\geq \mathbb{E}^{\pi}_x\Bigl[ U\Bigl(y+z\sum_{t=0}^{n'-1}\beta^t r(X_t,A_t) \Bigr)-\lambda\sum_{t=0}^{n'-1}\beta^t A_t \Bigr] - Lr_{\max}z\beta^n \frac{1-\beta^{(n'-n)}}{1-\beta}\\
    &\geq V_{n',\lambda}^\pi(x,y,z) - \varepsilon_{\lambda, n}(x,y,z)    
    \end{align*}
where we used  the Lipchitz property of \(U\), i.e. for any \(y_1 \ge y_2 \ge0\), we have
\[
U(y_1)\le U(y_2)+L(y_1-y_2),
\]
for some $L\geq 0$, and where $\varepsilon_{\lambda, n}(x,y,z):=  (Lr_{\max}+\lambda)z\beta^n (1-\beta)^{-1}$. On the other hand,
    \begin{align*}
    V_{n',\lambda}^\pi(x,y,z)&=\mathbb{E}^{\pi}_x\Bigl[ U\Bigl(y+z\sum_{t=0}^{n'-1}\beta^t r(X_t,A_t) \Bigr)-\lambda\sum_{t=0}^{n'-1}\beta^t A_t \Bigr]\\
    &\geq \mathbb{E}^{\pi}_x\Bigl[ U\Bigl(y+z\sum_{t=0}^{n-1}\beta^t r(X_t,A_t) - z\beta^n r_{\min}\frac{1-\beta^{(n'-n)}}{1-\beta}\Bigr)-\lambda\sum_{t=0}^{n-1}\beta^t A_t \Bigr] \\
    &\geq \mathbb{E}^{\pi}_x\Bigl[ U\Bigl(y+z\sum_{t=0}^{n-1}\beta^t r(X_t,A_t) - z\beta^n r_{\max}\frac{1-\beta^{(n'-n)}}{1-\beta}\Bigr)-\lambda\sum_{t=0}^{n-1}\beta^t A_t \Bigr] \\    
    &\geq \mathbb{E}^{\pi}_x\Bigl[ U\Bigl(y+z\sum_{t=0}^{n-1}\beta^t r(X_t,A_t) \Bigr)-\lambda\sum_{t=0}^{n-1}\beta^t A_t \Bigr] -  (Lr_{\max}+\lambda)z\beta^n \frac{1-\beta^{(n'-n)}}{1-\beta}\\
    &\geq V_{\lambda, n}^\pi(x,y,z)-\varepsilon_{\lambda, n}(x,y,z).
\end{align*}
where we used $0\leq r_{\min}\leq r_{\max}$.
This implies that:
\begin{equation}\label{eq:nandnprimebounds}
    V_{\lambda, n}^\pi(x,y,z)-\varepsilon_{\lambda, n}(x,y,z)\leq  V_{n',\lambda}^\pi(x,y,z) \leq V_{\lambda, n}^\pi(x,y,z)+\varepsilon_{\lambda, n}(x,y,z).
\end{equation}
As $\varepsilon_{\lambda, n}(x,y,z)\rightarrow 0$ as $n\rightarrow \infty$, one can conclude that $\{V_{\lambda, n}^\pi(x,y,z)\}_{n=0}^\infty$ is a Cauchy sequence and must therefore converge to a value in $\mathbb{R}$.

Since we also have established that for all $n>0$ and $\pi\in\Pi_H^\infty$:
\begin{align*}
    \infty & < U(y)-(Lr_{\max}+\lambda)\ (1-\beta)^{-1} \leq U(y)-\varepsilon_{\lambda, 0}(x,y,z) \\
    & \leq V_{\lambda, n}^\pi(x,y,z)\leq U(y)+\varepsilon_{\lambda, 0}(x,y,z) \leq U(y)+(Lr_{\max}+\lambda) (1-\beta)^{-1} <\infty.
\end{align*}
We must therefore have:
\begin{align*}
    \infty & < U(y)-(Lr_{\max}+\lambda) (1-\beta)^{-1} \leq \lim_{n\rightarrow\infty}V_{\lambda, n}^\pi(x,y,z) \\
    & \leq U(y)+(Lr_{\max}+\lambda) (1-\beta)^{-1} <\infty \quad \forall \pi\in\Pi_H^\infty.
\end{align*}
The supremum of $\lim_{n\rightarrow \infty} V_{\lambda, n}^\pi(x,y,z)$ over all $\pi\in\Pi_H^\infty$ must therefore exist in the reals.

\textbf{Step 2. $\BellOp^n V_{\lambda, 0}\rightarrow V_{\lambda, \infty}$}

 For any $(x,y,z)\in\hat{\mathcal{X}}$, $\pi\in\Pi_H^\infty$, and $n'>n>0$, on can take the limit as $n'\rightarrow \infty$ and the supremum over all $\pi\in\bar{\Pi}_{H}$ in equation \eqref{eq:nandnprimebounds} to get 
 \begin{equation}
V_{\lambda, n}(x,y,z)-\varepsilon_{\lambda, n}(x,y,z)\leq V_{\lambda, \infty}(x,y,z)\leq V_{\lambda, n}(x,y,z)+\varepsilon_{\lambda, n}(x,y,z).     \label{eq:sandwich}
 \end{equation}
Letting $n\rightarrow\infty$ gives $V_{\lambda, n}(x,y,z)\rightarrow V_{\lambda, \infty}(x,y,z)$ as $\varepsilon_{\lambda, n}(x,y,z)\rightarrow 0$.

\textbf{Step 3. $V_{\lambda, \infty}=\BellOp_\lambda V_{\lambda, \infty}$} 

We can consider any $n\geq 0$ and exploit $|V_{\lambda, \infty}- V_{\lambda, n}|\leq \varepsilon_{\lambda, n}$ to obtain:
\begin{align*}
    (\BellOp_\lambda V_{\lambda, \infty})(x,y,z) &= \max_{a\in\mathcal{A}} \sum_{x'\in\mathcal{X}} \Bigl[ V_{\lambda, \infty}\Bigl(x',  y+z r(x,a),  z\beta\Bigr) - \lambda z a \Bigr] P(x'|x,a)\\
    &\leq \max_{a\in\mathcal{A}} \sum_{x'\in\mathcal{X}} \Bigl[ V_{\lambda, n}\Bigl(x',  y+z r(x,a),  z\beta\Bigr) + \varepsilon_{\lambda, n}(x,y+z r(x,a),  z\beta) - \lambda z a \Bigr] P(x'|x,a)\\
    &= V_{\lambda, n+1}(x,y,z) + \varepsilon_{\lambda, n+1}(x,y,z)\\
    &\leq V_{\lambda, \infty}(x,y,z) + 2\varepsilon_{\lambda, n+1}(x,y,z).
\end{align*}
and similarly
\[(\BellOp_\lambda V_{\lambda, \infty})(x,y,z) \geq V_{\lambda, \infty}(x,y,z)  - 2\varepsilon_{\lambda, n+1}(x,y,z).\]
This implies that 
\[|\BellOp_\lambda V_{\lambda, \infty})(x,y,z)- V_{\lambda, \infty}(x,y,z)|\leq 2\varepsilon_{\lambda, n+1}(x,y,z)\]
for all $n>0$. Hence, $\BellOp_\lambda V_{\lambda, \infty}= V_{\lambda, \infty}$.

\textbf{Step 4. Optimality of $\pi_\lambda^*$}

Let \(f_{\lambda}^*\) satisfy  
\(
\BellOp_{\lambda}V_{\lambda, \infty}
=\BellOp_{f_{\lambda}^*,\lambda}V_{\lambda, \infty}.
\)  
Define the history‐dependent policy \(\pi_{\lambda}^*\) by  
\[
\pi_{\lambda, 0}^*(\hat x)=f_{\lambda}^*(\hat x,\hat y,\hat z),
\quad
\pi_{\lambda, n}^*(h_n)
=f_{\lambda}^*\bigl(x_n, \hat y+\hat z\sum_{t=0}^{n-1}\beta^t r(x_t,a_t), \hat z\beta^n\bigr).
\]  

Based on Theorem \ref{thm:discounted-finite}, we have that 
\begin{align*}
    V_{\lambda, n}^{\pi_\lambda^*}(x,y,z)&=\mathbb{E}_x^{\pi_\lambda^*}\Biggl[ U\Bigl(y+z\sum_{t=0}^{n-1}\beta^t r(X_t,A_t)\Bigr) - \lambda\sum_{t=0}^{n-1}z\beta^t A_t  \Biggr] \\
    &=(\BellOp_{f_\lambda^*,\lambda}^n V_{\lambda, 0})(x,y,z) = V_{\lambda,n}(x,y,z).
\end{align*}

Moreover, exploiting equation \eqref{eq:sandwich}, we get $\bigl|V_{\lambda, n}^{\pi_{\lambda}^*}-V_{\lambda, \infty}\bigr|
\le\varepsilon_{\lambda, n}$. Letting \(n\to\infty\) shows  that
\[
\lim_{n\to\infty}
\mathbb E_x^{\pi_{\lambda}^*}\Bigl[ U\bigl(\hat y+\hat z\sum_{t=0}^{n-1}\beta^t r(X_t,A_t)\bigr)
-\lambda\sum_{t=0}^{n-1}\hat z\beta^t A_t\Bigr]
=V_{\lambda, \infty}(x,\hat y,\hat z),
\]  
so \(\pi_{\lambda}^*\) attains the supremum in \eqref{eq:infHorizonProbExt}. \Halmos

\section{Proof of Lemma \ref{thm:monotoneSuperAdd}:}
Recall the definition:
\[Q_{\lambda, \infty}(x,y,z,a):=- \lambda z a + \sum_{x'\in\mathcal{X}} \Bigl[ V_{\lambda, \infty}\Bigl(x',  y+z r(x,a),  z\beta\Bigr)  \Bigr] P(x'|x,a).\]

Consider the family of decision rules $f_\lambda^*(x,y,z):=\min(\arg\max_a Q_{\lambda, \infty}(x,y,z,a))$. We first show that $f_\lambda^*$ satisfies $\mathcal{T}_\lambda V_{\lambda, \infty}=\mathcal{T}_{f_\lambda^*,\lambda} V_{\lambda, \infty}$. Namely,
\begin{align*}
    (\mathcal{T}_{f_\lambda^*,\lambda}V_{\lambda, \infty})(x,y,z) &= \sum_{x'\in\mathcal{X}} \Bigl[ V_{\lambda, \infty}\Bigl(x',  y+z r(x,f_\lambda^*(x,y,z)),  z\beta\Bigr) - \lambda z f_{\lambda}^*(x,y,z) \Bigr] P(x'|x,f_{\lambda}^*(x,y,z))\\
    &=Q_{\lambda, \infty}(x,y,z,f_\lambda^*(x,y,z))\\
    &=\max_a Q_{\lambda, \infty}(x,y,z,a)\\
    &=\max_a \sum_{x'\in\mathcal{X}} \Bigl[ V_{\lambda, \infty}\Bigl(x',  y+z r(x,a),  z\beta\Bigr) - \lambda z a \Bigr] P(x'|x,a) = \mathcal{T}_\lambda V_{\lambda, \infty}(x,y,z).
\end{align*}

Next, we show that the decision rule is monotone in $\lambda$. For any $(x,y,z)$, consider $\lambda_2\geq \lambda_1$. 

If $f_{\lambda_1}^*(x,y,z)=1$, then $f_{\lambda_2}^*(x,y,z)\leq 1 = f_{\lambda_1}^*(x,y,z)$ trivially.

If $f_{\lambda_1}^*(x,y,z)=0$, then by optimality we have $Q_{\infty,\lambda_1}(x,y,z,0) \geq Q_{\infty,\lambda_1}(x,y,z,1)$.

By the superadditivity condition \eqref{cond:superAdditiveQ}, 
we have, for $\lambda_1 \leq \lambda_2$:
\[0 \geq Q_{\infty,\lambda_1}(x,y,z,1)-Q_{\infty,\lambda_1}(x,y,z,0) \geq Q_{\infty,\lambda_2}(x,y,z,1)-Q_{\infty,\lambda_2}(x,y,z,0).\]

Therefore, $Q_{\infty,\lambda_2}(x,y,z,0) \geq Q_{\infty,\lambda_2}(x,y,z,1)$, which implies that
\(f_{\lambda_2}^*(x,y,z) \leq 0 = f_{\lambda_1}^*(x,y,z).\)

Thus, the decision rule is non-increasing in $\lambda$. \Halmos

\removed{
Next, we present a corollary of Theorem \ref{thm:inf-convergence-policy2}.

\begin{corollary}\label{thm:convergeQinfty}
For all $(x,y,z)\in\hat{\mathcal{X}}$, we have that $Q_{\lambda, n}(x,y,z,a)\rightarrow Q_{\lambda, \infty}(x,y,z,a)$ and $a\in\mathcal{A}$.
\end{corollary}

\begin{proof}{Proof:}
    This follows from the definition of $Q_{\lambda, n}$ and $Q_{\lambda, \infty}$, and from Theorem \ref{thm:inf-convergence-policy2}. Namely,
\[Q_{\lambda, n}(x,y,z,a)=\sum_{x'\in\mathcal{X}} \Bigl[ V_{\lambda, n-1}\Bigl(x',  y+z r(x,a),  z\beta\Bigr) - \lambda z a \Bigr] P(x'|x,a)\]
Taking the limit on both sides of the equality we get:
\begin{align*}
\lim_{n\rightarrow\infty}Q_{\lambda, n}(x,y,z,a)&=\lim_{n\rightarrow\infty}\sum_{x'\in\mathcal{X}} \Bigl[ V_{\lambda, n-1}\Bigl(x',  y+z r(x,a),  z\beta\Bigr) - \lambda z a \Bigr] P(x'|x,a)\\
&=\sum_{x'\in\mathcal{X}} \Bigl[ \lim_{n\rightarrow\infty}V_{\lambda, n-1}\Bigl(x',  y+z r(x,a),  z\beta\Bigr) - \lambda z a \Bigr] P(x'|x,a)\\
&=\sum_{x'\in\mathcal{X}} \Bigl[ V_{\lambda, \infty}\Bigl(x',  y+z r(x,a),  z\beta\Bigr) - \lambda z a \Bigr] P(x'|x,a) =Q_{\lambda, \infty}(x,y,z,a).   \Halmos
\end{align*}
\end{proof}}

\section{Proof of Theorem~\ref{thm:inf-indexable}}\label{app:proofThm6}

Consider the following finite horizon version of the Q function:
\[Q_{\lambda,n}(x,y,z,a):=- \lambda z a + \sum_{x'\in\mathcal{X}} \Bigl[ V_{\lambda, n-1}\Bigl(x',  y+z r(x),  z\beta\Bigr)  \Bigr] P(x'|x,a).\]
Reusing the notation of the proof of Lemma \ref{lem:monotone-policy} (see Appendix \ref{app:thm-ns-indexable}), we observe that 
\[Q_{\lambda,n}(x,y,z,a):=- \lambda z a + \sum_{x'\in\mathcal{X}} \Bigl[ W_0^\beta\Bigl(x',  y+z r(x),-\lambda z \beta (n-1)\Bigr)  \Bigr] P(x'|x,a)\]
with $W_0^\beta(x,s,\phi)$ defined on a finite horizon \NimaEdits{$T$ $=$ $n$ $-$ $1$} MDP using the tuple $(\mathcal{X},\mathcal{A},\{P_t(a)\}_{a\in\{0,1\}},\{\hat{r}_t\}_{t\in\mathcal{T}},x_0)$, where $\hat{r}_t(x,a):=z\beta r(x)$. Given that this $T=n-1$ MDP satisfies all conditions of Assumption \ref{ass:mdp}, we conclude from Lemma \ref{thm:WtProperties} that $W_0^\beta(x,s,\phi)$ is super-additive with respect to $(x,\phi)$. This further can be used to verify that $Q_{\lambda,n}(x,y,z,a)$ satisfies condition \eqref{cond:superAdditiveQ}. Namely, for $0\leq \lambda_1 \leq \lambda_2$:
\begin{align*}
    Q_{\lambda_1,n}&(x,y,z,1)-Q_{\lambda_1,n}(x,y,z,0) \\
    &= -\lambda_1 z+ \mathbb{E}_{p_{1}}[W_{0}^\beta(x', y + z r(x), -\lambda_1 z \beta(n-1))]-\mathbb{E}_{p_{2}}[W_{0}^\beta(x', y + z r(x), -\lambda_1 z \beta(n-1))]\\
    &\geq -\lambda_2 z+ \mathbb{E}_{p_{1}}[W_{0}^\beta(x', y + z r(x), -\lambda_2 z \beta(n-1))]-\mathbb{E}_{p_{2}}[W_{0}^\beta(x', y + z r(x), -\lambda_2 z \beta(n-1))]\\
    &= Q_{\lambda_2,n}^\beta(x,y,z,1)-Q_{\lambda_2,n}^\beta(x,y,z,0), 
\end{align*} 
where $p_1(x'):=P(x'|x,1)$ and $p_2(x'):=P(x'|x,0)$, using  Lemma \ref{lem:sa1} with Assumption \ref{ass:p1} and the super-additivity of $W_0^\beta(x,s,\phi)$ with respect to $(x,\phi)$.



Taking the limit as $n$ goes to infinity we first get:
\begin{align*}
\lim_{n\rightarrow\infty}Q_{\lambda, n}(x,y,z,a)&=\lim_{n\rightarrow\infty}\sum_{x'\in\mathcal{X}} \Bigl[ V_{\lambda, n-1}\Bigl(x',  y+z r(x,a),  z\beta\Bigr) - \lambda z a \Bigr] P(x'|x,a)\\
&=\sum_{x'\in\mathcal{X}} \Bigl[ \lim_{n\rightarrow\infty}V_{\lambda, n-1}\Bigl(x',  y+z r(x,a),  z\beta\Bigr) - \lambda z a \Bigr] P(x'|x,a)\\
&=\sum_{x'\in\mathcal{X}} \Bigl[ V_{\lambda, \infty}\Bigl(x',  y+z r(x,a),  z\beta\Bigr) - \lambda z a \Bigr] P(x'|x,a) =Q_{\lambda, \infty}(x,y,z,a).
\end{align*}
where we exploit the convergence result from Theorem \ref{thm:inf-convergence-policy2}. The fact that $Q_{\lambda, n}(x,y,z,a)$ satisfies condition \eqref{cond:superAdditiveQ} can therefore be extended to $Q_{\lambda, \infty}(x,y,z,a)$. Indeed, for all $0\leq \lambda_1\leq \lambda_2$:
\begin{align*}
Q_{\infty,\lambda_1}(x,y,z,1) - Q_{\infty,\lambda_1}(x,y,z,0) &= \lim_{n \to \infty} \bigl[Q_{n,\lambda_1}(x,y,z,1) - Q_{n,\lambda_1}(x,y,z,0)\bigr] \\
&\geq \lim_{n \to \infty} \bigl[Q_{n,\lambda_2}(x,y,z,1) - Q_{n,\lambda_2}(x,y,z,0)\bigr]\\
&= Q_{\infty,\lambda_2}(x,y,z,1) - Q_{\infty,\lambda_2}(x,y,z,0).
\end{align*}

Thus, condition~\eqref{cond:superAdditiveQ} is satisfied for the infinite-horizon Q-function. By Lemma~\ref{thm:monotoneSuperAdd}, there exists a family of optimal decision rules $\{f_\lambda^*\}_{\lambda \in [0,\infty)}$ or the augmented risk neutral MDP associated to the
arm that is monotone in $\lambda$. The rest of the proof follows exactly as for the proof of Theorem \ref{thm:ns-indexable}. \Halmos



\section{Proof of Theorem \ref{thm:inf-indexable2}}

This proof follows exactly the same steps as the proof of Theorem \ref{thm:ns-indexable2}.
Recall that when condition \ref{ass:r2} is satisfied, the Bellman equation associated to the augmented arm risk neutral discounted finite horizon  MDP is: 
\[V_{\lambda, n}(x,y,z)   = \max_{a\in\mathcal{A}} \sum_{x'\in\mathcal{X}} \Bigl[ V_{\lambda,n-1}\Bigl(x',  y+z r(x),  z\beta\Bigr) - \lambda z a \Bigr] P(x'|x,a)\]
for any $n\geq 1$, augmented state $(x,y,z)\in\hat{\mathcal{X}}$, $\lambda\in\mathbb{R}_+$, whereas $V_{\lambda,0}(x,y,z) = U(y)$. 

\begin{lemma}[Difference of Value Bound]
For any time $n\geq 0$, any state $(x,y,z)$ and any $\lambda_2\geq \lambda_1\geq 0$. the difference in value is bounded by: 
\[
-\frac{z}{1-\beta}(\lambda_2-\lambda_1) \leq V_{\lambda_2,n}(x, y,z) - V_{\lambda_1,n}(x, y,z) \leq 0.
\]
\end{lemma}
\textbf{\textit{Proof.}}
The proof is based on backward induction.

\textbf{Base Case:} We start at $n=0$, the Bellman equation is $V_{\lambda,0}(x, y,z) = U(y)$. Hence, we have that $V_{\lambda_2,0}(x,y,z) - V_{\lambda_1,T-1}(x,y,z)=0$, which is clearly inside $[-z(\lambda_1-\lambda_2)/(1-\beta),0]$ since $-z(\lambda_1-\lambda_2)/(1-\beta)\leq 0$. Thus, the bounds hold.

\textbf{Inductive Step:} Assume the bounds hold for time $t+1$, one can first verify at $t$ that:\EDcomment{There is confusion between using $P(x'|x,a)$ or $p(x'|x,a)$} \NimaResponse{Noticed and tried to resolve it everywhere.}
\begin{align*}
   V_{\lambda_2,n}&(x,y,z) - V_{\lambda_1,n}(x,y,z) =  \max_{a \in \{0,1\}} \left\{ -\lambda_2 z a + \sum_{x'} P(x'|x, a) V_{\lambda_2,n-1}(x', y+zr(x),z\beta) \right\} \\
   &\qquad\qquad- \max_{a \in \{0,1\}} \left\{ -\lambda_1 z a + \sum_{x'} P(x'|x, a) V_{\lambda_1,n-1}(x', y+zr(x),z\beta) \right\}\\
   &\leq \max_{a \in \{0,1\}} \left\{ -(\lambda_2-\lambda_1)za + \sum_{x'} P(x'|x, a) (V_{\lambda_2,n-1}(x', y+z r(x), z\beta) - V_{\lambda_1,n-1}(x', y+z r(x), z\beta))\right\}\\
   &\leq \max_{a \in \{0,1\}} \left\{ -(\lambda_2-\lambda_1)z a\right\} = 0,
\end{align*}
where we used $V_{\lambda_2,n-1}(x,y,z) - V_{\lambda_1,n-1}(x,y,z) \leq 0$ for all $(x,y,z)$. Next, one can see that 
\begin{align*}
   V_{\lambda_2,n}&(x,y,z) - V_{\lambda_1,n}(x,y,z) \\
   &\geq \min_{a \in \{0,1\}} \left\{ -(\lambda_2-\lambda_1)za + \sum_{x'} P(x'|x, a) (V_{\lambda_2,n-1}(x', y+z r(x), z\beta) - V_{\lambda_1,n-1}(x', y+z r(x), z\beta))\right\}\\
   &\geq \min_{a \in \{0,1\}} \left\{ -(\lambda_2-\lambda_1)z a - \frac{z\beta}{1-\beta}(\lambda_2-\lambda_1)\right\} = - \frac{z}{1-\beta}(\lambda_2-\lambda_1).\Halmos
\end{align*}

\begin{lemma}\label{thm:ass3increasingPol}
If a restless bandit arm satisfies condition \ref{ass:r2} and Assumption~\ref{ass:mdp3}, then there exists a family of optimal policies $\{f_\lambda^{*}\}_{\lambda\geq 0}$, for its augmented arm risk neutral MDP, that is non-increasing with respect to $\lambda$.
\end{lemma}

\textbf{\textit{Proof.}}
We again consider the advantage function 
\[
\Delta_n(x,y,z,\lambda) := z\lambda + \sum_{x'} (P(x'|x, 0) - P(x'|x, 1)) V_{\lambda,n-1}(x', y+zr(x),z\beta).
\]
and showing that it is non-decreasing in $\lambda$. Namely, for \NimaEdits{$n=1$} we have
\[\Delta_{1}(x,y,z,\lambda) = z\lambda + \sum_{x'} (P(x'|x, 0) - P(x'|x, 1)) V_{\lambda,0}(x', y+zr(x),z\beta) = \NimaEdits{z\lambda} ,\]
thus non-decreasing in $\lambda$. Whereas when $n\geq 2$, one can show that if $\lambda_2 \geq \lambda_1\geq 0$, then
\begin{align*}
    \Delta_n&(x, y,z,\lambda_2)  - \Delta_n(x,y,z,\lambda_1)= z\lambda_2 + \sum_{x'} (P(x'|x, 0) - P(x'|x, 1)) V_{\lambda_2,n-1}(x', y+zr(x), z\beta) \\ &- (z\lambda_1 + \sum_{x'} (P(x'|x, 0) - P(x'|x, 1)) V_{\lambda_1,n-1}(x', y+zr(x), z\beta))\\
    &= z(\lambda_2-\lambda_1) + \sum_{x'} [P(x'|x, 0) - P(x'|x, 1)] (V_{\lambda_2,n-1}(x', y+zr(x),z\beta)-V_{\lambda_1,n-1}(x', y+zr(x),z\beta)\\
    &\geq z(\lambda_2-\lambda_1) - \sum_{x'} |P(x'|x, 0) - P(x'|x, 1)| |V_{\lambda_2,n-1}(x', y+zr(x),z\beta)-V_{\lambda_1,n-1}(x', y+zr(x),z\beta)|\\
    &\geq z(\lambda_2-\lambda_1) -\frac{1-\beta}{\beta} \frac{z\beta}{1-\beta}(\lambda_2-\lambda_1) \geq 0.
\end{align*}

Letting $n$ go to infinity, we can define:
\[\Delta_\infty(x,y,z,\lambda) := \lim_{n\rightarrow \infty} \Delta_n(x,y,z,\lambda),\]
and consider that
\[\Delta_\infty(x,y,z,\lambda) = z\lambda + \sum_{x'} (P(x'|x, 0) - P(x'|x, 1)) V_{\lambda,\infty}(x', y+zr(x),z\beta)\]
and satisfies $\Delta_\infty(x,y,z,\lambda_2)\geq \Delta_\infty(x,y,z,\lambda_1)$ for all $\lambda_2\geq\lambda_1\geq 0$.
Considering the policy
$f_{\lambda}^{*}(x,y,z):=\min(\arg\max_{a \in \{0,1\}} \left\{ -z\lambda a + \sum_{x'} P(x'|x, a) V_{\lambda,\infty}(x', y + z r(x), z\beta) \right\})$, one can again show if $f_{\lambda_1}^{*}(x,y,z)=0$, then 
\[-z \lambda_1\cdot 1 + \sum_{x'} P(x'|x, 1) V_{\lambda_1,\infty}(x', y+zr(x),z\beta) \leq -z \lambda_1\cdot 0 + \sum_{x'} P(x'|x, 0) V_{\lambda_1,\infty}(x', y+zr(x),z\beta)\]
by definition of the policy, thus implying that $\Delta_\infty(x, s,\lambda_1)\geq 0$. Moreover, we have $0\leq \Delta_\infty(x,y,z,\lambda_1)\leq \Delta_\infty(x,y,z,\lambda_2)$, which implies that \NimaEdits{$f_{\lambda_2}^{*}(x,y,z)=0$}. This confirms that this family of optimal policies is monotone in $\lambda$.
\Halmos

The rest of the proof of Theorem \ref{thm:inf-indexable2} follows directly as the proof of Theorem \ref{thm:ns-indexable} with the difference that it is Lemma \ref{thm:ass3increasingPol} that ensures the existence of  a family of optimal policies $\{f_\lambda^{i*}\}_{
\lambda\geq 0}$ that is non-increasing in $\lambda$. \Halmos

\end{document}